\documentclass[10pt]{article} % For LaTeX2e
\usepackage[preprint]{tmlr}

\usepackage{amsmath,amsfonts,bm}

\def\eqref#1{equation~\ref{#1}}

\def\1{\bm{1}}

\DeclareMathAlphabet{\mathsfit}{\encodingdefault}{\sfdefault}{m}{sl}
\SetMathAlphabet{\mathsfit}{bold}{\encodingdefault}{\sfdefault}{bx}{n}

\definecolor{linkblue}{RGB}{0, 51, 153}      % dark blue for internal links
\definecolor{citegreen}{RGB}{8, 203, 0} 
\definecolor{urlmagenta}{RGB}{153, 0, 153}   % magenta for URLs

\usepackage{hyperref}
\hypersetup{
    colorlinks=true,
    linkcolor=linkblue,      % sections, equations, figures
    citecolor=citegreen,     % citations
    urlcolor=urlmagenta,     % external links
    filecolor=urlmagenta,
    pdfborder={0 0 0},
    pdfcreator={LaTeX with hyperref},
    pdfproducer={LaTeX}
}
\usepackage{url}

\usepackage{booktabs} % for professional tables
\usepackage{wrapfig}

\usepackage{amsmath}
\usepackage{amssymb}
\usepackage{mathtools}
\usepackage{amsthm}
\usepackage{xcolor}
\usepackage{fontawesome}

\usepackage{dsfont} % for indicator
\usepackage{multirow}
\usepackage{subcaption}
\usepackage{enumitem}
\usepackage{algorithm}
\usepackage[noend]{algpseudocode}
\newtheoremstyle{refplain}
  {3pt}        % space above
  {3pt}        % space below
  {\itshape}   % body font
  {}           % indent
  {\bfseries}  % head font
  {.}          % punctuation after head
  {.5em}       % space after head
  {\thmname{#1}\thmnote{ #3}} % head format

\theoremstyle{refplain}

\newtheorem*{prop1*}{Proposition}
\newtheorem*{theorem*}{Theorem}
\newtheorem*{prop*}{Proposition}
\newtheorem*{dfn*}{Definition}
\newtheorem*{lem*}{Lemma}
\newtheorem*{theowithref}{Theorem}
\newtheorem*{corwithref}{Corollary}

\theoremstyle{plain}
\newtheorem{theorem}{Theorem}[section]
\newtheorem{proposition}[theorem]{Proposition}
\newtheorem{lemma}[theorem]{Lemma}
\newtheorem{corollary}[theorem]{Corollary}
\theoremstyle{definition}
\newtheorem{definition}[theorem]{Definition}

\theoremstyle{definition}
\newtheorem{remark}{Remark}

\title{Hyperspherical Simplex Representations for Backward-Compatible Retrieval}

\author{\name Simone Ricci\thanks{The first two authors contributed equally to this work.}
    \thanks{Corresponding author: Simone Ricci (simone.ricci@unifi.it).}
    \email simone.ricci@unifi.it \\
    \addr Media Integration and Communication Center (MICC)\\
    Dipartimento di Ingegneria dell'Informazione\\
    University of Florence
    \AND
    \name Niccolò Biondi\footnotemark[1]
    \email niccolo.biondi@unitn.it \\
    \addr Multimedia \& Human Understanding (MHUG)\\
    Department of Information Engineering and Computer Science \\
    University of Trento
    \AND
    \name Federico Pernici
    \email federico.pernici@unifi.it\\
    \addr Media Integration and Communication Center (MICC)\\
    Dipartimento di Ingegneria dell'Informazione\\
    University of Florence
    \AND
    \name Alberto Del Bimbo
    \email alberto.delbimbo@unifi.it\\
    \addr Media Integration and Communication Center (MICC)\\
    Dipartimento di Ingegneria dell'Informazione\\
    University of Florence
    \\\\
      \textbf{\href{https://miccunifi.github.io/PSP/}
     {\textcolor{blue!60!black}{\faGlobe\enspace{Website}}}
    \quad
  \href{https://github.com/miccunifi/PSP}
     {\textcolor{blue!60!black}{\faGithub\enspace{Code}}}}
}

\def\month{MM}  % Insert correct month for camera-ready version
\def\year{YYYY} % Insert correct year for camera-ready version
\def\openreview{\url{https://openreview.net/forum?id=XXXX}} % Insert correct link to OpenReview for camera-ready version

\begin{document}

\maketitle

\begin{abstract}
Training modern AI models has become increasingly expensive, and updating deployed models can alter their internal feature representations, affecting applications built upon them.
In retrieval systems, such updates often require a complete gallery re-indexing, which is computationally expensive at scale.
Existing backward-compatibility methods enable direct comparison between representations produced by updated and earlier models, avoiding re-indexing.
However, they typically require compatibility-specific training, learned mapping functions, or modifications to the model architecture, and often depend explicitly on the previous model.
In this paper, we propose hyperspherical simplex representations derived from softmax outputs and logits for backward-compatible retrieval across independently trained models.
These representations leverage the simplex structure induced by softmax cross-entropy training with one-hot labels, which provides a common geometric reference across model updates.
Since this reference depends on the number of classes, introducing new ones during an update expands the simplex, thereby misaligning the prototypes of existing classes with their earlier counterparts.
We therefore derive a deterministic linear transformation to restore their alignment.
We further provide a theoretical analysis of the formal backward-compatibility definition, identifying two geometric conditions under which its inequalities hold in expectation, assuming von Mises--Fisher class distributions whose concentration increases across updates.
Aligned regular-simplex representations satisfy these conditions, providing a theoretical basis for the proposed representations under the stated assumptions.
Without using additional compatibility-specific losses, learned mappings, or modifications to the model architecture, the proposed representations achieve competitive results on standard compatibility benchmarks, including scenarios with expanded class sets, more expressive architectures, and independently trained models from public repositories.
\end{abstract}

\section{Introduction}\label{sec:intro}
Training modern AI models has become increasingly expensive, limiting accessibility to a few well-resourced organizations \citep{radford2021learning, dubey2024llama, jiang2023mistral, anthropic2024claude, yang2025qwen3, xu2026deepseek, team2026gemma}. Despite reductions in model parameters and computing costs \citep{dubey2024llama}, training from scratch, fine-tuning, and inference remain economically challenging \citep{wang2024planningnaturallanguageimproves, openai_learning_2024}. 
As a result, models are increasingly released through public repositories \citep{marcel2010torchvision, wolf2019huggingface, rw2019timm} or provided as services through APIs, simplifying application development and enabling widespread adoption across domains.

As these models become more integrated into real-world systems, the need to update them is also growing. This demand is driven by several factors \citep{raffel2023building, yadav2024survey}, including evolving training strategies \citep{biondiallyouneed, biondi2024stationary, echterhoff2024muscle, shen2020towards}, advancements in model architectures \citep{touvron2023llama, kaplan2020scaling}, the availability of higher-quality datasets \citep{gunasekar2023textbooks}, the expansion of training classes, or extended training periods \citep{biderman2023pythia, raffel2023building}. 
Such updates not only improve performance but also bring the rapid progress of the field into deployed models, simplifying their adoption despite the continued growth in the scale and complexity of models, datasets, and infrastructures \citep{bommasani2021opportunities, sorscher2022beyond}.

However,  a significant drawback of model updates is that they can compromise compatibility with earlier versions, altering the behavior of applications or services built upon them \citep{raffel2023building}.
For example, image retrieval services might experience unexpected changes in image rankings following a model update, forcing users to adapt when previously top-ranked images no longer appear first \citep{shen2020towards}. 
Similarly, end users, such as drivers of semi-autonomous cars, often develop a mental model of the machine learning system’s capabilities \citep{bansal2019beyond}. 
When updates modify system behavior, users must adjust their mental model of the system’s functionality, a task that can be challenging, dissatisfying, and even unsafe \citep{bansal2019updates, templeton2019tesla}.
Recent observations have highlighted similar issues with large language models (LLMs) \citep{echterhoff2024muscle, goswami2025query}, raising concerns about AI safety and alignment. This is particularly critical when decisions and textual outputs from these agents are translated into actions \citep{amodei2016concrete,ngoalignment}.
In such settings, it is essential that the predictions of the updated model mimic those of the old model for any correct classification \citep{milani2016launch, yan2021positive, riccimargin}.

Updating a model while ensuring compatibility with its earlier versions has been explored along three complementary dimensions:
1) Weights-level, which explores the merging of model parameters under the assumption of linear mode connectivity, to unify multiple versions into a single, robust model \citep{frankle2020linear, wortsman2022model, ainsworth2023git}; 2) Representation-level, which seeks to learn backward-compatible representations that are comparable across model updates, so they can be used interchangeably over time \citep{shen2020towards, Meng_2021_ICCV, biondiallyouneed}; and 3) Classifier decision-level, which ensures that the updated model preserves the correct predictions of the old model, preventing negative flips \citep{milani2016launch, yan2021positive, DBLP:journals/pami/ZhaoSXYXTSS24}.
They target distinct notions of compatibility and accordingly differ in their objectives, mechanisms, and evaluation criteria.

Among these paradigms, we focus on representation-level compatibility, which is particularly relevant to retrieval systems, where feature representations produced by different model versions must remain directly comparable.
Backward-compatibility training, introduced by \citet{shen2020towards} to avoid image re-indexing of the gallery set, has been established as a standard training paradigm to achieve backward-compatibility between model representations. 
However, most existing studies on backward-compatibility require training a new model using the old model as a reference \citep{shen2020towards, Meng_2021_ICCV, duggal2021compatibility, pan2023boundary, zhou2023bt}. 
This approach has become increasingly challenging due to restricted access to old model architectures and weights---often limited by proprietary constraints and non-disclosure policies---as well as the high computational costs associated with training large-scale modern AI models.
Recent studies \citep{biondi2023cores, biondi2024stationary, biondiallyouneed} have addressed this limitation by learning compatible representations without explicitly using the previous model as a reference. However, these approaches require all parties to agree in advance on a shared, pre-allocated fixed classifier~\citep{Pernici_2019_CVPR_Workshops, pernici2021regular, zhu2021geometric}.

In this paper, we show that hyperspherical simplex representations derived from the softmax outputs or logits of independently trained\footnote{In this work, the term \emph{independently trained} means that no model is trained using an earlier model’s parameters, representations, outputs, or compatibility losses. Models may nevertheless share the task, data, initialization, and label-to-output coordinate convention.} models provide backward-compatible retrieval.
Our formulation leverages the simplex geometry induced during training by the softmax function and one-hot encoded class labels as a model-agnostic reference, 
yielding hyperspherical simplex representations that are aligned across model updates.
% Since their simplex geometries evolve with the introduction of new classes, we propose a deterministic linear transformation that compensates for the resulting angular misalignment between the old and new simplex geometries, thereby preserving representation alignment across updates.
Since their simplex geometries evolve with the introduction of new classes, we propose a deterministic linear transformation that compensates for the resulting angular misalignment between the prototypes of existing classes and their earlier counterparts, thereby preserving their alignment.
We theoretically analyze the formal backward-compatibility definition and identify geometric conditions under which its inequalities \citep{shen2020towards} hold in expectation, given explicit assumptions on the feature distributions and their concentration across updates.
We then show that these conditions are satisfied by aligned regular-simplex representations, and therefore by the proposed hyperspherical simplex representations, which exhibit this geometry.
Although the proposed hyperspherical simplex representations satisfy the geometric conditions established in our analysis, their dimensionality increases linearly with the number of training classes, resulting in substantial storage costs for large-scale galleries.
To mitigate this limitation, we adopt a top-$k$ sparsification strategy motivated by the rank-preserving property of the softmax function and the top-$k$ calibration of the softmax cross-entropy loss \citep{lapin2016loss, DBLP:conf/icml/YangK20}, which reduces gallery storage with negligible changes in retrieval accuracy and without degrading compatibility.

Our main contributions are the following:
\begin{enumerate}
    \item We introduce a novel approach for backward-compatible retrieval based on hyperspherical simplex representations derived from softmax outputs and logits.

    \item We provide a theoretical analysis of the formal definition of backward-compatibility, identifying geometric conditions under which its inequalities are satisfied in expectation, given explicit distributional and concentration assumptions. We show that hyperspherical simplex representations derived from softmax outputs and logits satisfy these geometric conditions.

    \item We propose a top-$k$ sparsification method for hyperspherical simplex representations, motivated by the rank-preserving property of the softmax function and the top-$k$ calibration of the softmax cross-entropy loss, to reduce representation storage.

    \item We achieve competitive results on standard compatibility benchmarks, obtaining substantial improvements in scenarios involving a large number of model updates, without using additional compatibility-specific losses, learned mappings, or modifications to the model architecture.
    We provide empirical evidence that representations derived from softmax outputs and logits substantially improve backward-compatibility, even between independently pretrained models obtained from a public repository.

\end{enumerate}

The remainder of the paper is organized as follows.
In Sec.~\ref{sec:related_works}, we review related work on backward-compatibility.
In Sec.~\ref{sec:how_psp}, we derive hyperspherical simplex representations from softmax outputs and logits and analyze their backward-compatibility under stated theoretical assumptions.
In Sec.~\ref{sec:when_experimental_results}, we present extensive experiments and compare our proposed method with state-of-the-art alternatives.
Finally, in Sec.~\ref{sec:conclusions}, we draw conclusions and discuss the limitations of this work.

\section{Related Work} \label{sec:related_works}

Backward-compatibility was first formalized by \citet{shen2020towards}, who defined both theoretical and empirical criteria to verify whether learned representations are interoperable across model updates, thus eliminating the need to re-index the gallery set in retrieval systems.
To learn backward-compatible representations, they proposed Backward-Compatible Training (BCT), where the previously learned classifier is frozen and serves as a fixed reference during training of the new model, aligning new feature vectors with their corresponding old class prototypes. 
Subsequent work proposed additional regularization objectives to enforce alignment between old and new representations \citep{budnik2021asymmetric, duggal2021compatibility, Meng_2021_ICCV, zhang2021hot, ijcai2022uniBCT, pan2023boundary, wu2022neighborhood, goswami2025query}.
In \citet{Meng_2021_ICCV}, a combination of loss functions is used to align class means between different model updates and to achieve more compact intra-class distributions while learning new information, directly optimizing a point-to-set formulation of backward-compatibility.
In \citet{zhang2021hot}, compatible representations are learned with a contrastive loss that encourages each new-to-old positive pair to be more similar than both new-to-old and new-to-new negative pairs.
\citet{pan2023boundary} introduce an adversarial discriminator within BCT to minimize the distribution gap between features from the old and new models.
In these regularization-based approaches, due to the constraints imposed to achieve compatibility, the performance of the new backward-compatible model often does not reach that of a newly independently trained model~\citep{zhou2023bt}, trained without using the old model as a reference.
To address this issue, \citet{zhou2023bt} and \citet{ricci2024backward} focus on modifying the model architecture by expanding the feature space at each model update, such that the original subspace preserves compatibility with previous models, while information from the new data is learned in the additional dimensions. However, even if compatibility is achieved while retaining the discriminative performance of a newly independently trained model, these approaches require explicit modifications to the model architecture that introduce growing complexity with each successive feature space expansion.

Along a different line, several studies have proposed mapping-based approaches to achieve compatibility across model updates~\citep{Chen_2019_CVPR, wang2020unified, Meng_2021_ICCV, ramanujan2022forward, hu2022learning, jaecklefastfill, ricci2025orthogonality}.
These methods learn mapping functions to align corresponding representations of the same inputs obtained from two different models, enabling direct comparison between them. 
Such approaches are closely related to prior work in cross-model and cross-modal representation learning, where alignment between feature spaces of different models or modalities is a central objective~\citep{Moschella2022-yf, cannistraci2023boot, cannistraci2024from, maniparambil2024vision, maniparambil2025harnessing, ricci2026spherical}.

However, feature-space alignment alone does not guarantee satisfaction of the formal compatibility inequalities, and mapping-based approaches introduce significant overhead in scenarios with many model updates, as they require composing multiple learned mapping functions across successive updates.
To address these limitations, the solution proposed by~\citet{biondi2023cores, biondi2023cl2r, biondi2024stationary, biondiallyouneed} involves learning with a pre-allocated $d$-simplex fixed classifier, which enforces feature alignment across updates by fixing the position of class prototypes a priori (i.e., from the beginning of training), while accommodating new classes within pre-allocated regions of the feature space.
Although this approach achieves compatibility without using the old model as a reference, it relies on the same pre-allocated fixed classifier across model updates and requires architectural modifications.

In this paper, we show that hyperspherical representations with an aligned regular-simplex geometry can be obtained directly from the logits or softmax outputs of independently trained models, without requiring a pre-allocated fixed classifier shared across updates. Building on this observation, we derive hyperspherical simplex representations for backward-compatible retrieval that do not require additional compatibility-specific losses, learned mapping functions, or modifications to the model architecture.
At the theoretical level, while \citet{biondiallyouneed} numerically show that stationary $d$-simplex representations satisfy the compatibility inequalities in expectation, our analysis complements this result by providing closed-form geometric conditions for aligned hyperspherical representations under explicit distributional model assumptions.

\section{Compatibility of Hyperspherical Simplex Representations}\label{sec:how_psp}

In this section, we first introduce hyperspherical simplex representations derived from softmax outputs and logits. We then identify geometric conditions under which the compatibility inequalities of~\citet{shen2020towards} hold in expectation given explicit assumptions on the feature distributions and their increasing concentration across model updates. Finally, we show that the proposed representations satisfy these conditions.

\subsection{Preliminaries and Motivation} \label{sec:prelims_and_motivation}
According to~\citet{shen2020towards}, backward-compatibility across model updates is achieved when the representation produced by the updated model at step $t$ can be compared with the representation produced by the earlier model at step $k$. 
Formally, it requires that the following two inequalities hold: 
\begin{definition}[\textbf{\textit{Backward-Compatibility}}]
\label{def:compatibility}
\begin{subequations}\label{eq:compatible_set_dist}
\begin{align}
 &d \big(\mathbf{h}^t_{i}, \mathbf{h}^k_{j} \big) \leq d \big(\mathbf{h}^k_{i}, \mathbf{h}^k_{j} \big), \forall (i, j) \in \{(i, j) \, | \, y_i = y_j \} \label{eq:first} \\
 &d \big(\mathbf{h}^t_{i}, \mathbf{h}^k_{j}  \big) \geq d \big( \mathbf{h}^k_{i}, \mathbf{h}^k_{j}  \big), \forall (i, j) \in \{(i, j) \, | \, y_i \neq y_j \} \label{eq:second}
\end{align}
\end{subequations}
where $\mathbf{h}^t$ and $\mathbf{h}^k$ are feature vectors extracted respectively with the updated model at step $t$ and the earlier model at step $k$, $d$ is a distance function (e.g., cosine distance), and $y_i$ and $y_j$ are class labels associated with $\mathbf{h}_{i}$ and $\mathbf{h}_{j}$, respectively.
\end{definition}
In retrieval tasks, the inequalities in Eq.~\ref{eq:first} and Eq.~\ref{eq:second} suggest that the updated model representation, when used to compare against the old representation, should perform at least as well as the old one in grouping images from the same classes and separating those from different classes.
While Eq.~\ref{eq:first} and Eq.~\ref{eq:second} describe desirable goals for compatible representations, directly verifying them requires exhaustive pairwise comparisons between the two representations, which is impractical at scale \citep{shen2020towards}.

To obtain a numerical assessment of these inequalities, \citet{biondi2024stationary,biondiallyouneed} estimate the corresponding expected distances through Monte Carlo integration under the unconstrained feature model (UFM) assumption \citep{mixon2022neural}, considering feature distributions centered around the prototypes of a pre-allocated $d$-simplex fixed classifier.
In their initial analysis, \citet{biondi2024stationary} model class features as samples from Euclidean hyperballs and evaluate compatibility using Euclidean distance. Under this formulation, the different-class inequality (Eq.~\ref{eq:second}) is not satisfied, and stationary $d$-simplex representations are therefore characterized as an optimal approximation to backward-compatibility. Extending this analysis, \citet{biondiallyouneed} model class features as samples from hyperspherical caps on the unit hypersphere and evaluate compatibility using cosine distance. Assuming that the hyperspherical caps shrink across model updates, they demonstrate that both compatibility inequalities hold in expectation for representations learned with a pre-allocated $d$-simplex fixed classifier.

While these studies assess compatibility numerically for stationary $d$-simplex representations, we analytically characterize the compatibility inequalities and identify geometric conditions under which they are satisfied in expectation.
Specifically, we model hyperspherical features using von Mises--Fisher (vMF) distributions on the unit hypersphere \citep{fisher1953dispersion, Banerjee2005ClusteringOT}, which yield closed-form expressions for the expected cosine distances in terms of class prototype directions and concentration parameters.
Assuming that the concentration of the class distributions increases across model updates, we prove in Theorem~\ref{theo:compatibility} of Sec.~\ref{sec:theo_section} that both inequalities of Def.~\ref{def:compatibility} are satisfied in expectation if the following two geometric conditions hold:
\begin{enumerate}[label=(\alph*), ref=(\alph*),itemsep=0pt, topsep=3pt]
  \item \label{cond:align} corresponding class prototypes are aligned across updates;
  \item \label{cond:obtuse} the angle between prototypes of different classes is greater than $\pi/2$.
\end{enumerate}
% These conditions directly guide the design of the proposed representations: condition~\ref{cond:obtuse} motivates the centering projection introduced in Sec.~\ref{sec:psp}, while condition~\ref{cond:align} motivates the angular correction under class expansion described in Sec.~\ref{sec:projection}.
% We then show that aligned regular-simplex representations satisfy both conditions, and that the proposed hyperspherical simplex representations derived from softmax outputs and logits exhibit this geometry across model updates.
% Since their simplex structure emerges naturally from softmax cross-entropy training with one-hot labels, the proposed representations provide a practical and theoretically grounded approach to backward-compatibility between independently trained models.

In particular, both conditions are satisfied by regular-simplex representations that are centered at the origin and aligned across model updates, as established in Cor.~\ref{cor:regulsimplex_compatibility}.
Since softmax cross-entropy training with one-hot labels induces a simplex structure in both the softmax outputs and the logits, these provide a principled basis for constructing backward-compatible representations without compatibility-specific training.
Accordingly, the two conditions directly guide the design of the proposed representations: condition~\ref{cond:obtuse} motivates the centering projection introduced in Sec.~\ref{sec:psp}, whereas condition~\ref{cond:align} motivates the angular correction under class expansion described in Sec.~\ref{sec:projection}.
% Sec.~\ref{sec:theo_section} then shows that, the proposed hyperspherical simplex representations satisfying both conditions and providing a practical and theoretically grounded approach to backward-compatibility between independently trained models.

% 
% In particular, both conditions are satisfied by regular-simplex representations that are centered at the origin and aligned across model updates (Cor.~\ref{cor:regulsimplex_compatibility}).
% Since softmax cross-entropy training with one-hot labels induces a simplex structure in both the softmax outputs and the logits, these provide a natural starting point for constructing backward-compatible representations without compatibility-specific training.
% Accordingly, the two conditions directly guide the design of the proposed representations: condition~\ref{cond:obtuse} motivates the centering projection introduced in Sec.~\ref{sec:psp}, whereas condition~\ref{cond:align} motivates the angular correction under class expansion described in Sec.~\ref{sec:projection}.
% The resulting hyperspherical simplex representations exhibit an aligned regular-simplex geometry across model updates, thereby providing a practical and theoretically grounded approach to backward-compatibility between independently trained models.
% }

\subsection{Hyperspherical Simplex Representations}

We now formalize two hyperspherical simplex representations, which are derived directly from the softmax outputs and logits of a trained classification model.

\subsubsection{Features Derived From Softmax Outputs}
\label{sec:psp}

Given a model at update step $k$, the softmax output \(\sigma(\mathbf{z}^k)\), obtained by applying the softmax function \(\sigma(\cdot)\) to the logits $\mathbf{z}^k$, represents a probability distribution over the $C^k$ classes and lies within the probability simplex, which is defined as follows:
\begin{definition}[\textbf{\textit{Probability Simplex}}]
\label{def:probability_simplex}
The probability simplex $\Delta^{C^k-1}$ is the \( (C^k-1) \)-dimensional simplex in \( \mathbb{R}^{C^k} \) whose vertices are the canonical basis vectors \( \mathbf{e}_1, \mathbf{e}_2, \ldots, \mathbf{e}_{C^k} \), and its center is \( \mathbf{o}^k = \frac{1}{C^k} \sum_{c=1}^{C^k} \mathbf{e}_c \), such that 
$
\Delta^{C^k-1} = \{ \mathbf{u} \in \mathbb{R}^{C^k} : u_1+\dots + u_{C^k} = 1, \, u_j \geq 0 \;\;\; \text{for} \;\;\; j=1,\dots,C^k \}.
$
\end{definition}

\begin{wrapfigure}[26]{r}{0.5\textwidth}
\vspace{-1.2cm}
  \begin{center}
    \includegraphics{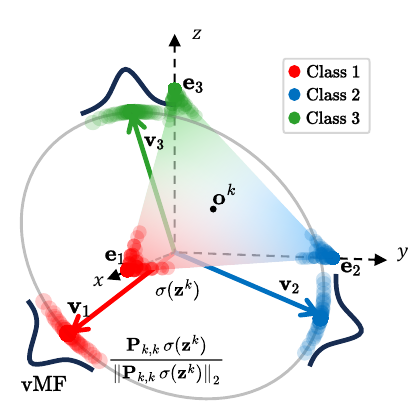}
  \end{center}
  \vspace{-0.4cm}
    \caption{
    Softmax outputs lie in the probability simplex $\Delta^{C^k-1}$ (colored area), centered at $\mathbf{o}^k$, and are driven during training toward its vertices $\mathbf{e}_c$, which correspond to the one-hot target labels.
    The projection $\mathbf{P}_{k,k}$ shifts the simplex center to the origin, mapping softmax outputs onto a regular simplex whose vertices define the class prototypes $\mathbf{v}^k_c$ (colored vectors).
    The features $\mathbf{h}^k$ of Eq.~\ref{eq:softmax_outputs_as_features} are then obtained by $\ell_2$-normalization onto the hypersphere, where each class is modeled by a vMF distribution with mean direction along the corresponding class prototype.
  }
  \label{fig:vmF}
\end{wrapfigure}
For a training sample of class $c$, the cross-entropy loss is minimized over the probability simplex $\Delta^{C^k-1}$ at the vertex $\mathbf{e}_c$, which corresponds to the one-hot encoded target label. Hence, the optimization of the cross-entropy loss drives the softmax outputs toward the vertices $\mathbf{e}_c$ of $\Delta^{C^k-1}$ during training (see Fig.~\ref{fig:vmF}).
A straightforward way to obtain hyperspherical representations from softmax outputs is to apply $\ell_2$-normalization directly.
However, this approach does not satisfy condition~\ref{cond:obtuse} of Sec.~\ref{sec:prelims_and_motivation}. Since the probability simplex $\Delta^{C^k-1}$ lies entirely in the non-negative orthant, with its center $\mathbf{o}^k$ offset from the origin, any two softmax outputs have a non-negative inner product: the vertices $\mathbf{e}_c$ toward which the softmax outputs are driven are mutually orthogonal, and the $\ell_2$-normalized softmax outputs of different classes are separated by angles of at most $\pi/2$.

To overcome this limitation and obtain angles greater than $\pi/2$ between distinct classes, we apply a projection matrix $\mathbf{P}_{k,k}$ before the $\ell_2$-normalization of the softmax outputs. This projection centers the probability simplex $\Delta^{C^k-1}$ at the origin, mapping each softmax output to the corresponding point of a $(C^k-1)$-regular simplex\footnote{A $(C^k-1)$-regular simplex centered at the origin has equidistant vertices and a constant angle between every pair of distinct vertices, equal to $\arccos(-1/(C^k-1))$.}, while preserving pairwise distances, as shown in Fig.~\ref{fig:vmF}.

To formulate the projection matrix, we define the class prototypes $\mathbf{v}^k_c$ as the vectors from the center $\mathbf{o}^k$ to the vertices $\mathbf{e}_c$ toward which the softmax outputs align:
\begin{equation} \label{eq:class_prototypes}
\mathbf{v}^k_c = \mathbf{e}_c - \mathbf{o}^k \quad \forall c \;,\; 1 \leq c \leq C^k, \quad \mathbf{v}^k_c \in \mathbb{R}^{C^k}.
\end{equation}

These prototypes span the $(C^k-1)$-dimensional hyperplane of the regular simplex and therefore define the projection matrix. Accordingly, the deterministic projection $\mathbf{P}_{k,k}$ is defined as the following $C^k \times C^k$ matrix:
\begin{equation} \label{eq:projection_prot}
\mathbf{P}_{k,k} = [ \mathbf{v}^k_c \; ]_{c=1}^{C^k} = \mathbf{I}_{C^k} - \frac{1}{C^k} \mathbf{J}_{C^k}
\end{equation}
The hyperspherical simplex representation $\mathbf{h}^k$ is then obtained as the $\ell_2$-normalized projection of the softmax output\footnote{When the softmax output corresponds to the uniform distribution (i.e., maximum predictive uncertainty) so that $\sigma(\mathbf{z}^k)=\mathbf{o}^k$, the projected vector $\mathbf{P}_{k,k}\,\sigma(\mathbf{z}^k)$ is the zero vector, making $\ell_2$-normalization undefined. In this particular case, we define $\mathbf{h}^k \triangleq \mathbf{0}$. Although theoretically possible, this case is unlikely in practice because trained classifiers typically produce non-uniform output distributions~\citep{DBLP:conf/iclr/HendrycksG17}.}:
\begin{equation}
    \mathbf{h}^k = \frac{\mathbf{P}_{k,k} \, \sigma(\mathbf{z}^k)}{\left\lVert \mathbf{P}_{k,k} \, \sigma(\mathbf{z}^k) \right\rVert_2}, \quad \mathbf{h}^k \in \mathbb{R}^{C^k}.
    \label{eq:softmax_outputs_as_features}
\end{equation}

% The set of these vectors spans the ($C^k-1$)-dimensional hyperplane of the regular simplex, a fundamental property that allows these vectors to directly define the subspace for the projection. 
% Thus, the deterministic projection \(\mathbf{P}_{k,k}\)  
% is defined as the following ${C^k \times C^k}$ matrix:  
% \begin{equation} \label{eq:projection_prot}
% \mathbf{P}_{k,k} = \mathbf{V}^k = [ \mathbf{v}^k_c \; ]_{c=1}^{C^k} = \mathbf{I}_{C^k} - \frac{1}{C^k} \mathbf{J}_{C^k}
% \end{equation}
% The hyperspherical simplex representation \( \mathbf{h}^k \) is the $\ell_2$-normalized projection\footnote{When the softmax output corresponds to the uniform distribution (i.e., maximum predictive uncertainty) so that $\sigma(\mathbf{z}^k)=\mathbf{o}^k$, the projected vector $\mathbf{P}_{k,k}\,\sigma(\mathbf{z}^k)$ is the zero vector, making $\ell_2$-normalization undefined. In this particular case, we define $\mathbf{h}^k \triangleq \mathbf{0}$. Although theoretically possible, this case is unlikely in practice because trained classifiers typically produce non-uniform output distributions~\citep{DBLP:conf/iclr/HendrycksG17}.} of softmax output vector via $\mathbf{P}_{k,k}$, defined as:
% \begin{equation}
%     \mathbf{h}^k = \frac{\mathbf{P}_{k,k} \, \sigma(\mathbf{z}^k)}{\left\lVert \mathbf{P}_{k,k} \, \sigma(\mathbf{z}^k) \right\rVert_2}, \quad \mathbf{h}^k \in \mathbb{R}^{C^k}.
%     \label{eq:softmax_outputs_as_features}
% \end{equation}

The class prototypes of the resulting representation $\mathbf{h}^k$ satisfy condition~\ref{cond:obtuse} by construction, as they form a regular simplex centered at the origin with pairwise angles greater than $\pi/2$~\citep{pernici2021regular}.
Moreover, since the class prototypes $\mathbf{v}^k_c$ are defined with respect to the canonical basis vectors $\mathbf{e}_c$ associated with the one-hot encoding of class $c$, their directions are fixed for a given set of classes, so that condition~\ref{cond:align} holds across model updates when the class set is unchanged; the case in which new classes are introduced during an update is addressed in Sec.~\ref{sec:projection}.
On the hypersphere, we model each class distribution by a vMF distribution, the hyperspherical analogue of an isotropic Gaussian, whose mean direction is aligned with the class prototype $\mathbf{v}^k_c$ (see Fig.~\ref{fig:vmF}).

% \textcolor{teal}{On the hypersphere, we model each class distribution of the hyperspherical simplex representation by a vMF distribution \citep{fisher1953dispersion, Banerjee2005ClusteringOT}, the hyperspherical analogue of an isotropic Gaussian distribution (see Fig.~\ref{fig:vmF}), where the mean direction of each vMF is aligned with the class prototype $\mathbf{v}_c$.}
% 
% This modeling assumption enables a closed-form analysis of the expected cosine distances for both Eq.~\ref{eq:first} and Eq.~\ref{eq:second}, thereby allowing the representation in Eq.~\ref{eq:softmax_outputs_as_features} to be studied within the theoretical framework of Theorem~\ref{theo:compatibility}.
% }

\subsubsection{Features Derived From Logits}\label{sec:logits_theory}

The simplex geometry induced by softmax cross-entropy training is reflected not only in the softmax outputs but also in the underlying logits.
Motivated by the UFM analyses of \citet{mixon2022neural} and \citet{fang2021exploring}, we analyze this relationship directly at the level of the logits.
Specifically, rather than modeling the shared backbone and classifier, we treat the logit vector of each training sample as an independent optimization variable, initialized at the origin (i.e., at the equiprobable softmax distribution), and study its evolution under gradient descent on the softmax cross-entropy loss; we refer to this abstraction as the logit-UFM.
This formulation isolates the geometry induced directly by the loss from the additional interactions introduced by the network architecture and shared model parameters, leading to the following proposition:

% 
% Following \citet{mixon2022neural} and \citet{fang2021exploring}, we analyze this relationship through a logit-space formulation of the Unconstrained Feature Model (UFM).
% Specifically, rather than explicitly modeling the shared backbone and classifier, we treat the logit vector associated with each training sample as an independent optimization variable and study its evolution under gradient descent on the softmax cross-entropy loss, initialized at the origin, which corresponds to the equiprobable softmax distribution.
% This logit-UFM formulation isolates the geometry induced directly by the loss, without the additional interactions introduced by the network architecture and shared model parameters.
% We then assess whether this geometric behavior is reflected in the logits produced by standard neural networks using the Neural Collapse (NC) hypotheses~\citep{papyan2020prevalence}.
% We formalize the resulting logit dynamics in the following proposition:
% }

\begin{proposition}[Proof in Appendix~\ref{sec:app_softmax_affects_logits}]
\label{theo:softmax_affects_logits}
Under the logit-UFM assumptions, as the softmax outputs converge toward the vertices of the probability simplex during training, the corresponding logit vectors align with the directions defined in Eq.~\ref{eq:class_prototypes}.
\end{proposition}

\begin{wrapfigure}[21]{r}{0.5\textwidth}
\vspace{-0.6cm}
  \begin{center}
    \includegraphics[width=0.85\linewidth]{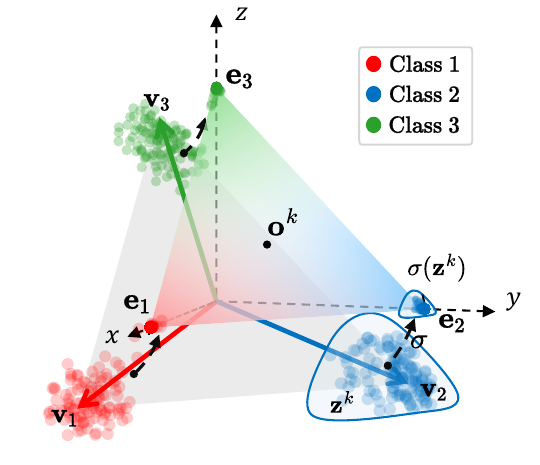}
  \end{center}
  \vspace{-0.4cm}
  \caption{Softmax outputs and logits obtained from a ResNet-18 trained on three CIFAR-10 classes. While softmax outputs concentrate near the vertices of the probability simplex (colored area), logits exhibit a similar simplex geometry (gray area) centered at the origin with a broader within-class spread.} \label{fig:probability_simplex}
\end{wrapfigure}
This behavior also emerges in standard networks, as illustrated in Fig.~\ref{fig:probability_simplex} for a representative three-class case from the CIFAR-10 dataset.
At the end of training, the softmax outputs concentrate near the vertices of the probability simplex (colored area), while the logits exhibit a regular-simplex geometry (gray area) aligned with the same class prototype directions $\mathbf{v}^k_c$.
Moreover, the logits are expected to be approximately centered in the terminal phase of training, consistent with theoretical analyses of Neural Collapse (NC)~\citep{papyan2020prevalence} that yield centered prediction vectors at global minimizers~\citep{hong2024neural} and with the observation that normalization mechanisms commonly used in neural networks tend to center latent representations~\citep{Moschella2022-yf}.
We therefore directly $\ell_2$-normalize the logits, without the centering projection required for the softmax outputs, yielding a second hyperspherical simplex representation:
\begin{equation} \label{eq:logits_as_features}
    \mathbf{h}^k = \frac{\mathbf{z}^k}{\| \mathbf{z}^k\|_2} \; , \quad \mathbf{h}^k \in \mathbb{R}^{C^k}.
\end{equation}

Analogously to the representation in Eq.~\ref{eq:softmax_outputs_as_features}, we model each class distribution of the normalized logit representation with a vMF distribution~\citep{Banerjee2005ClusteringOT} whose mean direction is aligned with the class prototype $\mathbf{v}^k_c$.
Since both representations share the same prototype directions, the representation in Eq.~\ref{eq:logits_as_features} inherits the geometric properties discussed in Sec.~\ref{sec:psp} and can likewise be analyzed within the theoretical framework of Theorem~\ref{theo:compatibility}.
Compared with the representation derived by the softmax outputs as in Eq.~\ref{eq:softmax_outputs_as_features}, the normalized logit representation exhibits a larger angular spread (see Fig.~\ref{fig:probability_simplex}), preserving greater within-class variability, a property that has been associated with improved transferability and robustness of learned representations~\citep{chen2022perfectly,wang2020understanding}.

To assess whether the behavior predicted by Prop.~\ref{theo:softmax_affects_logits} emerges in standard neural networks, and to quantify the spread of the two representations, we leverage the NC hypotheses introduced by~\citet{papyan2020prevalence}.
These hypotheses provide a methodology for examining how the representations derived from softmax outputs and logits collapse toward the class prototype directions defined in Eq.~\ref{eq:class_prototypes}.
The NC hypotheses evaluate whether the within-class covariance tends to 0 (NC1); whether class means form a simplex whose vertices have the same norm (NC2 equinorm) and are maximally and equally distant from each other (NC2 equidistance); whether class means converge to class prototypes (NC3); and whether the nearest class-mean rule for classifying features holds (NC4).
% To empirically assess Prop.~\ref{theo:softmax_affects_logits}
% and the spread of features, we leverage the NC hypotheses introduced by~\citet{papyan2020prevalence}. 
% These hypotheses provide a methodology for examining how the representations derived from softmax outputs and logits collapse toward the class prototype directions defined in Eq.~\ref{eq:class_prototypes}.
% The NC hypotheses evaluate: (a) whether the within-class covariance tends to 0 (NC1); (b)  whether class means of features tend to form a simplex whose vertices have the same norm (NC2 equinorm) and are placed at the maximum possible distance from each other (NC2 equidistance); (c) whether class means converge to class prototypes (NC3); (d) whether the nearest class-means rule for classifying features holds (NC4).

Fig.~\ref{fig:neural_collpase_proof} reports the corresponding NC measures for softmax outputs (blue curves) and logits (orange curves) during the terminal phase of training of a ResNet-18 on CIFAR-100.
The higher NC1 values observed for logits indicate greater residual within-class variability than for softmax outputs.
Similarly, the higher NC2-equidistance and NC3 values indicate a lower degree of convergence of the logit representations toward the ideal simplex geometry and the corresponding class-prototype directions.
These results show that, while both representations exhibit the same underlying simplex structure, logits retain a larger within-class spread and a lower degree of prototype alignment than softmax outputs.
This distinct trade-off between alignment and spread is reflected in the retrieval performance reported in Sec.~\ref{sec:when_experimental_results}.
Similar trends across different model architectures and datasets are reported in Appendix~\ref{sec:app_neural_collapse_logits}.

\begin{figure*}[t]
    \centering
    \hspace{-10pt}
    \includegraphics[width=1.03\linewidth]{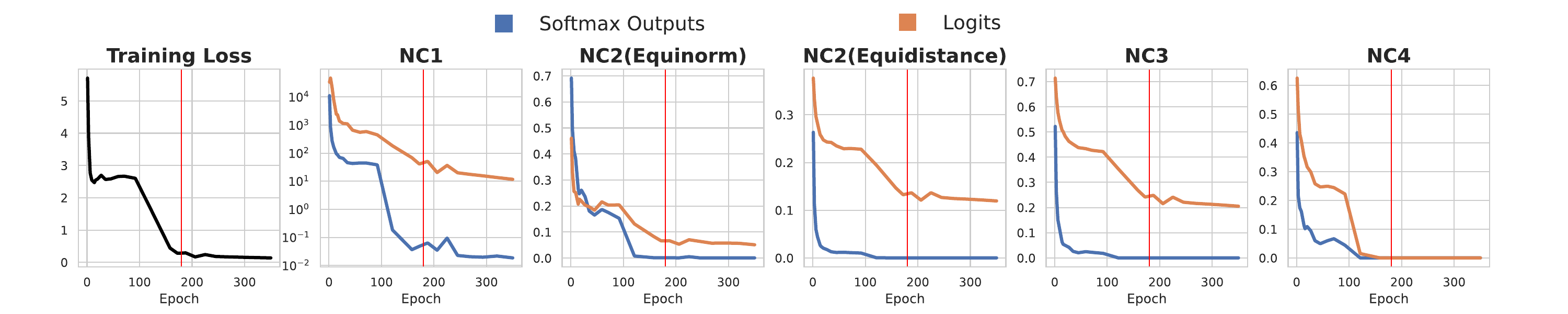}
    \caption{
    Plots of training loss and Neural Collapse (NC) hypotheses for softmax outputs (blue) and logits (orange) on CIFAR-100 with ResNet-18. The terminal phase of training is shown with a vertical red line.
    }
    \label{fig:neural_collpase_proof}
\end{figure*}

\subsection{Deterministic Angular Correction Under Class Expansion}\label{sec:projection}

In this section, we address the misalignment that arises in the proposed hyperspherical simplex representations when new classes are introduced during model updates, analyzing the behavior of the features through their class prototypes \citep{papyan2020prevalence}.
We formally characterize this misalignment and introduce a deterministic linear transformation that restores prototype alignment across updates.
% The transformation exploits the geometry of the representations in Eq.~\ref{eq:softmax_outputs_as_features} and Eq.~\ref{eq:logits_as_features} and requires neither learnable parameters nor data forwarding, providing a training-free mechanism for alignment under class expansion.
% 
% In this section, we address the prototype misalignment that arises in the proposed hyperspherical simplex representations when new classes are introduced during model updates.
% We formally characterize this misalignment and introduce a deterministic linear transformation that restores prototype alignment across updates.
% The transformation exploits the geometry of the representations in Eq.~\ref{eq:softmax_outputs_as_features} and Eq.~\ref{eq:logits_as_features} and requires neither learnable parameters nor data forwarding, providing a training-free mechanism for alignment under class expansion.
% }

As follows from Def.~\ref{def:probability_simplex}, the geometry of the probability simplex depends explicitly on the number of classes.
When the model is updated from step $k$ to step $t$ by introducing additional classes, the classifier expands from $C^k$ to $C^t$ outputs, changing the simplex center and, consequently, the directions of the prototypes associated with the previously existing classes according to Eq.~\ref{eq:class_prototypes}.
Since the representations at the two steps have different dimensions, throughout this section vectors at step $k$ are embedded into $\mathbb{R}^{C^t}$ by zero-padding whenever they are compared with vectors at step $t$.
As illustrated by the dashed arrows in Fig.~\ref{fig:projections_main}, the corresponding prototypes at step $t$ belong to a higher-dimensional simplex and are misaligned with their counterparts at step $k$.
The resulting angular misalignment (derived in Appendix~\ref{sec:app_theo_misalignment_short}) is the same for all previously existing classes and is given by:
\begin{equation}\label{eq:misalignement}
\theta_{c}^{k\to t}
=\arccos\left(\sqrt{\frac{C^k-1}{C^k}\cdot \frac{C^t}{C^t-1}}\right),
\end{equation}
% where $C^k$ and $C^t$ represent the number of classes at steps $k$ and $t$, respectively, with $C^k \leq C^t$.
% From Eq.~\ref{eq:misalignement}, it follows that:
% (a) when $C^t = C^k$, $\theta_c^{k\to t} = 0$, the corresponding class prototypes are therefore aligned;
% (b) when $C^k < C^t$, $\theta_c^{k\to t} > 0$, class prototypes across steps are not aligned.
where $C^k$ and $C^t$ represent the number of classes at steps $k$ and $t$, respectively, with $C^k \leq C^t$.
Eq.~\ref{eq:misalignement} shows that $\theta_c^{k\to t} = 0$ when $C^t = C^k$, so that the corresponding class prototypes are aligned.
In contrast, $\theta_c^{k\to t} > 0$ when $C^t > C^k$, so that adding new classes always introduces a misalignment.

To align the prototypes when new classes are added, we compensate for the angular misalignment in Eq.~\ref{eq:misalignement} using a deterministic linear transformation, defined as the matrix $\mathbf{P}_{t,k} = \big[ \mathbf{P}_{k,k} \,|\, \mathbf{0} \big]$, where $\mathbf{0}$ is the $C^k \times (C^t - C^k)$ zero matrix.

The block $\mathbf{P}_{k,k}$, i.e., the centering projection in Eq.~\ref{eq:projection_prot}, acts on the first $C^k$ coordinates and maps the prototypes at step $t$ of the classes already present at step $k$, $\{\mathbf{v}^t_c\}_{c=1}^{C^k}$, onto the subspace spanned by their counterparts $\{\mathbf{v}^k_c\}_{c=1}^{C^k}$ at step $k$, thereby removing the angular misalignment introduced by the new classes.
The zero block $\mathbf{0}$ discards the components along the $C^t-C^k$ coordinates introduced at step $t$, which have no counterpart in the step-$k$ representation.
The following two propositions formalize this behavior.
% To align the prototypes when new classes are added, we compensate for the angular misalignment in Eq.~\ref{eq:misalignement} using a deterministic linear transformation, defined as the $C^k \times C^t$ matrix $\mathbf{P}_{t,k} = \big[  \mathbf{P}_{k,k} \,|\, \mathbf{0} \big]$, where $\mathbf{0}$ is the $C^k \times (C^t - C^k)$ zero matrix.
% % To obtain alignment of prototypes when new classes are added, we compensate for the angular misalignment in Eq.~\ref{eq:misalignement} using a deterministic linear transformation. 
% % \textcolor{red}{This transformation is defined as \( \mathbf{P}_{t,k} = \big[ \mathbf{P}_{k,k} | \mathbf{0} \big] \), where $\mathbf{0}$ is a zero matrix of dimension \(C^k\times (C^t - C^k)\).}
% It consists of two components:
% 
% (a) the matrix $\mathbf{V}^k$ acts on the first $C^k$ coordinates, mapping the prototypes at step $t$ of the classes already present at step $k$, $\{\mathbf{v}^t_c\}_{c=1}^{C^k}$, onto the subspace spanned by their step-$k$ counterparts, $\{\mathbf{v}^k_c\}_{c=1}^{C^k}$.
% }
% It removes the angular misalignment introduced by the new classes at step $t$, thereby aligning these prototypes across updates; 
% (b) the matrix $\mathbf{0}$ discards the components associated with the additional $C^t-C^k$ coordinates introduced at step $t$, which have no counterpart in the step-$k$ representation.
% The following two propositions formalize the properties of $\mathbf{P}_{t,k}$, where comparisons across update steps are performed after embedding the lower-dimensional prototypes into the higher-dimensional space via zero-padding.

\begin{figure*}[t]
  \centering
\includegraphics[width=1.0\linewidth]{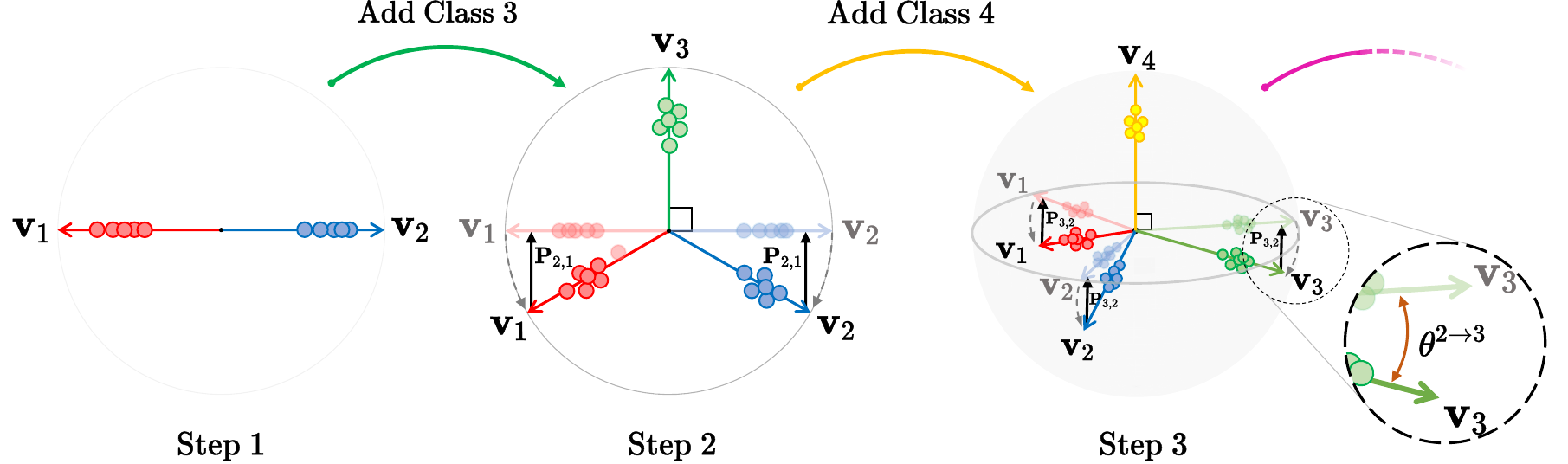}
% \vspace{-0.6cm}
  \caption{Deterministic linear transformation across model updates of the class prototypes in Eq.~\ref{eq:class_prototypes} to which the hyperspherical simplex representations of Eq.~\ref{eq:softmax_outputs_as_features} and Eq.~\ref{eq:logits_as_features} align.
From left to right: Step 1, a 1-simplex with two classes (red and blue); Step 2, a 2-simplex with a new class (green); Step 3, a 3-simplex with a new class (yellow).
Dashed arrows at steps 2 and 3 show how the class prototypes rotate to form the new simplex. The black dashed circle on the right shows the angular misalignment $\theta_c^{k\to t}$ of the class-3 prototype between steps 2 and 3.
Solid black arrows indicate the linear transformation $\mathbf{P}_{t,k}$, which aligns the updated prototypes of the previously existing classes with their counterparts at the earlier step.
  }
  \label{fig:projections_main}
\end{figure*}

\begin{proposition}[Proof in Appendix~\ref{sec:orth_new_class}]
\label{theo:orth_new_class}
When updating the model from step $k$ to step $t$, if the number of classes increases from $C^k$ to $C^t$, the prototypes associated with the newly introduced classes at step $t$ are orthogonal to the subspace spanned by the class prototypes at step $k$.
\end{proposition}

% \begin{proposition}[Proof in Appendix~\ref{sec:app_theo_projections}]\label{theo:projections_align}
% Let \( \mathbf{v}^{t}_c\) be the prototype vector for class \( c \) defined in Eq.~\ref{eq:class_prototypes} at step $t$, then \( \mathbf{u}^t_c = \mathbf{P}_{t,k} \mathbf{v}^{t}_c \) is aligned with the prototype $\mathbf{v}^k_c$ at step $k$.
% \end{proposition}

\begin{proposition}[Proof in Appendix~\ref{sec:app_theo_projections}]\label{theo:projections_align}
For every class $c$ already present at step $k$, i.e., $1 \leq c \leq C^k$, the transformed prototype $\mathbf{u}^t_c = \mathbf{P}_{t,k}\mathbf{v}^t_c$ is aligned with its counterpart $\mathbf{v}^k_c$ at step $k$.
\end{proposition}

Together, Prop.~\ref{theo:orth_new_class} and Prop.~\ref{theo:projections_align} explain the origin of the misalignment and how $\mathbf{P}_{t,k}$ removes it.
Prop.~\ref{theo:orth_new_class} shows that the prototypes associated with newly introduced classes at step $t$ lie in directions orthogonal to the subspace spanned by the prototypes at step $k$.
Since the step-$t$ prototypes remain centered at the origin, the updated prototypes of the existing classes acquire components along these orthogonal directions, resulting in the angular misalignment $\theta_c^{k\to t}$ defined in Eq.~\ref{eq:misalignement}.
Prop.~\ref{theo:projections_align} then shows that $\mathbf{P}_{t,k}$ removes these components and maps each updated old-class prototype onto the direction of its corresponding prototype at step $k$.
Fig.~\ref{fig:projections_main} illustrates this process: dashed arrows indicate the prototype change induced by class expansion, whereas solid arrows show how $\mathbf{P}_{t,k}$ restores the alignment.

% 
% Prop.~\ref{theo:orth_new_class} and Prop.~\ref{theo:projections_align} characterize how the deterministic transformation $\mathbf{P}_{t,k}$ restores prototype alignment under class expansion.
% Prop.~\ref{theo:orth_new_class} shows that the prototypes associated with newly introduced classes at step $t$ lie in directions orthogonal to the subspace spanned by the prototypes at step $k$. Consequently, class expansion introduces components orthogonal to the previous prototype subspace in the updated prototypes of the existing classes, resulting in the angular misalignment $\theta_c^{k\to t}$ defined in Eq.~\ref{eq:misalignement}.
% Prop.~\ref{theo:projections_align} then shows that $\mathbf{P}_{t,k}$ removes these components and maps each updated old-class prototype onto the direction of its corresponding prototype at step $k$.
% Fig.~\ref{fig:projections_main} illustrates this process: dashed arrows indicate the prototype rotation induced by class expansion, whereas solid arrows show the alignment obtained through $\mathbf{P}_{t,k}$.
% }

Since $\mathbf{P}_{t,k}$ is determined entirely by the simplex geometry and the number of classes at the two update steps, it can be computed without learning additional parameters or accessing any data.
It therefore provides a deterministic mechanism for restoring prototype alignment of the hyperspherical simplex representations in Eq.~\ref{eq:softmax_outputs_as_features} and Eq.~\ref{eq:logits_as_features} across model updates under class expansion.

We refer to the approach based on the hyperspherical simplex representation derived from softmax outputs in Eq.~\ref{eq:softmax_outputs_as_features} as Probability Simplex Projection (PSP), and to the approach based on the logit representation in Eq.~\ref{eq:logits_as_features} as Logits Simplex Projection (LSP).
When $C^t > C^k$, the transformation $\mathbf{P}_{t,k}$ is applied to the softmax outputs or logits of the updated model before $\ell_2$-normalization, so that the resulting features lie in the representation space of step $k$; when $C^t = C^k$, the representations in Eq.~\ref{eq:softmax_outputs_as_features} and Eq.~\ref{eq:logits_as_features} are used directly.
Appendix~\ref{sec:pseudo_code} provides pseudocode for gallery indexing with PSP or LSP and for comparing query and gallery representations across model updates.

% 
% Since $\mathbf{P}_{t,k}$ is determined entirely by the simplex geometry and the number of classes at the two update steps, it requires neither learnable parameters nor forward passes through either model.
% It therefore provides a deterministic, training-free mechanism for restoring prototype alignment of the hyperspherical simplex representations in Eq.~\ref{eq:softmax_outputs_as_features} and Eq.~\ref{eq:logits_as_features} across model updates under class expansion.
% }

% In practice, $\mathbf{P}_{t,k}$ is applied to the softmax outputs or logits of the updated model before $\ell_2$-normalization, so that the resulting features lie in the representation space of step $k$.

% 
% We refer to the approach based on the hyperspherical simplex representation derived from softmax outputs in Eq.~\ref{eq:softmax_outputs_as_features} as Probability Simplex Projection (PSP), and to the approach based on the logit representation in Eq.~\ref{eq:logits_as_features} as Logits Simplex Projection (LSP).
% Pseudocode for gallery indexing with PSP/LSP and for comparing query--gallery representations across model updates is provided in Appendix~\ref{sec:pseudo_code}.
% }

\subsection{Theoretical Analysis of the Backward-Compatibility Definition}\label{sec:theo_section}

In the following, we establish geometric conditions under which the backward-compatibility inequalities hold in expectation, and then analyze whether these conditions are satisfied by the proposed hyperspherical simplex representations.

\begin{theorem}[\textit{\textbf{Compatibility Theorem}}; Proof in 
Appendix~\ref{sec:app_theo_compatibility}.]
\label{theo:compatibility}
Consider hyperspherical representations obtained at update steps $k$ and $t$, with $t > k$.
Assume that (i) the features of each class $c$ are drawn independently from a von Mises--Fisher distribution {\rm vMF($\boldsymbol{\mu}_c,\kappa_c$)} whose mean direction $\boldsymbol{\mu}_c$ is the corresponding class prototype, and (ii) the concentration $\kappa_c$ of each class distribution increases from step $k$ to step $t$.
Then the compatibility inequalities of Eq.~\ref{eq:first} and Eq.~\ref{eq:second} are satisfied in expectation if (a) corresponding class prototypes are aligned across updates and (b) the angle between prototypes of different classes is greater than $\pi/2$.
\end{theorem}

The vMF distribution in assumption (i) is employed as a tractable modeling assumption rather than as an exact characterization of the learned distributions: it yields closed-form expressions for the expected cosine distances in Eq.~\ref{eq:first} and Eq.~\ref{eq:second} in terms of the prototype directions and concentration parameters (Lemma~\ref{theo:lemma} in Appendix~\ref{sec:app_theo_compatibility}), which enables the analytical characterization in Theorem~\ref{theo:compatibility}.
It provides a standard probabilistic model for directional data on the unit hypersphere \citep{Banerjee2005ClusteringOT}, and hyperspherical and angular representations have been extensively investigated in deep classification and face recognition \citep{liu2016large,Liu2017CVPR,DBLP:conf/cvpr/DengGXZ19}.
More broadly, vMF-based formulations have been employed to define hyperspherical latent-variable models \citep{davidson2018hyperspherical}, learn discriminative hyperspherical representations \citep{hasnat2017mises,scott2021mises}, model normalized contrastive and hyperspherical features \citep{du2024probabilistic,bai2024hypo}, and characterize class-conditional latent structure in theoretical analyses of cross-entropy learning \citep{reizinger2025crossentropy}.
Remark~\ref{rem:role_vmf} further clarifies the specific role of the vMF assumption and its first-moment characterization in the proof of Theorem~\ref{theo:compatibility}.

Assumption (ii), that class representations become increasingly concentrated around their prototypes across model updates, is supported by several complementary observations.
Scaling-law studies show that increasing model capacity or training data generally improves predictive, transfer, and retrieval performance \citep{fang2024scaling,prato2021scalingfewshot,zhai2022scaling}, motivating the model-update scenarios considered here, in which updated models are more capable or trained on larger datasets.
Neural Collapse analyses provide a more direct geometric connection: within-class variability progressively decreases during training as features concentrate around their class means \citep{papyan2020prevalence,yaras2022neural, pmlr-v235-jiang24i}, and stronger collapse with reduced relative within-class variability is observed as network width and depth increase \citep{yaras2022neural,sukenik2026neural}.
Under the vMF modeling assumption, such a reduction in angular dispersion around the class prototype corresponds to an increase in the concentration parameter $\kappa$.

These observations primarily concern penultimate-layer features, but they can be related to the output representations used by PSP and LSP.
For a linear classifier, variations in the backbone representations induce bounded variations in the corresponding logits \citep{gouk2021regularisation}, and the progressive alignment between class means and classifier weights observed during training \citep{papyan2020prevalence} relates the geometry of the penultimate-layer representations to that of the logits.
% Moreover, \citet{nair2026softmax} shows that the softmax function is $1/2$-Lipschitz, which bounds the within-class variability of the softmax outputs by that of the logits.
In Appendix~\ref{sec:concentration}, we provide empirical evidence on CIFAR-10 that both logits and softmax outputs become increasingly concentrated around their class prototypes across model updates.

The aligned regular-simplex geometry centered at the origin, on which the proposed representations are built, satisfies the geometric conditions of Theorem~\ref{theo:compatibility}. The following corollary formalizes this result for any hyperspherical representation exhibiting this geometry.

% As anticipated in Sec.~\ref{sec:prelims_and_motivation}, the aligned regular-simplex geometry on which the proposed representations are built is a natural candidate for satisfying the conditions of Theorem~\ref{theo:compatibility}. The following corollary formalizes this for any hyperspherical representation with this specific geometry.

\begin{corollary}[\textit{\textbf{Compatibility of Aligned Regular-Simplex Hyperspherical Representations}}; Proof in 
Appendix~\ref{sec:app_theo_compatibility}.]
\label{cor:regulsimplex_compatibility}
Under the assumptions of Theorem~\ref{theo:compatibility}, hyperspherical representations whose class prototypes form a regular simplex centered at the origin and remain aligned across updates satisfy the compatibility inequalities of Eq.~\ref{eq:first} and Eq.~\ref{eq:second} in expectation.
\end{corollary}

This follows from the geometry of a regular simplex centered at the origin, where the inner product between any two distinct normalized class prototypes is negative, implying a pairwise angle greater than $\pi/2$~\citep{pernici2021regular}. Hence, once corresponding class prototypes are aligned across updates, both geometric conditions required by Theorem~\ref{theo:compatibility} hold.
Remark~\ref{rem:simplex_etf} in Appendix~\ref{sec:app_theo_compatibility} further clarifies the relation between this geometry and the Simplex ETF configuration arising in Neural Collapse.

% The following corollary demonstrates that the geometric conditions established in Theorem~\ref{theo:compatibility} are satisfied by hyperspherical representations whose class prototypes form a regular simplex centerd at the origin and remain aligned across updates.

% \begin{corollary}[\textit{\textbf{Compatibility of Aligned Regular-Simplex Hyperspherical Representations}}; Proof in 
% Appendix~\ref{sec:app_theo_compatibility}.]
% \label{cor:regulsimplex_compatibility}
% Under the assumptions of Theorem~\ref{theo:compatibility}, hyperspherical representations whose class prototypes form a regular simplex centered at the origin and remain aligned across updates satisfy the compatibility inequalities of Eq.~\ref{eq:first} and Eq.~\ref{eq:second} in expectation.
% \end{corollary}

% This follows from the regular-simplex geometry: the inner product between any two distinct normalized class prototypes is negative, implying a pairwise angle greater than $\pi/2$. Hence, once corresponding class prototypes are aligned across updates, both geometric conditions required by Theorem~\ref{theo:compatibility} are satisfied. Remark~\ref{rem:simplex_etf} further clarifies the relation between this condition and the Simplex ETF geometry arising in Neural Collapse.

% The proposed PSP and LSP representations defined in Eq.~\ref{eq:softmax_outputs_as_features} and Eq.~\ref{eq:logits_as_features}, respectively, instantiate this aligned regular-simplex geometry.

The proposed PSP and LSP representations satisfy the geometric hypotheses of Corollary~\ref{cor:regulsimplex_compatibility}.
% We now verify that PSP and LSP satisfy the hypotheses of Corollary~\ref{cor:regulsimplex_compatibility}.
Their class prototypes form a regular simplex centered at the origin: for PSP, by construction (Sec.~\ref{sec:psp}); for LSP, exactly under the logit-UFM assumptions of Prop.~\ref{theo:softmax_affects_logits}, with the emergence of this geometry in neural networks assessed empirically through the Neural Collapse analysis of Sec.~\ref{sec:logits_theory}.
Corresponding prototypes remain aligned across updates when the class set is unchanged, while under class expansion the transformation $\mathbf{P}_{t,k}$ restores their alignment (Prop.~\ref{theo:projections_align}).
Therefore, under the assumptions of Theorem~\ref{theo:compatibility}, PSP and LSP satisfy the compatibility inequalities of Eq.~\ref{eq:first} and Eq.~\ref{eq:second} in expectation.
In Appendix~\ref{app:psp_ablation}, we empirically isolate the two operations underlying PSP in a class-expansion update, showing that positive average margins for both compatibility inequalities are obtained only when the centering projection and the angular correction are combined.

The theoretical result is not specific to PSP and LSP: Corollary~\ref{cor:regulsimplex_compatibility} also applies, for instance, to representations learned with pre-allocated $d$-simplex fixed classifiers \citep{biondi2023cores,biondi2024stationary,biondiallyouneed}, whose class prototypes are fixed a priori and thus remain aligned across updates. From a practical perspective, however, PSP and LSP do not require a shared pre-allocated fixed classifier across model updates.

\subsection{Scalability via Top-\texorpdfstring{$k$}{k} Sparsification}
\label{sec:dimensionality_reduction_topk}

Hyperspherical simplex representations derived from softmax outputs (Eq.~\ref{eq:softmax_outputs_as_features} for PSP) or logits (Eq.~\ref{eq:logits_as_features} for LSP) have dimensionality equal to the number of classes.
Although training on a larger number of classes can improve the quality and transferability of learned representations \citep{zhu2022learning, DBLP:journals/tmlr/HongCFCL24}, this class-dependent dimensionality poses a limitation in large-scale retrieval scenarios: the memory required to store the gallery scales linearly with both the number of gallery samples and the number of classes, potentially becoming prohibitive at scale.
To address this, we construct sparse representations by retaining only the highest-ranked components of each feature vector, a choice motivated by the top-$k$ calibration of the softmax cross-entropy loss and the rank-preserving property of the softmax function \citep{lapin2016loss, DBLP:conf/icml/YangK20}.

Following the theoretical analysis of \citet{lapin2016loss}, the Bayes-optimal classifier under the softmax cross-entropy loss satisfies:
\begin{equation}\label{eq:lapin}
f_y^*(x) =
\begin{cases}
\log p_y(x) + \log \alpha & \text{if } p_y(x) > 0, \\[6pt]
-\infty & \text{otherwise},
\end{cases}
\end{equation}

where $\alpha > 0$ is a class-independent constant and $p_y(x) = P(Y=y \mid X=x)$ denotes the posterior probability of label $y$ given input $x$.
Since the logarithm is strictly increasing, the optimal scores are a monotonic reparameterization of the class posterior probabilities and preserve their ordering, which underlies the top-$k$ calibration of the softmax cross-entropy loss \citep{lapin2016loss}.
Moreover, the softmax function is rank-preserving \citep{DBLP:conf/icml/YangK20}, i.e., $z_i > z_j \iff \sigma(\mathbf{z})_i > \sigma(\mathbf{z})_j$, so that logits and softmax outputs induce the same class ranking.
This ranking is also preserved by the operations defining PSP and LSP: the centering projection $\mathbf{P}_{k,k}$ subtracts the same constant from all entries of a softmax output, and $\ell_2$-normalization rescales all entries by a positive factor. Consequently, the largest entries of PSP and LSP correspond to the highest-ranked classes.

Accordingly, we adopt a top-$k$ sparsification strategy for both PSP and LSP, applied after $\ell_2$-normalization: we retain the $k$ largest entries of the normalized feature vector $\mathbf{h}^k$ and set the remaining entries to zero.
% The resulting sparse representations can be directly processed by approximate nearest-neighbor search libraries designed for sparse high-dimensional data (e.g., PySparNN), supporting efficient search over large galleries.
The resulting sparse representations can be directly processed by approximate nearest-neighbor search libraries designed for sparse high-dimensional data (e.g., PySparNN \citep{pysparnn} and Seismic \citep{bruch2024efficient, bruch2024pairing}), supporting efficient search over large galleries \citep{simhadri2025results}.
In Sec.~\ref{sec:when_experimental_results}, we empirically evaluate top-$k$ sparsification for both PSP and LSP, showing that high levels of sparsity can be achieved with negligible degradation in retrieval performance while preserving backward-compatibility.

\section{Experimental Results}\label{sec:when_experimental_results}

Building on the theoretical analysis, in this section we examine whether the backward-compatibility properties of PSP and LSP are reflected empirically across various model-update scenarios, under established training protocols and retrieval-based evaluation procedures.

\subsection{Evaluation Setting}\label{sec:evaluation_setting}

Following prior work on model compatibility \citep{shen2020towards, ijcai2022uniBCT, Meng_2021_ICCV, jaecklefastfill}, we adopt image retrieval as the evaluation task, where the goal is to retrieve images from a gallery $\mathcal{G} = \{(\mathbf{x}_i, y_i)\}_{i=1}^{N_g}$ using queries from a query set $\mathcal{Q} = \{(\mathbf{x}_i, y_i)\}_{i=1}^{N_q}$.
Since verifying the inequalities of Def.~\ref{def:compatibility} is impractical at scale (Sec.~\ref{sec:prelims_and_motivation}), we follow standard practice and adopt the Empirical Compatibility Criterion introduced by~\citet{shen2020towards}:
\begin{equation}\label{eq:multistepecc}
M \big( \phi^t, \phi^k; \mathcal{Q}, \mathcal{G} \big) >
M \big( \phi^k, \phi^k; \mathcal{Q}, \mathcal{G} \big),
\qquad t > k,
\end{equation}
where $\phi^k$ denotes the model at step $k$ and $M$ denotes a retrieval performance metric. 
The term $M\big(\phi^t,\phi^k;\mathcal{Q},\mathcal{G}\big)$ denotes the \textit{cross-model accuracy}, obtained when query representations are extracted using the updated model $\phi^t$ and gallery representations using the previous model $\phi^k$.
Conversely, $M\big(\phi^k,\phi^k;\mathcal{Q},\mathcal{G}\big)$ denotes the \textit{same-model accuracy}, obtained when both query and gallery representations are extracted using the same model $\phi^k$.

Following prior work~\citep{ramanujan2022forward, zhou2023bt, ricci2024backward, jaecklefastfill, ricci2025orthogonality}, we use Cumulative Matching Characteristic (CMC@$1$) as the retrieval performance metric $M$, and the full test set of each dataset is used as both the query set and gallery set.
Retrieval is performed using cosine distance between query and gallery representations, excluding the query sample itself from the gallery to avoid a trivial match. CMC@$1$ is computed as the fraction of queries for which the gallery sample ranked first has the same class label as the query.

To evaluate compatibility across multiple model updates, we construct a compatibility matrix following~\citet{biondi2023cores, biondi2023cl2r, bui2025learning}:
\begin{equation}\label{eq:compatibility_matrix}
    \mathbf{C}_{t,k} =
    \begin{cases}
      0 & \text{if } t < k, \\
      M\big(\phi^k,\phi^k;\mathcal{Q},\mathcal{G}\big) & \text{if } t = k, \\
      M\big(\phi^t,\phi^k;\mathcal{Q},\mathcal{G}\big) & \text{if } t > k.
    \end{cases}
\end{equation}

This compatibility matrix summarizes retrieval performance across all valid pairs of model updates, extending the evaluation beyond the two or three update steps commonly considered in previous studies~\citep{shen2020towards, zhou2023bt, zhang2021hot}.
Compatibility is assessed by comparing each cross-model entry $\mathbf{C}_{t,k}$, with $t>k$, against the corresponding same-model entry $\mathbf{C}_{k,k}$ according to the Empirical Compatibility Criterion in Eq.~\ref{eq:multistepecc}.

To summarize the overall behavior across multiple updates, we adopt the metrics defined by~\citet{biondi2023cores} and~\citet{biondiallyouneed}, which aggregate the entries of the compatibility matrix into distinct measures, conceptually analogous to metrics used in continual learning~\citep{lopez2017gradient}:

\begin{itemize}

\item \textit{Average Compatibility (AC)}, which measures the fraction of compatible cross-model comparisons over $T$ update steps:
\mbox{$
AC =
\frac{2}{T(T-1)}
\sum_{t=2}^{T}
\sum_{k=1}^{t-1}
\mathds{1}\!\left(\mathbf{C}_{t,k} > \mathbf{C}_{k,k}\right),
$}
where $\mathds{1}$ denotes the indicator function.

\item \textit{Average Accuracy (AA)}, which measures the average retrieval performance across all same-model and cross-model comparisons over $T$ update steps, thus capturing the overall retrieval accuracy regardless of compatibility:
$AA = \frac{2}{T(T+1)} \sum_{t=1}^{T} \sum_{k=1}^{t} \mathbf{C}_{t,k}$.

\item \textit{Average Compatibility Accuracy (ACA)}, which jointly captures retrieval performance and compatibility by averaging the cross-model accuracy over all update pairs, while assigning zero contribution to incompatible comparisons:
\mbox{$
\mathit{ACA} =
\frac{2}{T(T-1)}
\sum_{t=2}^{T}
\sum_{k=1}^{t-1}
\mathbf{C}_{t,k}
\mathds{1}\!\left(\mathbf{C}_{t,k} > \mathbf{C}_{k,k}\right).
$}
\end{itemize}

\begin{table*}[t]
 \caption{Comparative analysis of compatibility performance under the ``Extended Classes'' benchmark on CIFAR-100 for multiple updates. Dark and light blue indicate respectively the highest and the second-highest value for each metric.
 }
 \small
 \label{tab:cifar}
    \centering 
        \vspace{-0.25cm}
        \begin{tabular}{l ccc ccc ccc ccc
        } 
            \toprule 
            \multirow{4}{*}{\shortstack{\textsc{Method}}} &
            \multicolumn{3}{c}{2 steps} &
            \multicolumn{3}{c}{5 steps} &
            \multicolumn{3}{c}{20 steps} 
            & \multicolumn{3}{c}{50 steps} 
            \\   
            \cmidrule(lr){2-4}  \cmidrule(lr){5-7} 
            \cmidrule(lr){8-10}
            \cmidrule(lr){11-13} 
            
            &
            {$AC$} &
            {$AA$ } &
            {$ACA$ } &
            {$AC$} &
            {$AA$ } &
            {$ACA$ } &
            {$AC$} &
            {$AA$ } & 
            {$ACA$ } &
            {$AC$} &
            {$AA$ } & 
            {$ACA$ }
            \\ 
            \midrule
            Baseline & 0 & 29.63 & 0 & 0 & 13.60 & 0 & 0 & 04.22 & 0  & 0 & 02.29 & 0\\ 
            BCT & \textcolor[RGB]{0,0,139}{\textbf{1}} & 40.64 & 35.51 & 0 & \textcolor[RGB]{65,105,225}{\textbf{31.03}} & 0& 0 & 22.66 & 0 %
            & 0.01 & 15.78 & 0.84 \\ 
            RACT & \textcolor[RGB]{0,0,139}{\textbf{1}} & \textcolor[RGB]{0,0,139}{\textbf{44.19}} & \textcolor[RGB]{0,0,139}{\textbf{45.09}} & 0.60 & 30.43 & 16.53 &  0.05 & 21.61 & 2.42 & 0 & 16.71 & 0\\
            AdvBCT & 0 & 35.32 & 0  & 0.40 & 26.13 & 11.67 & 0.02 & 19.79 & 0.19  & 0 & 14.70 & 0.03\\ 
            LCE & \textcolor[RGB]{0,0,139}{\textbf{1}} & \textcolor[RGB]{65,105,225}{\textbf{43.48}} & \textcolor[RGB]{65,105,225}{\textbf{40.71}}   & 0.10 & \textcolor[RGB]{0,0,139}{\textbf{32.63}} & 04.63 & 0 & 20.95 & 0.25  & 0 & 13.81 & 0.03 \\ 
            CoReS &  \textcolor[RGB]{0,0,139}{\textbf{1}} & 40.94 & 35.21 &  0 & 29.63 &0 & 0.12 & \textcolor[RGB]{0,0,139}{\textbf{24.88}} & 4.61 & 0.11 & \textcolor[RGB]{0,0,139}{\textbf{23.50}} & 3.91 \\

            \midrule
            PSP & \textcolor[RGB]{0,0,139}{\textbf{1}} & 36.31 & 29.05  & \textcolor[RGB]{0,0,139}{\textbf{0.90}} & 26.04 & \textcolor[RGB]{65,105,225}{\textbf{21.76}}  & \textcolor[RGB]{0,0,139}{\textbf{0.52}} & 20.56 & \textcolor[RGB]{65,105,225}{\textbf{12.99}}    & \textcolor[RGB]{0,0,139}{\textbf{0.39}} & 19.42 & \textcolor[RGB]{65,105,225}{\textbf{10.76}} \\
            LSP & \textcolor[RGB]{0,0,139}{\textbf{1}} & 41.14 & 36.38  & \textcolor[RGB]{65,105,225}{\textbf{0.70}} & 30.36 & \textcolor[RGB]{0,0,139}{\textbf{21.91}}  & \textcolor[RGB]{65,105,225}{\textbf{0.44}} & \textcolor[RGB]{65,105,225}{\textbf{24.11}} & \textcolor[RGB]{0,0,139}{\textbf{13.26}}   & \textcolor[RGB]{65,105,225}{\textbf{0.36}}  & \textcolor[RGB]{65,105,225}{\textbf{22.68}} & \textcolor[RGB]{0,0,139}{\textbf{11.25}} \\

            \bottomrule
        \end{tabular} 
\end{table*}

\subsection{Standard Compatibility Benchmarks} \label{sec:experiments_extended_classes}

In this section, we compare PSP and LSP with established backward-compatibility methods on standard compatibility benchmarks. Following prior work~\citep{shen2020towards, zhou2023bt, Meng_2021_ICCV, biondi2023cores, pan2023boundary}, we evaluate three update scenarios: \emph{extended classes}, where new classes are introduced across model updates; \emph{extended backbone}, where the model architecture is upgraded; and \emph{extended classes and backbone}, where both changes occur simultaneously.

Following~\citet{shen2020towards} and subsequent literature, unless otherwise specified, each updated model is randomly initialized and trained from scratch rather than initialized from the previous model.
This protocol removes parameter inheritance across model updates, ensuring that observed compatibility cannot be attributed to warm-starting from a previously learned representation.
Models are trained using classification as the surrogate task for representation learning.
Implementation details for the standard compatibility benchmarks are provided in Appendix~\ref{sec:implementation_details}.

We compare against the following backward-compatibility methods: BCT~\citep{shen2020towards}, RACT~\citep{zhang2021hot}, AdvBCT~\citep{pan2023boundary}, LCE~\citep{Meng_2021_ICCV}, and CoReS~\citep{biondi2023cores}.
We also include a Baseline based on the $\ell_2$-normalized encoder features without any compatible method applied. When the encoder feature dimensionality differs across model updates, the lower-dimensional representation is zero-padded to match the dimensionality of the higher-dimensional one before computing cross-model similarities.
For reproducibility, we use the official publicly available implementations of all compared methods.
All experiments are conducted on two NVIDIA A100 40\,GB GPUs.

\begin{table*}[t]
 \caption{Comparative analysis of compatibility performance under the ``Extended Classes'' benchmark on ImageNet-1K and Google Landmarks v2 for multiple updates. Dark and light blue indicate respectively the highest and the second-highest value for each metric.
 }
 \small
 \centering
 \setlength{\tabcolsep}{4pt}
 \label{tab:imagenet}
        \vspace{-0.2cm}

        \begin{tabular}{l ccc ccc c@{\hspace{0.4cm}} ccc ccc}
            \toprule 
            \multirow{7}{*}{\shortstack{\textsc{Method}}}
            & \multicolumn{6}{c}{ImageNet-1K} 
            & 
            & \multicolumn{6}{c}{Google Landmark v2} 
            \\
            \cmidrule(lr){2-7} 
            \cmidrule(lr){9-14} 
        
            & \multicolumn{3}{c}{2 steps} 
            & \multicolumn{3}{c}{5 steps} 
            &  %
            & \multicolumn{3}{c}{2 steps} 
            & \multicolumn{3}{c}{5 steps} 
            \\   
            \cmidrule(lr){2-4} \cmidrule(lr){5-7} 
            \cmidrule(lr){9-11} \cmidrule(lr){12-14} 
            
            &
            {$AC$} &
            {$AA$} &
            {$ACA$} &
            {$AC$} &
            {$AA$} &
            {$ACA$} 
            & %
            & 
            {$AC$} &
            {$AA$} &
            {$ACA$} &
            {$AC$} &
            {$AA$} & 
            {$ACA$}
            \\ 

            \midrule
            Baseline & 0 & 37.52 & 0 & 0 & 16.18 & 0 && 0 & 07.66 & 0 & 0 & 07.87 & 0 \\ 
            BCT  & 
            \textcolor[RGB]{0,0,139}{\textbf{1}} & 48.81 & \textcolor[RGB]{65,105,225}{\textbf{45.86}} & 0.20 & 34.96 & 5.28 && \textcolor[RGB]{0,0,139}{\textbf{1}} & 09.37 & 08.22 & 0 & 08.46 & 0  \\ 
            RACT &  \textcolor[RGB]{0,0,139}{\textbf{1}} & \textcolor[RGB]{0,0,139}{\textbf{55.17}} & \textcolor[RGB]{0,0,139}{\textbf{51.16}} & 0.10 &  \textcolor[RGB]{65,105,225}{\textbf{40.94}} & 6.20 &  & 0 & \textcolor[RGB]{65,105,225}{\textbf{11.01}} & 0 & 0 & 09.99 & 0   \\
            AdvBCT  & 0 &  \textcolor[RGB]{65,105,225}{\textbf{53.07}} & 0 & 0 &  \textcolor[RGB]{0,0,139}{\textbf{43.91}} & 0 & & \textcolor[RGB]{0,0,139}{\textbf{1}} & \textcolor[RGB]{0,0,139}{\textbf{11.12}} & \textcolor[RGB]{0,0,139}{\textbf{09.67}} & 0 &  \textcolor[RGB]{0,0,139}{\textbf{10.73}} & 0 \\ 
            LCE & 0 & 47.22 & 0 & 0 & 32.67 & 0 & &  \textcolor[RGB]{0,0,139}{\textbf{1}} & 09.41 & 08.03 & 0.40 & \textcolor[RGB]{65,105,225}{\textbf{10.31}} & 3.80 \\ 
            CoReS  & 0 & 46.11 & 0 & 0 & 37.67 & 0 && \texttt{-} & \texttt{-} & \texttt{-} & \texttt{-} & \texttt{-} & \texttt{-}  \\ 
            \midrule
            PSP & \textcolor[RGB]{0,0,139}{\textbf{1}} & 49.10 & 39.80 & \textcolor[RGB]{0,0,139}{\textbf{1}} & 35.20 & \textcolor[RGB]{0,0,139}{\textbf{31.09}} && \textcolor[RGB]{0,0,139}{\textbf{1}} & 08.74 & 08.24 & \textcolor[RGB]{0,0,139}{\textbf{0.90}} & 08.53 & 7.21 \\
            \hspace{0.05cm} $\hookrightarrow$ w top-256 &  \textcolor[RGB]{0,0,139}{\textbf{1}} & 49.11 & 39.81 &  \textcolor[RGB]{0,0,139}{\textbf{1}} & 35.20 & \textcolor[RGB]{0,0,139}{\textbf{31.09}} && \textcolor[RGB]{0,0,139}{\textbf{1}} & 08.74 & 08.24 & \textcolor[RGB]{0,0,139}{\textbf{0.90}} & 08.53 & 7.21 \\
            LSP & \textcolor[RGB]{0,0,139}{\textbf{1}} & 50.72 & 45.62 & 0.50 & 38.03 & 14.28  && \textcolor[RGB]{0,0,139}{\textbf{1}} & 09.31 & \textcolor[RGB]{65,105,225}{\textbf{08.36}} & \textcolor[RGB]{65,105,225}{\textbf{0.80}} & 09.36 & \textcolor[RGB]{0,0,139}{\textbf{7.34}} \\ 
            \hspace{0.05cm} $\hookrightarrow$ w top-256 &  \textcolor[RGB]{0,0,139}{\textbf{1}} &  50.34 & 45.84 &  \textcolor[RGB]{65,105,225}{\textbf{0.70}} &  37.96 &  \textcolor[RGB]{65,105,225}{\textbf{24.85}} && \textcolor[RGB]{0,0,139}{\textbf{1}} & 09.30 & 08.35 & \textcolor[RGB]{65,105,225}{\textbf{0.80}} & 09.34 & \textcolor[RGB]{65,105,225}{\textbf{7.31}} \\
            
            \bottomrule
        \end{tabular} 

\end{table*}

\subsubsection{Extended Classes Benchmark}\label{sec:extended_classes}

In the extended-classes benchmark, the training set of each dataset is partitioned into $T$ sequential tasks, each introducing a disjoint and equal number of new classes.
At each update step $t$, the model is trained from scratch using the cumulative training data from tasks $1$ through $t$.
Following~\citet{zhou2023bt, ricci2024backward, jaecklefastfill, zhou2022forward}, retrieval is evaluated at each update step over the entire test set, including classes that have not yet been observed during training.
This results in a semi-open-set evaluation scenario, in which the training class set expands across updates while retrieval is performed over the complete set of classes.
Tab.~\ref{tab:cifar} reports \emph{AC}, \emph{AA}, and \emph{ACA} on CIFAR-100~\citep{Krizhevsky2009LearningML} using a ResNet-18 backbone.
Tab.~\ref{tab:imagenet} reports the same metrics on larger-scale datasets: ImageNet-1K~\citep{russakovsky2015imagenet} using a ResNet-50 backbone and Google Landmarks v2~\citep{weyand2020google} using a ResNet-18 backbone initialized from ImageNet-1K pretrained weights at every update step.
Overall, Tab.~\ref{tab:cifar} and Tab.~\ref{tab:imagenet} show that PSP and LSP achieve competitive retrieval performance while consistently providing strong compatibility across model updates.
The advantage of our proposed representations over the compared approaches becomes increasingly pronounced as the number of update steps increases, as reflected by the compatibility metrics in the 50-step setting on CIFAR-100 and the 5-step setting on ImageNet-1K.
The compatibility matrices for the 5-step CIFAR-100 experiment are reported in Fig.~\ref{fig:comp_matrix_cifar_5step} in Appendix~\ref{sec:comp_CIFAR-100}, providing a visual summary of retrieval performance across model updates.

In the 2-step setting, where compatibility is evaluated over a single model-update pair, several of the compared approaches achieve strong performance.
On CIFAR-100, BCT, RACT, LCE, CoReS, PSP, and LSP all satisfy the empirical compatibility criterion ($AC=1$), with RACT achieving the highest retrieval performance among the compared methods.
A similar behavior is observed on ImageNet-1K, where BCT, RACT, PSP, and LSP satisfy the compatibility criterion, while RACT obtains the highest \emph{AA} and \emph{ACA}.
On Google Landmarks v2, BCT, AdvBCT, LCE, PSP, and LSP satisfy the criterion, whereas RACT does not.
In contrast, the Baseline based directly on the encoder features does not achieve compatibility in any evaluated setting, showing that independently trained encoder representations are not naturally compatible across model updates.

As the number of updates increases, the differences among the methods become more pronounced.
On CIFAR-100, the compatibility of BCT, RACT, AdvBCT, and LCE drops to nearly zero in the 20- and 50-step settings.
CoReS is more robust than these pairwise compatibility approaches over longer update sequences and retains comparatively high \emph{AA}; however, its \emph{AC} and \emph{ACA} remain substantially below those of PSP and LSP.
For example, in the 50-step setting, CoReS achieves $AC=0.11$ and $ACA=3.91$, compared with $AC=0.39$ and $ACA=10.76$ for PSP and $AC=0.36$ and $ACA=11.25$ for LSP.
A similar trend is observed in the 5-step ImageNet-1K benchmark: although several competing methods retain high \emph{AA}, their compatibility frequency is low or zero, whereas PSP achieves $AC=1$ and $ACA=31.09$, and LSP achieves $AC=0.50$ and $ACA=14.28$.
On Google Landmarks v2, PSP and LSP similarly attain the highest compatibility values in the 5-step setting, with $AC=0.90$ and $AC=0.80$, respectively, compared with $AC=0.40$ for LCE and $AC=0$ for the remaining evaluated baselines.

One possible explanation for the degradation observed for several of the compared methods as the number of updates increases is that regularization- and mapping-based approaches typically establish compatibility only with respect to the previous model.
When such compatibility is enforced sequentially across multiple updates, deviations and misalignment with respect to earlier model versions may accumulate because compatibility with all previous representations is not directly enforced.
Rather than sequentially enforcing compatibility with their predecessors, PSP and LSP exploit a regular-simplex geometry that provides a common reference across all model updates, consistently with Corollary~\ref{cor:regulsimplex_compatibility}.

Among the compared approaches, CoReS~\citep{biondi2023cores} is conceptually closer to PSP and LSP, as it also exploits a regular-simplex geometry as reference across model updates.
The main practical distinction is that CoReS requires the classifier geometry and the number of prototypes to be pre-allocated before training, including prototypes reserved for classes that may be introduced in future updates~\citep{pernici2021class,biondi2023cores}.
The outputs reserved for future classes participate in the cross-entropy objective from the beginning of training and act as negative classes for current samples~\citep{pernici2021class,biondi2023cores}, thereby constraining representation learning with respect to a prototype geometry and class capacity specified in advance.
This fixed pre-allocated prototype assignment may limit adaptation to the observed class structure \citep{saadabadi2024hyperspherical}.
In contrast, PSP and LSP derive their reference geometry directly from the output space of each model.
Furthermore, CoReS scales poorly with the number of classes, as its memory requirements, driven by both the pre-allocated fixed classifier and the additional linear layer mapping the backbone features to the simplex representation space, grow substantially with the class count. For Google Landmarks v2, these requirements made training computationally infeasible under our experimental configuration, and CoReS results are therefore not reported in Tab.~\ref{tab:imagenet}.

Notably, across the different numbers of update steps reported in Tab.~\ref{tab:cifar} and Tab.~\ref{tab:imagenet}, LSP consistently achieves higher \emph{AA} than PSP, while its \emph{AC} is generally equal to or lower than that of PSP.
This behavior is consistent with the larger within-class spread and lower prototype alignment of LSP observed in the NC analysis of Sec.~\ref{sec:logits_theory}. 
As a result of this trade-off, LSP often achieves comparable or higher \emph{ACA}. However, since incompatible comparisons contribute zero to \emph{ACA}, PSP achieves a higher \emph{ACA} when it is compatible in a substantially larger fraction of update pairs, as in the 5-step ImageNet-1K setting ($AC=1$ for PSP vs.\ $AC=0.50$ for LSP).

\begin{wraptable}{r}{0.5\linewidth}
  \vspace{-1.2\baselineskip} % tighten if needed; remove if it looks cramped
  \centering
  \footnotesize
\caption{Comparison of PSP and LSP across different top-$k$ values in terms of $AC$, $AA$, and $ACA$, in the 5-step ImageNet-1K setting of the extended-classes benchmark.}
  \vspace{-0.1cm}
  \label{tab:topk}
  \begin{tabular}{@{}ccccccc@{}}
    \toprule
    \multirow{2}{*}{Top-$k$} & \multicolumn{3}{c}{PSP} & \multicolumn{3}{c}{LSP} \\
    \cmidrule(lr){2-4}\cmidrule(lr){5-7}
    with $k=$ & $AC$ & $AA$ & $ACA$ & $AC$ & $AA$ & $ACA$ \\
    \midrule
    256  & 1 & 35.20 & 31.09 & 0.7 & 37.96 & 24.85 \\
    384  & 1 & 35.20 & 31.09 & 0.7 & 38.00 & 24.87 \\
    512  & 1 & 35.20 & 31.09 & 0.5 & 38.03 & 14.24 \\
    768  & 1 & 35.20 & 31.09 & 0.5 & 38.01 & 14.28 \\
    896  & 1 & 35.20 & 31.09 & 0.5 & 38.03 & 14.28 \\
    1000 & 1 & 35.20 & 31.09 & 0.5 & 38.03 & 14.28 \\
    \bottomrule
  \end{tabular}
  \vspace{-\baselineskip} % adjust/remove to taste
\end{wraptable}

The rows labeled ``$\!\xrightarrow{}$ w top-256'' in Tab.~\ref{tab:imagenet} report the performance of PSP and LSP after applying the top-$k$ sparsification described in Sec.~\ref{sec:dimensionality_reduction_topk}.
Top-$256$ sparsification leaves the PSP results essentially unchanged across the evaluated benchmarks, while also preserving the retrieval accuracy of LSP. In the 5-step ImageNet-1K setting, sparsification increases the \emph{AC} of LSP from $0.50$ to $0.70$ and its \emph{ACA} from $14.28$ to $24.85$, with only a negligible change in \emph{AA} (from $38.03$ to $37.96$). Consistently, Tab.~\ref{tab:topk} shows that retrieval accuracy remains essentially stable over the evaluated range of $k$, while the observed compatibility behavior is preserved. 

The memory benefit of sparsification becomes particularly relevant when the number of classes is large.
For a gallery of $N_g$ representations of dimension $C$, dense FP32
storage requires approximately $4N_gC$ bytes, whereas storing the
top-$k$ representation using one 32-bit value and one 32-bit index per
retained coordinate requires approximately $8N_gk$ bytes, excluding
minor sparse-format metadata overhead.
For Google Landmarks v2, where the representation dimensionality exceeds $81$k at the final update, this corresponds to approximately $318$\,KiB per dense representation and only $2$\,KiB with $k=256$, yielding approximately a $159\times$ reduction in gallery storage. Sparse representations also avoid dense $C$-dimensional similarity computation; the resulting search-time benefit depends on the adopted indexing and search backend~\citep{simhadri2025results}.

\begin{table*}[t]
  \caption{Comparative analysis of compatibility performance under the ``Extended Backbone'' and ``Extended Classes and Backbone'' benchmarks on CIFAR-100 and ImageNet-1K. Dark and light blue indicate respectively the highest and the second-highest value for each metric.}
  \label{tab:combined_benchmarks}
  \small
  \centering
  \vspace{-0.25cm}
  \begin{tabular}{l ccc ccc ccc ccc}
    \toprule
    \multirow{7}{*}{\shortstack{\textsc{Method}}} &
    \multicolumn{6}{c}{Ext. Backbone} &
    \multicolumn{6}{c}{Ext. Classes \& Backbone} \\
    \cmidrule(lr){2-7}\cmidrule(lr){8-13}
    & \multicolumn{3}{c}{CIFAR-100} & \multicolumn{3}{c}{ImageNet-1K} 
    & \multicolumn{3}{c}{CIFAR-100} & \multicolumn{3}{c}{ImageNet-1K} \\
    \cmidrule(lr){2-4}\cmidrule(lr){5-7}\cmidrule(lr){8-10}\cmidrule(lr){11-13}
    & $AC$ & $AA$ & $ACA$ & $AC$ & $AA$ & $ACA$ & $AC$ & $AA$ & $ACA$ & $AC$ & $AA$ & $ACA$ \\
    \midrule
    Baseline        & 0 & 35.13 & 0       & 0 & 40.58 & 0   & 0 & 29.88 & 0      & 0 & 35.88 & 0 \\ 
    BCT     & \textcolor[RGB]{0,0,139}{\textbf{1}} & 53.57 & 54.06  & \textcolor[RGB]{0,0,139}{\textbf{1}} & 61.80 & 61.43  & \textcolor[RGB]{0,0,139}{\textbf{1}} & \textcolor[RGB]{65,105,225}{\textbf{43.61}} & \textcolor[RGB]{65,105,225}{\textbf{40.79}} & \textcolor[RGB]{0,0,139}{\textbf{1}} & 50.64 & 42.98 \\ 
    RACT      & \textcolor[RGB]{0,0,139}{\textbf{1}} & 53.80 & 54.62  & \textcolor[RGB]{0,0,139}{\textbf{1}} & \textcolor[RGB]{65,105,225}{\textbf{62.15}} & 62.07 & \textcolor[RGB]{0,0,139}{\textbf{1}} & \textcolor[RGB]{0,0,139}{\textbf{43.66}} & \textcolor[RGB]{0,0,139}{\textbf{41.01}} &  \textcolor[RGB]{0,0,139}{\textbf{1}} &  \textcolor[RGB]{0,0,139}{\textbf{52.74}}     &   \textcolor[RGB]{0,0,139}{\textbf{48.92}}    \\ 
    AdvBCT  & 0 & 49.56 & 0   & \textcolor[RGB]{0,0,139}{\textbf{1}} & 59.22 & 57.73  & \textcolor[RGB]{0,0,139}{\textbf{1}} & 39.06 & 35.05 & \textcolor[RGB]{0,0,139}{\textbf{1}} & 50.37 & 45.42 \\ 
    LCE   & \textcolor[RGB]{0,0,139}{\textbf{1}} & 52.14 & 52.26  & \textcolor[RGB]{0,0,139}{\textbf{1}} & 61.22 & 61.40 & \textcolor[RGB]{0,0,139}{\textbf{1}} & 42.42 & 39.72 & \textcolor[RGB]{0,0,139}{\textbf{1}} & \textcolor[RGB]{65,105,225}{\textbf{52.28}} & 48.06 \\ 
    CoReS     & \textcolor[RGB]{0,0,139}{\textbf{1}} & \textcolor[RGB]{65,105,225}{\textbf{56.67}} & \textcolor[RGB]{65,105,225}{\textbf{57.43}}  & \textcolor[RGB]{0,0,139}{\textbf{1}} & 62.03 & \textcolor[RGB]{65,105,225}{\textbf{63.19}} & \textcolor[RGB]{0,0,139}{\textbf{1}} & 42.61 & 37.20 & \textcolor[RGB]{0,0,139}{\textbf{1}} & 49.10 & 39.90 \\ 
    \midrule
    PSP             &  \textcolor[RGB]{0,0,139}{\textbf{1}} &   54.57   &   56.01         &  \textcolor[RGB]{0,0,139}{\textbf{1}} &  \textcolor[RGB]{0,0,139}{\textbf{68.35}}     & \textcolor[RGB]{0,0,139}{\textbf{65.28}}  &  \textcolor[RGB]{0,0,139}{\textbf{1}} &    38.65   &   32.84 &   \textcolor[RGB]{0,0,139}{\textbf{1}} & 48.42 & 41.24      \\ 
    LSP             & \textcolor[RGB]{0,0,139}{\textbf{1}} & \textcolor[RGB]{0,0,139}{\textbf{56.76}} & \textcolor[RGB]{0,0,139}{\textbf{57.96}}  & \textcolor[RGB]{0,0,139}{\textbf{1}} & 62.00 & 63.16 & \textcolor[RGB]{0,0,139}{\textbf{1}} & 43.08 & 37.89 & \textcolor[RGB]{0,0,139}{\textbf{1}} & 52.13 & \textcolor[RGB]{65,105,225}{\textbf{48.26}} \\ 
    \bottomrule
  \end{tabular}
\end{table*}

\subsubsection{Extended Backbone Benchmark}\label{sec:extended_backbone}

In this experimental setting, we maintain the same dataset across all update steps (with no introduction of additional classes during model updates) while increasing model architectural complexity across two steps (i.e., ResNet-18 $\rightarrow$ ResNet-50).
The goal is to evaluate compatibility when model updates involve a change in model architecture rather than the introduction of new training classes.
Results for CIFAR-100 and ImageNet-1K are summarized in Tab.~\ref{tab:combined_benchmarks}.
Most of the evaluated approaches satisfy the compatibility criterion (\emph{AC}$=1$), suggesting that compatibility is generally easier to maintain in this setting than in the extended-classes benchmark.
Since this benchmark consists of two update steps, \emph{AC} is determined by a single cross-model comparison, and performance differences therefore emerge primarily in terms of \emph{AA} and \emph{ACA}.
LSP achieves the highest \emph{AA} and \emph{ACA} on CIFAR-100 and remains within $0.1$--$0.2$ percentage points of the best-performing methods, CoReS and RACT, on ImageNet-1K.
These results indicate that LSP maintains strong compatibility and retrieval performance even when the model architecture changes substantially.
Notably, PSP and LSP achieve competitive performance without requiring compatibility-specific regularization, learnable mappings, or modifications to the model architecture.

\subsubsection{Extended Classes and Backbone Benchmark}
\label{sec:extended_classes_backbone}

This experimental configuration, consisting of two update steps, combines the two previously examined update mechanisms: at the second step, the model is retrained from scratch on an expanded training set using a more expressive backbone.
Results for this combined setting are reported in Tab.~\ref{tab:combined_benchmarks} for the CIFAR-100 and ImageNet-1K datasets.
The Baseline, which does not employ any compatibility-preservation mechanism, fails to satisfy the compatibility criterion (\emph{AC}$=0$) on both datasets, consistent with the behavior observed in the previous experimental settings.
In contrast, all evaluated backward-compatibility approaches, including PSP and LSP, satisfy the criterion (\emph{AC}$=1$), showing that compatibility can be maintained even when the model update simultaneously introduces new classes and changes the backbone architecture.
Across both CIFAR-100 and ImageNet-1K, LSP shows robust performance under this combined update setting, achieving \emph{AA} and \emph{ACA} values that match or closely approach those of the strongest compared methods.
PSP also satisfies the compatibility criterion on both datasets, but attains lower \emph{AA} and \emph{ACA} than LSP.
This behavior is consistent with the smaller within-class spread of the softmax output representation observed in the Neural Collapse analysis of Sec.~\ref{sec:logits_theory}, which may reduce retrieval performance in the semi-open-set setting, where the test set also contains classes not yet observed during training.

\begin{table*}[t]
        \centering 
        \caption{Compatibility between pretrained models on ImageNet-1K and Places365: (\textit{a}) AlexNet~$\xrightarrow{}$ ResNet-50 (R-50) $\xrightarrow{}$ RegNetX\_3.2GF (RegNet) $\xrightarrow{}$ ResNet-152 (R-152) $\xrightarrow{}$ MaxViT\_T and (\textit{b}) ViT-B-32 $\xrightarrow{}$ ViT-B-16 $\xrightarrow{}$ ViT-L-16. 
        Models are pretrained on ImageNet-1K and downloaded from \cite{marcel2010torchvision}. Dark blue numbers indicate the highest values, while light blue the second-highest values.}
        \label{tab:pretrained_models}
        \footnotesize      
        \begin{tabular}{l cccc cccc cccc ccc}
            \toprule 
            \multirow{7}{*} 
            {\shortstack{\textsc{Method}}}
            & \multicolumn{7}{c}{AlexNet $\xrightarrow{}$  R-50 $\xrightarrow{}$ RegNet $\xrightarrow{}$ R-152 $\xrightarrow{}$ MaxViT} 
            & 
            & \multicolumn{7}{c}{ViT-B-32 $\xrightarrow{}$ ViT-B-16 $\xrightarrow{}$ ViT-L-16} 
            \\

            \cmidrule(lr){2-8} 
            \cmidrule(lr){9-16}

            & \multicolumn{3}{c}{ImageNet-1K} 
            & & \multicolumn{3}{c}{Places365} && \multicolumn{3}{c}{ImageNet-1K} 
            & 
            & \multicolumn{3}{c}{Places365} \\

            \cmidrule(lr){2-4} 
            \cmidrule(lr){6-8} 
            \cmidrule(lr){10-12} 
            \cmidrule(lr){14-16} 

            &
            {$AC$} &
            {$AA$} &
            {$ACA$} &
            & 
            {$AC$} &
            {$AA$} &
            {$ACA$} &
            & 
            {$AC$} &
            {$AA$} &
            {$ACA$} &
            & 
            {$AC$} &
            {$AA$} &
            {$ACA$} 
            \\ 
            \midrule
            
            Baseline & 0 & 21.63 & 0 & & 0 & 8.38 & 0 && 0 & 38.63 & 0  && 0 & 15.10 & 0 \\ 
            PSP & \textcolor[RGB]{0,0,139}{\textbf{1}} & \textcolor[RGB]{0,0,139}{\textbf{74.34}} & \textcolor[RGB]{0,0,139}{\textbf{76.12}}  & & \textcolor[RGB]{0,0,139}{\textbf{0.9}} & \textcolor[RGB]{65,105,225}{\textbf{15.92}} & \textcolor[RGB]{0,0,139}{\textbf{14.08}} && \textcolor[RGB]{0,0,139}{\textbf{1}} & \textcolor[RGB]{0,0,139}{\textbf{78.94}} & \textcolor[RGB]{0,0,139}{\textbf{80.53}} && \textcolor[RGB]{0,0,139}{\textbf{1}} & \textcolor[RGB]{65,105,225}{\textbf{18.31}} & \textcolor[RGB]{0,0,139}{\textbf{17.60}} \\ 
            LSP & \textcolor[RGB]{65,105,225}{\textbf{0.7}} & \textcolor[RGB]{65,105,225}{\textbf{60.68}} & \textcolor[RGB]{65,105,225}{\textbf{37.61}} && \textcolor[RGB]{65,105,225}{\textbf{0.2}} & \textcolor[RGB]{0,0,139}{\textbf{21.36}} & \textcolor[RGB]{65,105,225}{\textbf{4.06}} && \textcolor[RGB]{65,105,225}{\textbf{0.67}}  & \textcolor[RGB]{65,105,225}{\textbf{71.82}}  &  \textcolor[RGB]{65,105,225}{\textbf{45.91}}  && \textcolor[RGB]{65,105,225}{\textbf{0.33}}  & \textcolor[RGB]{0,0,139}{\textbf{25.71}} & \textcolor[RGB]{65,105,225}{\textbf{7.99}}  \\ 
            
            \bottomrule
        \end{tabular} 
\end{table*}

\begin{figure*}[t]
    \vspace{-0.2cm}
    \centering
        \begin{subfigure}{0.32\linewidth}
        \centering
        \includegraphics[width=0.95\linewidth]{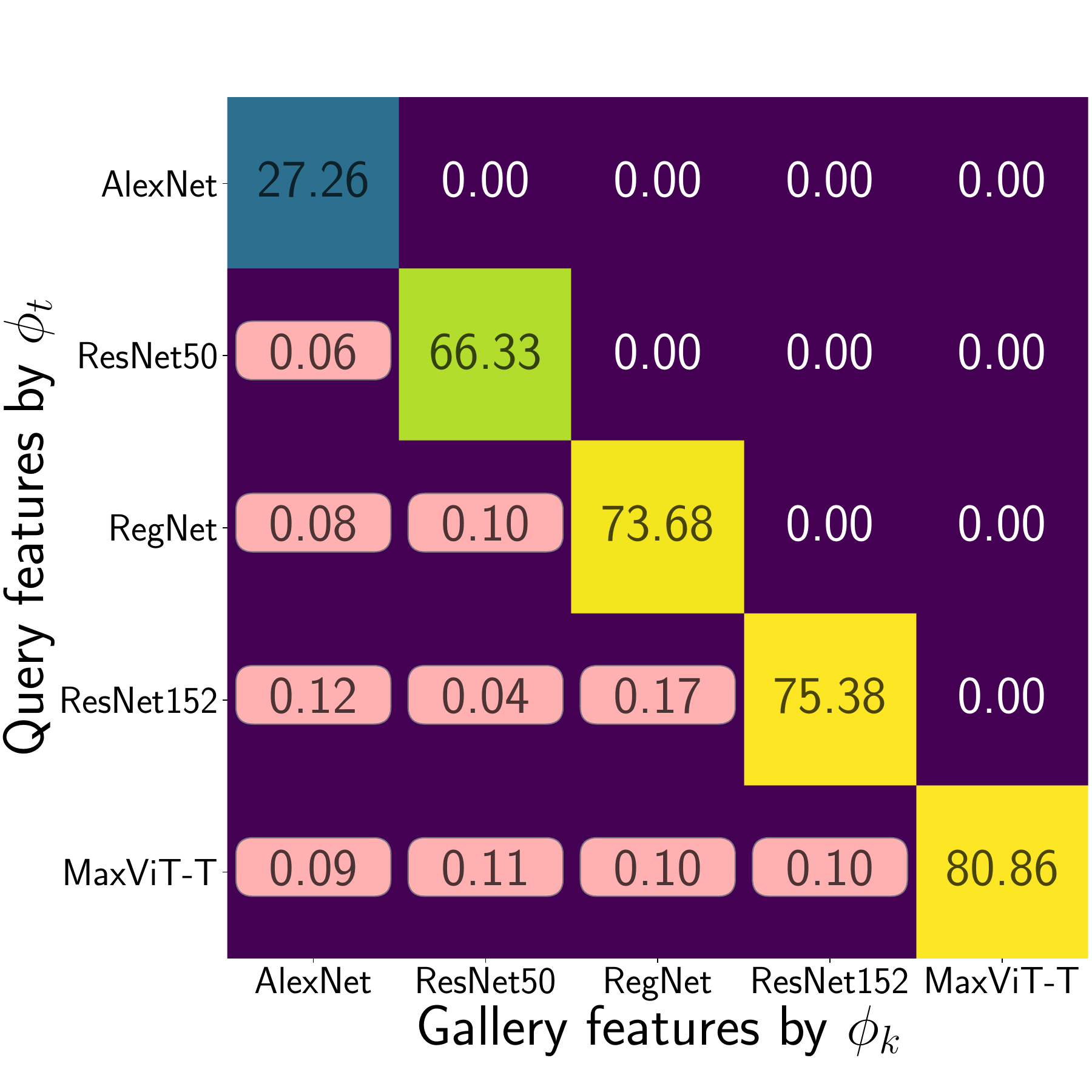}        \caption{Baseline}
    \label{fig:pretrained_models_prelogit_5step}
    \end{subfigure}
    % \hspace{-5pt}
    \hfill
    \begin{subfigure}{0.32\linewidth}
        \centering
        \includegraphics[width=0.95\linewidth]{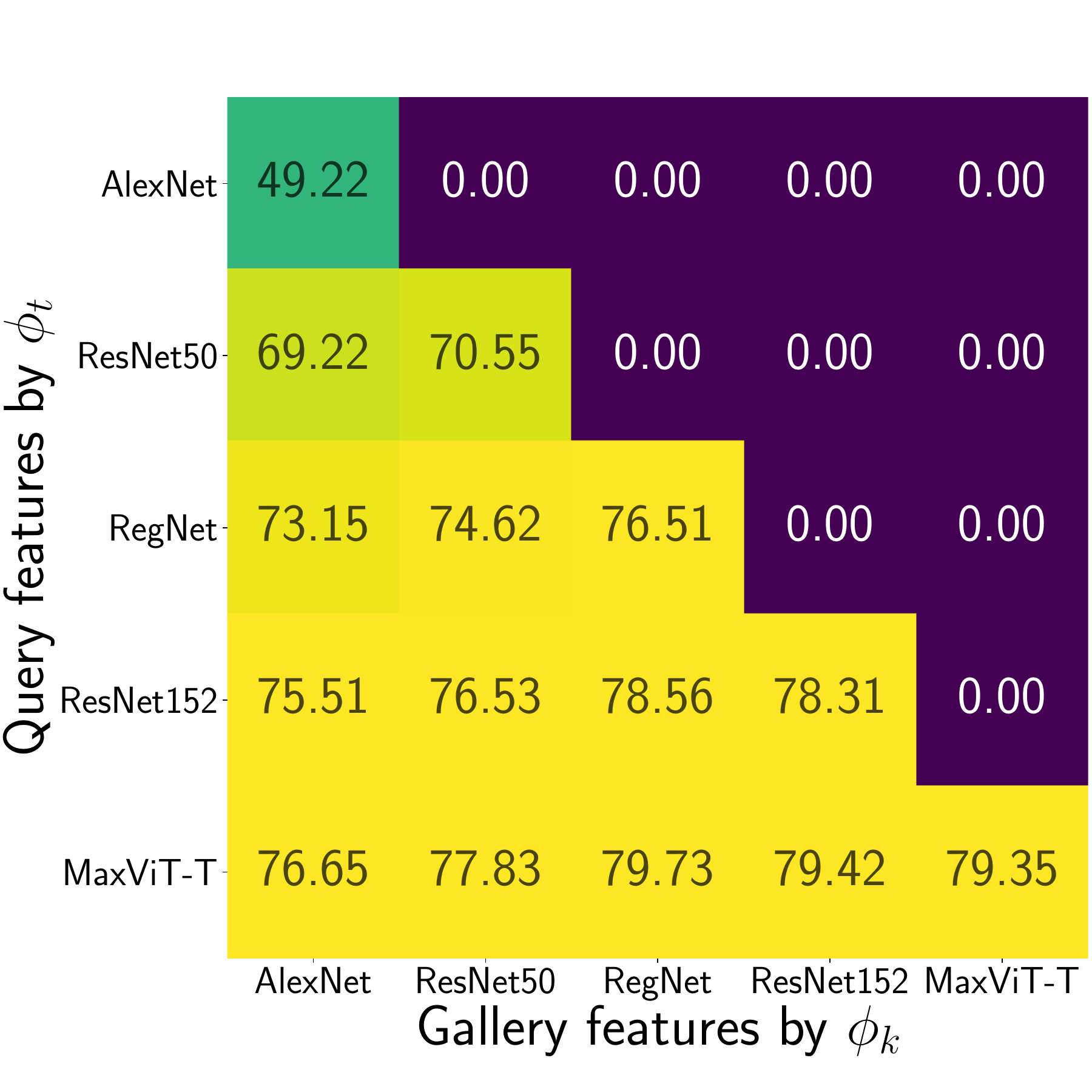}
        \caption{PSP}
    \label{fig:pretrained_models_softmaxp_5step}
    \end{subfigure}
    \hfill
    % \hspace{-5pt}
    \begin{subfigure}{0.32\linewidth}
        \centering
        \includegraphics[width=0.95\linewidth]{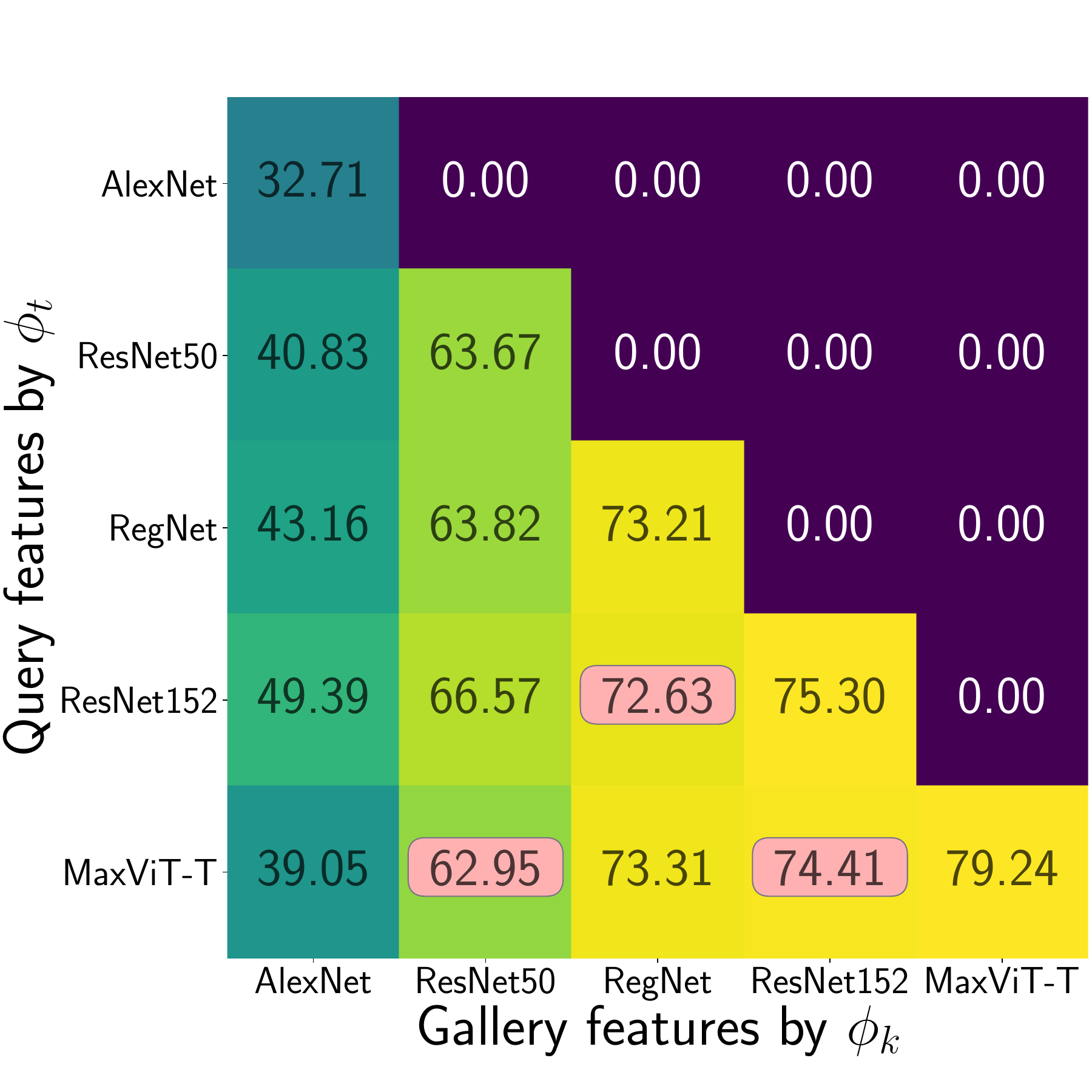}
        \caption{LSP}
    \label{fig:pretrained_models_logit_5step}
    \end{subfigure}
    \caption{Compatibility between independently trained models, all pretrained on ImageNet-1K. Compatibility matrices for five steps with increasingly expressive models (AlexNet, ResNet-50, RegNetX\_3.2GF, ResNet-152, and MaxViT\_T). The matrices display retrieval performance on the ImageNet-1K test set.
    Entries that do not satisfy the compatibility criterion in Eq.~\ref{eq:multistepecc} are highlighted with a light-red background. Model representations from: (a) encoder output (Baseline); (b) softmax outputs (PSP); (c) logits (LSP).}
\label{fig:pretrained_comp_matrix_paper}
\vspace{-1cm}
\end{figure*}

\subsection{Compatibility Across Independently Trained Models From Public Repositories}\label{sec:compatibility_pretrained_models}

The standard compatibility benchmarks of Sec.~\ref{sec:experiments_extended_classes} consider models explicitly trained within a backward-compatibility protocol.
Here, we evaluate a complementary and more practical setting in which compatibility is measured directly between independently trained models obtained from a public repository, without any additional compatibility-specific training or adaptation.

We consider two sequences of pretrained models from the publicly available torchvision repository~\citep{marcel2010torchvision}: a convolutional and hybrid sequence comprising AlexNet~\citep{krizhevsky2012imagenet}, ResNet-50~\citep{he2016deep}, RegNetX\_3.2GF~\citep{radosavovic2020designing}, ResNet-152~\citep{he2016deep}, and MaxViT\_T~\citep{tu2022maxvit}, and a transformer sequence comprising ViT-B/32, ViT-B/16, and ViT-L/16~\citep{Dosovitskiy2020AnII}.
All models are pretrained on ImageNet-1K while differing substantially in architecture and representational capacity.

We evaluate retrieval on the ImageNet-1K test set, which follows the pretraining distribution (closed-set), and on Places365~\citep{zhou2017places}, whose classes and data distribution differ from those used for pretraining (open-set).
We compare PSP and LSP with the Baseline introduced in Sec.~\ref{sec:experiments_extended_classes}, obtained by $\ell_2$-normalizing the encoder features.
Existing backward-compatibility methods are not included because they require additional training, adaptation, or architectural modifications and therefore cannot be directly applied to these publicly available, independently trained models.

The results in Tab.~\ref{tab:pretrained_models} show a clear advantage of the proposed hyperspherical simplex representations over the Baseline.
PSP provides highly consistent compatibility across the evaluated model sequences, attaining \emph{AC}$=1$ in three of the four settings and \emph{AC}$=0.90$ in the remaining open-set setting.
LSP also substantially improves compatibility over the Baseline, although it satisfies the empirical compatibility criterion for a smaller fraction of model-update pairs.
The compatibility matrices in Fig.~\ref{fig:pretrained_comp_matrix_paper} illustrate these differences: LSP and the Baseline achieve comparable same-model retrieval performance along the main diagonal, but the Baseline fails to maintain this performance across model updates, resulting in \emph{AC}$=0$.
A different trade-off between PSP and LSP emerges in closed-set and open-set evaluation.
On ImageNet-1K, corresponding to the same data distribution used to pretrain all models, PSP consistently achieves higher \emph{AC}, \emph{AA}, and \emph{ACA} than LSP.
This behavior is consistent with the stronger alignment of PSP representations with the regular-simplex geometry and their lower within-class spread discussed in Sec.~\ref{sec:logits_theory}.
On Places365, whose data distribution differs from the pretraining one, LSP instead achieves higher \emph{AA} than PSP for both model sequences.
This result is consistent with the broader angular spread of LSP representations, which may provide greater flexibility under domain shift.
However, PSP retains substantially higher compatibility in this setting, achieving \emph{AC}$=0.90$ and $1.00$, compared with $0.20$ and $0.33$ for LSP, and consequently also obtains higher \emph{ACA}.

Overall, these experiments show that PSP and LSP can substantially improve cross-model retrieval compatibility even between models that were not explicitly trained to be compatible.
PSP provides the most consistent empirical compatibility, whereas LSP preserves strong retrieval accuracy and achieves higher open-set \emph{AA}, at the cost of lower compatibility than PSP.

These results reflect a structural property of the proposed representations: their common reference is determined by the label space rather than by any interaction between models.
Models trained on the same set of classes under the standard class-indexing convention of a dataset, as is the case for all torchvision models trained on ImageNet-1K, induce the same one-hot encoding and therefore share the same reference geometry, without requiring coordination during training or access to previous models.
Complementary compatibility matrices for the open-set evaluation and the ViT model sequence are reported in Appendix~\ref{sec:compatiiblity_matrix_ind_trained}.

\section{Conclusions}
\label{sec:conclusions}

In this paper, we introduce hyperspherical simplex representations derived from softmax outputs (PSP) and logits (LSP) for backward-compatible retrieval.
The proposed representations leverage the simplex structure induced in the model output space by softmax cross-entropy training with one-hot encoded class labels, which provides a common geometric reference that aligns corresponding class prototypes across model updates, without requiring compatibility-specific losses, learned mapping functions, a shared pre-allocated fixed classifier, or modifications to the model architecture.
When new classes are introduced, the resulting change in simplex geometry is compensated by a deterministic linear transformation that restores the alignment of the prototypes of previously observed classes.
We provide a theoretical analysis of the formal backward-compatibility inequalities, showing that, under a von Mises--Fisher model with class concentration increasing across updates, they are satisfied in expectation if corresponding class prototypes are aligned across updates and prototypes of different classes form angles greater than $\pi/2$.
Aligned regular-simplex representations centered at the origin satisfy both conditions; since PSP instantiates this geometry by construction and LSP under the logit-UFM, both satisfy the compatibility inequalities in expectation.
We show that PSP and LSP achieve competitive retrieval performance on standard backward-compatibility benchmarks, with particularly high compatibility in scenarios involving many model updates.
Experiments further demonstrate that the proposed representations improve cross-model retrieval compatibility between independently trained models obtained from a public repository, without any additional adaptation.
Finally, we introduce a top-$k$ sparsification strategy motivated by the rank-preserving property of the softmax to address the linear growth of the representation dimensionality with the number of classes.
Across the evaluated settings, top-$k$ sparsification substantially reduces representation storage, with negligible changes in retrieval accuracy and without degrading compatibility, supporting the use of PSP and LSP in large-scale retrieval scenarios.

\textbf{Limitations.}
PSP and LSP rely on the simplex structure induced in the model output space by softmax cross-entropy training with one-hot encoded class labels.
This requirement limits the direct applicability of our approach to models trained with alternative objectives, such as self-supervised pretraining without explicit classification heads or regression-based formulations.
However, this constraint can be mitigated by training a linear classifier on top of a frozen pretrained encoder using a modest amount of external labeled data (see Appendix~\ref{sec:linear_probe}), with limited computational overhead since only the classifier is optimized.

A second limitation concerns the consistency of the label space across model updates: our method assumes that the mapping between each class label and its output index remains unchanged.
Changes in class ordering or label definitions would disrupt the geometric correspondence on which the backward-compatibility of the proposed representations relies.
A reordering of the classes can be undone by a fixed permutation of the output coordinates, provided that the correspondence between old and new indices is known.
Moreover, model updates are commonly driven by additional samples for existing classes or by the introduction of new classes, corresponding to the domain-incremental and class-incremental scenarios of incremental learning~\citep{van2022three, lin2021clear} and to the update settings considered in backward-compatible representation learning~\citep{shen2020towards}; in both cases, the existing class indices are preserved.
% Moreover, in practice, datasets typically evolve more conservatively than models: while architectures and training paradigms may change frequently, dataset versioning mainly adds samples to existing classes or introduces new classes~\citep{lin2022continual}.

Finally, our theoretical analysis relies on modeling assumptions, namely von Mises--Fisher class distributions and increasing concentration across updates, which provide a tractable approximation of learned representations rather than an exact characterization, so that the resulting guarantees hold in expectation under these assumptions.
Empirically, PSP and LSP nonetheless substantially improve compatibility across the evaluated settings, indicating that the identified geometric conditions are informative in practice.

\section*{Broader Impact Statement}
Backward-compatible retrieval can reduce the computational, storage, and environmental costs associated with repeatedly re-indexing large galleries.
However, retaining legacy representations may also preserve biases or limitations inherited from earlier models, training data, or gallery composition.
Moreover, aggregate compatibility metrics do not guarantee uniform performance across subpopulations.
Future work should therefore consider disaggregated compatibility measures and establish when gallery re-indexing is preferable to preserving compatibility, particularly after updates intended to correct bias, modify class semantics, remove previously learned information, or address substantial distribution shifts.

\bibliography{main}
\bibliographystyle{tmlr}

\newpage

\appendix

\section{Proof of Proposition~\ref{theo:softmax_affects_logits}}
\label{sec:app_softmax_affects_logits}

\begin{prop*}[\ref{theo:softmax_affects_logits}]
Under the logit-UFM assumptions, as the softmax outputs converge toward the vertices of the probability simplex during training, the corresponding logit vectors align with the directions defined in
Eq.~\ref{eq:class_prototypes}.
\end{prop*}

\begin{proof}
Consider a sample belonging to class $c$ at model update step $k$, with
$1 \leq c \leq C^k$. For notational simplicity, let $C=C^k$, and
let $\mathbf{z}^{(r)}\in\mathbb{R}^{C}$ denote its logit vector at
optimization step $r$. We denote by $\mathbf{e}_c$ the corresponding
one-hot encoded label vector.

Motivated by the UFM analyses of \citet{mixon2022neural,fang2021exploring}, we consider a logit-space adaptation in which the shared backbone and classifier are not modeled explicitly. Instead, the logit vector associated with each training sample is treated as an independent unconstrained optimization variable and is optimized directly by gradient descent on the softmax cross-entropy loss. This idealized abstraction (logit-UFM) allows us to isolate the geometry induced by the loss from the additional coupling introduced by the network architecture and its shared parameters.
Within this abstraction, we consider an equiprobability initialization by placing each logit vector at the origin of the logit space. Consequently, the corresponding softmax outputs are the uniform probability distribution i.e., $\mathbf{o}^k = \big[1/C^k, ..., 1/C^k \big]$.
The gradient of the softmax
cross-entropy loss with respect to the $j$-th logit is
\[
\frac{\partial\mathcal{L}}{\partial z_j}
=
\sigma(\mathbf{z})_j-(\mathbf{e}_c)_j
=
\begin{cases}
\sigma(\mathbf{z})_c-1, & \text{if } j=c,\\
\sigma(\mathbf{z})_j,   & \text{if } j\neq c.
\end{cases}
\]
Consequently, gradient descent updates the logit vector according to
$
\mathbf{z}^{(r+1)}
=
\mathbf{z}^{(r)}
+
\eta_r
\left(
\mathbf{e}_c-\sigma(\mathbf{z}^{(r)})
\right),
$
where $\eta_r>0$ denotes the learning rate.

The first gradient update is
\[
z_j^{(1)}
=
\begin{cases}
\eta_0\left(1-\dfrac{1}{C}\right),
& \text{if } j=c,\\[2mm]
-\dfrac{\eta_0}{C},
& \text{if } j\neq c,
\end{cases}
\]
and hence
\[
\mathbf{z}^{(1)}
=
\eta_0
\left(
\mathbf{e}_c-\frac{1}{C}\mathbf{1}
\right)
=
\eta_0\mathbf{v}_c^k.
\]

We now show by induction that this alignment is preserved throughout
optimization. Suppose that, at iteration $r$,
$
\mathbf{z}^{(r)}=a_r\mathbf{v}_c^k
$
for some scalar $a_r\geq 0$.
Here, $a_r$ is the scale of the logit vector along the prototype direction $\mathbf{v}_c^k$: it determines its norm, $\|\mathbf{z}^{(r)}\|_2 = a_r\|\mathbf{v}_c^k\|_2$, and equals the margin $z^{(r)}_c - z^{(r)}_j$ between the target logit and each non-target logit $j \neq c$. At initialization, $a_0 = 0$, and after the first update, $a_1 = \eta_0$.

Since
\[
\mathbf{v}_c^k
=
\mathbf{e}_c-\frac{1}{C}\mathbf{1},
\]
all non-target components of $\mathbf{z}^{(r)}$ are equal and therefore
have identical softmax probabilities.

Let
\[
p_r=\sigma(\mathbf{z}^{(r)})_c
\]
denote the target-class probability. Because the softmax probabilities
sum to one, for every $j\neq c$,
\[
\sigma(\mathbf{z}^{(r)})_j
=
\frac{1-p_r}{C-1}.
\]
Consequently,
\[
\mathbf{e}_c-\sigma(\mathbf{z}^{(r)})
=
\frac{C(1-p_r)}{C-1}
\left(
\mathbf{e}_c-\frac{1}{C}\mathbf{1}
\right)
=
\frac{C(1-p_r)}{C-1}\mathbf{v}_c^k.
\]
Substituting this expression into the gradient update yields
\[
\mathbf{z}^{(r+1)}
=
\left(
a_r+
\eta_r\frac{C(1-p_r)}{C-1}
\right)
\mathbf{v}_c^k.
\]
Hence, $\mathbf{z}^{(r+1)}$ is also a non-negative scalar multiple of
$\mathbf{v}_c^k$. Thus, by induction, the logit vector remains a non-negative scalar multiple of $\mathbf{v}_c^k$ at every optimization step.
Since
\[
p_r=\frac{\exp(a_r)}{\exp(a_r)+C-1}
\]
is continuous and strictly increasing in $a_r$ on $[0,\infty)$, with
$p_r\to 1$ if and only if $a_r\to\infty$, the hypothesis that the
softmax output converges to the vertex $\mathbf{e}_c$ (i.e.\
$p_r\to 1$) forces $a_r\to\infty$. Throughout this entire
trajectory---including in this limit---$\mathbf{z}^{(r)}=a_r\mathbf{v}_c^k$, so the logit direction remains aligned with
\[
\mathbf{v}_c^k
=
\mathbf{e}_c-\mathbf{o}^k.
\]
\end{proof}

\begin{figure*}[t]
    \centering
    % \hspace{-10pt}
    \begin{subfigure}{\linewidth}
    \centering
    \includegraphics[width=0.9\linewidth, trim={0 21.7cm 0 0cm}, clip]{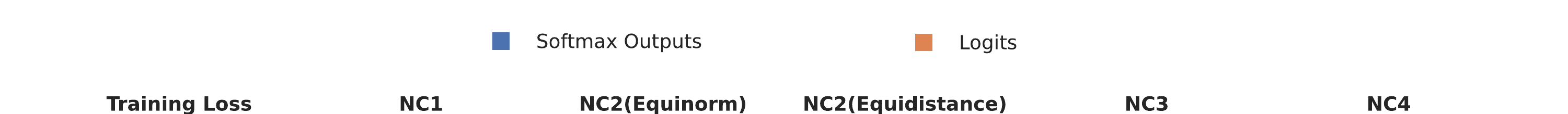}
    \includegraphics[width=0.9\linewidth, trim={0 0 0 10.5cm}, clip]{images/nc_cifar_2_no_error.pdf}
    \caption{CIFAR-100}
    \label{fig:neural_collapse_proof_cifar}
    \end{subfigure}
    \begin{subfigure}{\linewidth}
    \centering
    \includegraphics[width=0.9\linewidth, trim={0 0 0 10.5cm}, clip]{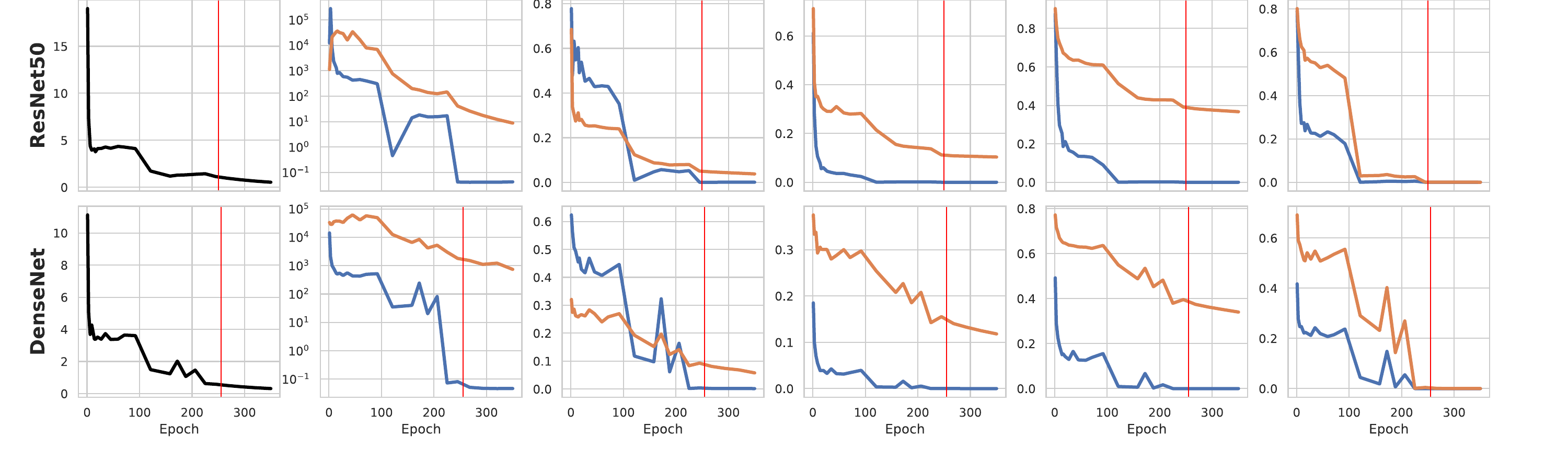}
    \caption{Tiny-ImageNet-200}
    \label{fig:neural_collapse_proof_tiny}
    \end{subfigure}
    \caption{Plots of  training loss and Neural Collapse hypotheses for feature representations derived from softmax outputs (blue) and logits (orange) on CIFAR-100 and Tiny-ImageNet-200 with ResNet-50 and DenseNet. The terminal phase of training is shown with a vertical red line.
    }
    \label{fig:neural_collpase_proof_cifar_tiny}
\end{figure*}

\section{Convergence of PSP and LSP to a Simplex Configuration}\label{sec:app_neural_collapse_logits}

In the following, we present additional experiments beyond those discussed in Fig.~\ref{fig:neural_collpase_proof}, illustrating the NC metrics for softmax outputs (blue curves) and logits (orange curves) during training for ResNet-50 and DenseNet on the CIFAR-100 and Tiny-ImageNet-200 datasets.

The results in Fig.~\ref{fig:neural_collpase_proof_cifar_tiny} present plots that are consistent with those shown in Fig.~\ref{fig:neural_collpase_proof} for ResNet-18 on CIFAR-100, confirming that logits exhibit larger spread of feature distributions and less alignment with the simplex geometry than softmax outputs.

\section{Mathematical Derivation of Eq.~\ref{eq:misalignement}}\label{sec:app_theo_misalignment_short}

Let $\mathbf{v}^{k}_c = \mathbf{e}_c - \mathbf{o}^k \in \mathbb{R}^{C^k}$ and $\mathbf{v}^{t}_c = \mathbf{e}_c - \mathbf{o}^t \in \mathbb{R}^{C^t}$ with $C^k\le C^t$ the prototypes of class $c$ at two different steps $t$ and $k$. To compare prototypes across different dimensions, we extend $\mathbf{v}^{k}_c$ to $\mathbb{R}^{C^t}$ by zero-padding (not affecting the norm of a vector):
$\tilde{\mathbf{v}}^{\,k}_c = [\mathbf{v}^{k}_c \;|\;\mathbf{0}_{C^t - C^k}] \in \mathbb{R}^{C^t}$.

We compute the dot product between the two prototypes:
\begin{align}
\tilde{\mathbf{v}}^{\,k}_c\!\cdot\!\mathbf{v}^t_c=(\mathbf{e}_c-\mathbf{o}^k)\cdot(\mathbf{e}_c-\mathbf{o}^t)=1-\tfrac{1}{C^t}-\tfrac{1}{C^k}+\tfrac{1}{C^t}=1-\tfrac{1}{C^k}=\frac{C^k-1}{C^k}. \nonumber  
\end{align}
Since the norms of the two prototypes are:
$$
\|\tilde{\mathbf{v}}^{\,k}_c\|^2=\frac{C^k-1}{C^k},\qquad
\|\mathbf{v}^t_c\|^2=\frac{C^t-1}{C^t},
$$
their cosine similarity is:
$$
\cos(\tilde{\mathbf{v}}^{\,k}_c,\mathbf{v}^t_c)
=\frac{\frac{C^k-1}{C^k}}{\sqrt{\frac{C^k-1}{C^k}}\sqrt{\frac{C^t-1}{C^t}}}
=\sqrt{\frac{C^k-1}{C^k}\cdot\frac{C^t}{C^t-1}}.
$$
Therefore, the misalignment angle between $\tilde{\mathbf{v}}^{\,k}_c$ and $\mathbf{v}^{t}_c$ is provided by the following formula:
$$
\theta_{c}^{k\to t}
=\arccos\!\left(\sqrt{\frac{C^k-1}{C^k}\cdot\frac{C^t}{C^t-1}}\right).
$$

\section{Proof of Proposition~\ref{theo:orth_new_class}}\label{sec:orth_new_class}
\begin{prop1*}[\ref{theo:orth_new_class}]
  When updating the model from step $k$ to step $t$, if the number of classes increases from $C^k$ to $C^t$, the prototypes associated with the newly introduced classes
  at step $t$ are orthogonal to the subspace spanned by the class prototypes at step $k$.
\end{prop1*}
\begin{proof}

We first demonstrate the proposition in the case where only one class is added between step \(k\) and step \(t\).
For simplicity in notation, we denote \(C^k=d\), resulting in \(C^t=C^k+1=d+1\).
We assume that existing classes retain their coordinate indices across the model update.
If \(d=1\), the old prototype is zero and the proposition holds trivially; hence, we assume \(d\geq2\) in the following calculation.

From Eq.~\ref{eq:class_prototypes}, a prototype vector of the $c$-th class at step $k$ is denoted as $\mathbf{v}^k_c=\mathbf{e}_c-\mathbf{o}^k$ in $\mathbb{R}^d$, where
$1\leq c\leq d$. Each component $v^k_{c,j}$ is:
\begin{equation}
\label{eq:v_components}
v^k_{c,j} =
\begin{cases}
  1-\frac{1}{d} & \text{if } j=c,\\
  -\frac{1}{d} & \text{otherwise},
\end{cases}
\qquad j=1,\ldots,d.
\end{equation}

Adding a new class in the training set at step $t$ adds a new vertex $\mathbf{e}_{d+1}$ in the probability simplex, shifting its center to:
$\mathbf{o}^{t}=\frac{1}{d+1}\sum_{c=1}^{d+1}\mathbf{e}_c$.
The prototype vector of the new class $d+1$ is then defined as $\mathbf{v}^t_{d+1}=\mathbf{e}_{d+1}-\mathbf{o}^t$.

To measure the angle between the new prototype and the old prototypes, we compute the cosine similarity between $\mathbf{v}^t_{d+1}$ and the generic prototype $\mathbf{v}
^k_c$ after zero-padding the latter.

According to Eq.~\ref{eq:v_components}, the norms of \(\mathbf{v}^k_c\) and \(\mathbf{v}^t_{d+1}\) are respectively:
\begin{equation}
\begin{aligned}
\|\mathbf{v}^k_c\|
&= \sqrt{\sum_{j=1}^{d}\left(v^k_{c,j}\right)^2}\\
&= \sqrt{\sum_{\substack{j=1\\j\neq c}}^{d}
  \left(v^k_{c,j}\right)^2+\left(v^k_{c,c}\right)^2}\\
&= \sqrt{\sum_{\substack{j=1\\j\neq c}}^{d}
  \frac{1}{d^2}+\frac{(d-1)^2}{d^2}}
= \sqrt{\frac{d-1}{d^2}+\frac{(d-1)^2}{d^2}}
= \sqrt{\frac{d-1}{d}}
\end{aligned}
\end{equation}
and
\begin{equation}
\begin{aligned}
\|\mathbf{v}^t_{d+1}\|
&= \sqrt{\sum_{j=1}^{d+1}\left(v^t_{d+1,j}\right)^2}\\
&= \sqrt{\sum_{j=1}^{d}\left(v^t_{d+1,j}\right)^2
  +\left(v^t_{d+1,d+1}\right)^2}\\
&= \sqrt{\sum_{j=1}^{d}\frac{1}{(d+1)^2}
  +\frac{d^2}{(d+1)^2}}
= \sqrt{\frac{d}{d+1}}.
\end{aligned}
\end{equation}

The components of the corresponding normalized vectors are respectively:
\begin{equation}
\label{eq:norm_old_prototype}
\frac{v^k_{c,j}}{\|\mathbf{v}^k_c\|}
=
\begin{cases}
  \frac{1-\frac{1}{d}}{\sqrt{\frac{d-1}{d}}}
  & \text{if } j=c,\\
  \frac{-\frac{1}{d}}{\sqrt{\frac{d-1}{d}}}
  & \text{otherwise},
\end{cases}
\qquad j=1,\ldots,d,
\end{equation}
and
\begin{equation}
\label{eq:norm_new_prototype}
\frac{v^t_{d+1,j}}{\|\mathbf{v}^t_{d+1}\|}
=
\begin{cases}
  \frac{1-\frac{1}{d+1}}{\sqrt{\frac{d}{d+1}}}
  & \text{if } j=d+1,\\
  \frac{-\frac{1}{d+1}}{\sqrt{\frac{d}{d+1}}}
  & \text{otherwise},
\end{cases}
\qquad j=1,\ldots,d+1.
\end{equation}

In order to compute the dot product, \(\mathbf{v}^k_c\) is padded with a zero component (this operation does not change its norm). Retaining the same notation for the embedded vector, we write:
\[
\mathbf{v}^k_c=\left(\mathbf{e}_c-\mathbf{o}^k,0\right).
\]
Then, the dot product can be expressed as follows, with the components in the $c$-th and $(d+1)$-th positions written separately:
\begin{equation}
\label{eq:dot_product_teo1}
\begin{aligned}
\frac{\mathbf{v}^k_c}{\|\mathbf{v}^k_c\|}
\cdot
\frac{\mathbf{v}^t_{d+1}}{\|\mathbf{v}^t_{d+1}\|}
&=
\sum_{\substack{j=1\\j\neq c}}^{d}
\frac{v^k_{c,j}}{\|\mathbf{v}^k_c\|}
\frac{v^t_{d+1,j}}{\|\mathbf{v}^t_{d+1}\|}\\
&+
\frac{v^k_{c,c}}{\|\mathbf{v}^k_c\|}
\frac{v^t_{d+1,c}}{\|\mathbf{v}^t_{d+1}\|}
+
\frac{v^k_{c,d+1}}{\|\mathbf{v}^k_c\|}
\frac{v^t_{d+1,d+1}}{\|\mathbf{v}^t_{d+1}\|}.
\end{aligned}
\end{equation}

Since the component $v^k_{c,d+1}=0$, the expression in Eq.~\ref{eq:dot_product_teo1} simplifies to:
\begin{equation*}
\frac{\mathbf{v}^k_c}{\|\mathbf{v}^k_c\|}
\cdot
\frac{\mathbf{v}^t_{d+1}}{\|\mathbf{v}^t_{d+1}\|}
=
\sum_{\substack{j=1\\j\neq c}}^{d}
\frac{v^k_{c,j}}{\|\mathbf{v}^k_c\|}
\frac{v^t_{d+1,j}}{\|\mathbf{v}^t_{d+1}\|}
+
\frac{v^k_{c,c}}{\|\mathbf{v}^k_c\|}
\frac{v^t_{d+1,c}}{\|\mathbf{v}^t_{d+1}\|}.
\end{equation*}
Using Eq.~\ref{eq:norm_old_prototype} and Eq.~\ref{eq:norm_new_prototype}, we obtain:
\begin{equation*}
\begin{aligned}
\frac{\mathbf{v}^k_c}{\|\mathbf{v}^k_c\|}
\cdot
\frac{\mathbf{v}^t_{d+1}}{\|\mathbf{v}^t_{d+1}\|}
&=
\frac{d-1}{d(d+1)\sqrt{\frac{d-1}{d}}\sqrt{\frac{d}{d+1}}}
-
\frac{d-1}{d(d+1)\sqrt{\frac{d-1}{d}}\sqrt{\frac{d}{d+1}}}
=0.
\end{aligned}
\end{equation*}

This demonstrates that a generic prototype vector $\mathbf{v}^k_c$ and the new prototype vector $\mathbf{v}^t_{d+1}$ are orthogonal.

More generally, if $C^t=d+m$ with $m\geq1$, every new-class prototype $\mathbf{v}^t_a$, $a=d+1,\ldots,d+m$, has its first $d$ components equal to $-1/(d+m)$. After padding
each old prototype with $m$ zeros, we obtain:
\[
\mathbf{v}^k_c\cdot\mathbf{v}^t_a
=
-\frac{1}{d+m}\sum_{j=1}^{d}v^k_{c,j}
=
-\frac{1}{d+m}
\left(1-\frac{1}{d}-\frac{d-1}{d}\right)
=0.
\]
Since this holds for every $c=1,\ldots,d$, each new-class prototype is orthogonal to the subspace spanned by the zero-padded class prototypes at step $k$.
\end{proof}

\section{Proof of Proposition~\ref{theo:projections_align}}\label{sec:app_theo_projections}

\begin{prop1*}[\ref{theo:projections_align}]
% Let \( \mathbf{v}^{t}_y \in \mathbb{R}^{C^t} \) be the prototype vector for class \( y \) within the probability simplex \( \Delta^{C^t-1} \). Define its projection as \( \mathbf{u}^t_y = \mathbf{P}_{t,k} \mathbf{v}^{t}_y \), where \( \mathbf{P}_{t,k} = [\mathbf{V}^k | \mathbf{0}] \) with  $\mathbf{V}^k = [ \mathbf{v}^k_y \; ]_{y=1}^{C^k} \in \mathbb{R}^{C^k \times C^k} $ and $ \mathbf{0} \in \mathbb{R}^{(C^t - C^k) \times C^k}$ is a zero matrix. Then the resulting projected prototype $\mathbf{u}^t_y$ is aligned with the prototype $\mathbf{v}^k_y$ within \(\Delta^{C^k-1}\).\\
For every class $c$ already present at step $k$, i.e., $1 \leq c \leq C^k$, the transformed prototype $\mathbf{u}^t_c = \mathbf{P}_{t,k}\mathbf{v}^t_c$ is aligned with its counterpart $\mathbf{v}^k_c$ at step $k$.
% where \( \mathbf{P}_{t,k} = \big[ \mathbf{V}^k | \mathbf{0} \big] \) with $ \mathbf{0}$ is a zero matrix of dimension \(C^k\times (C^t - C^k)\).
\end{prop1*}

\begin{proof}

Without loss of generality, we demonstrate the proposition in the case where only one class is added between step \( k \) and step \( t \). 
For simplicity in notation, we denote \( C^k = d \), resulting in \( C^t = C^k + 1 = d + 1 \).
% We assume that the ordering of class indices and their label association remains unchanged across model updates.

% In this case, it holds the projection matrix \(\mathbf{P}_{t,k} = \big[ \mathbf{V}^k | \mathbf{0} \big] \), where $\mathbf{V}^k = [ \mathbf{v}^k_c \; ]_{c=1}^{d}$ is a $d \times d$ matrix and  $\mathbf{0}$ is a zero matrix $ 1\times d$. 
% $\mathbf{P}_{t,k}$ projects a vector $\mathbf{v}^t_{y}$ from \(\mathbb{R}^{d+1}\) down to \(\mathbb{R}^{d}\).

Let $\mathbf{v}^k_c = \mathbf{e}_c - \mathbf{o}^k$ and $\mathbf{v}^t_{c} = \mathbf{e}_{c} - \mathbf{o}^t$ the  prototype vectors of the $c$-th class at step $k$ and $t$, respectively. Their components $v^k_{c,j}$ and $v^t_{c,j}$ are respectively defined as:
\begin{equation}
\label{eq:vkc_components}
v^k_{c,j} = 
\begin{cases}
    v^k_{c,c} = 1 -\frac{1}{d} \quad\,\, \text{if } j  \text{ equals the index $c$ of class } y \\
    v^k_{c,j} = -\frac{1}{d} \qquad\,\, \text{otherwise} 
\end{cases}
\end{equation}
and
\begin{equation}
\label{eq:vtc_components}
v^t_{c,j} = 
\begin{cases}
    v^t_{c,c} = 1 -\frac{1}{d+1} \quad\,\, \text{if } j  \text{ equals the index $c$ of class } y  \\
    v^t_{c,j} = -\frac{1}{d+1} \qquad\,\, \text{otherwise}  
\end{cases}
\end{equation}

We consider the vector $\mathbf{u}^t_c$ in \(\mathbb{R}^{d}\) obtained from $\mathbf{v}^t_{c}$ in $\mathbb{R}^{d+1}$ using the linear transformation matrix \(\mathbf{P}_{t,k}\):
\begin{equation}
\label{eq:utc}
\mathbf{u}^t_c = \mathbf{P}_{t,k} \, \mathbf{v}^{t}_c.
\end{equation}
where \(\mathbf{P}_{t,k} = [\mathbf{P}_{k,k}\,| \,\mathbf{0}]\) with $\mathbf{P}_{k,k} = [ \mathbf{v}^k_c \; ]_{c=1}^{d}$ a $d \times d$ matrix and  $\mathbf{0}$ a zero matrix $ d\times1$.
% Therefore, Eq.~\ref{eq:utc} can be rewritten as:
% \begin{equation}
% \label{eq:uct_withV}
% \mathbf{u}^t_c = \mathbf{V}^k \mathbf{v}^t_c
% \end{equation}
% which develops as 
% \begin{equation*}
% \mathbf{u}^t_c = 
% \begin{cases}
%     u^t_{c,c} \quad\,\, \text{if } j  \text{ equals the index $c$ of class } y  \\
%     u^t_{c,j}  \quad\,\, \text{otherwise}  
% \end{cases}
% \end{equation*}
% where the $c$-th component of $\mathbf{u}^t_c$ has been put into evidence.

% The following expressions hold for $u^t_{c,c}$ and $u^t_{c,j}$ respectively:
% \begin{equation}
% \label{eq:utcc}
% u^t_{c,c} 
% = v^k_{c,c}v^t_{c,c} + \sum_{\substack{j=1 \\ j \neq c}}^{d} v^k_{c,j}v^t_{c,j}
% \end{equation}
% \begin{equation}
% \label{eq:utcj}
% u^t_{c,j} = v^k_{c,j}v^t_{c,c} + v^k_{c,c}v^t_{c,j} + \sum_{\substack{j=1 \\ j \neq y}}^{d-1} v^k_{c,j}v^t_{c,j}
% \end{equation}
% where the $c$-th component of $\mathbf{u}^t_c$ has been put into evidence.

The following expressions hold for $u^t_{c,c}$ and $u^t_{c,j}$, respectively:
\begin{equation}
\label{eq:utcc}
u^t_{c,c} 
= v^k_{c,c}v^t_{c,c} + \sum_{\substack{j=1 \\ j \neq c}}^{d} v^k_{c,j}v^t_{c,j}
\end{equation}
\begin{equation}
\label{eq:utcj}
u^t_{c,j} = v^k_{c,j}v^t_{c,c} + v^k_{c,c}v^t_{c,j} 
+ \sum_{\substack{\ell=1 \\ \ell \neq c,j}}^{d} v^k_{c,\ell}v^t_{c,\ell},
\qquad j \neq c
\end{equation}
where the $c$-th component of $\mathbf{u}^t_c$ has been put into evidence in the first expression, while the second expression considers a generic component $j \neq c$.

Substituting the expressions in Eq.~\ref{eq:vkc_components} and Eq.~\ref{eq:vtc_components} into Eq.~\ref{eq:utcc} and Eq.~\ref{eq:utcj} respectively, we obtain the following equations:
{
\begin{equation*}
\begin{aligned}
u^t_{c,c} = \frac{d-1}{d} v^t_{c,c} - \frac{1}{d} \sum_{\substack{j=1 \\ j \neq c}}^{d} v^t_{c,j} = 
\frac{d-1}{d} \frac{d}{d+1} - \frac{1}{d}(d-1)\left(-\frac{1}{d+1}\right) = \frac{d-1}{d+1} + \frac{d-1}{d(d+1)} = 
% \frac{(d+1)(d-1)}{d(d+1)} = 
1 - \frac{1}{d} = v^k_{c,c}
\end{aligned}
\end{equation*}
}
and
{
\begin{equation*}
% \resizebox{1.02\linewidth}{!}{$
\begin{aligned}
u^t_{c,j} & = -\frac{1}{d}v^t_{c,c} + \frac{d-1}{d}v^t_{c,j} -\frac{1}{d}\sum_{\substack{j=1 \\ j \neq c}}^{d-1} v^t_{c,j} = -\frac{1}{d}\frac{d}{d+1} + \frac{d-1}{d}\left(-\frac{1}{d+1}\right) -\frac{1}{d}(d-2)\left(-\frac{1}{d+1}\right)
= -\frac{1}{d} = v^k_{c,j},
% = -\frac{1}{d+1} - \frac{d-1}{d(d+1)} + \frac{d-2}{d(d+1)}=
\end{aligned}
% $}
\end{equation*}
}
% \[
% = -\frac{d+d-1-d+2}{d(d+1)} = -\frac{1}{d} = \mathbf{v}^k_{c,j}.
% \]
i.e., the components of the vector $\mathbf{u}^t_c$ are identical to those of $\mathbf{v}^k_{c}$.

It results that the vector obtained from \(\mathbf{v}^{t}_c\) is aligned with the prototype vector \(\mathbf{v}^k_{c}\) through the matrix $\mathbf{P}_{t,k}$:
 \[
\mathbf{u}^t_c = \mathbf{P}_{t,k} \, \mathbf{v}^{t}_c = \mathbf{v}^k_{c}.
\]

This proof can be immediately extended to model updates where multiple classes are introduced.
\end{proof}

\section{Pseudocode}
\label{sec:pseudo_code}

This appendix describes how PSP and LSP are integrated into a retrieval system.
Importantly, PSP and LSP are two alternative feature definitions, not post-hoc transformations to be applied to arbitrary embeddings after deployment.
Whether to use features derived from softmax outputs (PSP, Eq.~\ref{eq:softmax_outputs_as_features}) or logits (LSP, Eq.~\ref{eq:logits_as_features}) must therefore be chosen before the first deployment of the retrieval system and then used consistently to build the gallery index and process future queries.
Accordingly, gallery items are indexed once using the chosen representation and stored in the retrieval database.
% The deterministic transformation $\mathbf{P}_{t,k}$ is only required when comparing representations extracted at different update steps, e.g., when the gallery was indexed with a model at step $k$ and an incoming query is processed by an updated model at step $t$.

The deterministic linear transformation $\mathbf{P}_{t,k}$ is required only when the updated model has a larger output space than the gallery model, i.e., when $C^t>C^k$. In this case, it maps the query representation produced at step $t$ into the $C^k$-dimensional feature space of the gallery indexed at step $k$.
When $C^t=C^k$, no cross-update transformation is required. PSP queries are computed using the standard within-step projection $\mathbf{P}_{k,k}$, which is part of the PSP feature definition, whereas LSP queries are obtained by directly normalizing the logits.
In the backward-compatible setting, the gallery is not re-indexed.
This is the inference-time protocol used in our compatibility experiments.

\begin{algorithm}[t]
\caption{Initial Gallery Indexing with PSP or LSP}
\label{alg:index_psp_lsp}
\begin{algorithmic}[1]
\Require Model at gallery step $k$ with $C^k$ classes; gallery item $\mathbf{x}$; representation type $\texttt{repr} \in \{\texttt{PSP}, \texttt{LSP}\}$
\Ensure Stored gallery feature $\mathbf{h}^k(\mathbf{x})$
\State Compute logits $\mathbf{z}^k(\mathbf{x})$
\If{$\texttt{repr} = \texttt{PSP}$}
    \State Compute softmax outputs $\mathbf{p}^k(\mathbf{x}) \gets \sigma(\mathbf{z}^k(\mathbf{x}))$
    \State $\tilde{\mathbf{h}}^k(\mathbf{x}) \gets \mathbf{P}_{k,k}\,\mathbf{p}^k(\mathbf{x})$
    \If{$\|\tilde{\mathbf{h}}^k(\mathbf{x})\|_2 > 0$}
        \State $\mathbf{h}^k(\mathbf{x}) \gets \tilde{\mathbf{h}}^k(\mathbf{x}) / \|\tilde{\mathbf{h}}^k(\mathbf{x})\|_2$
    \Else
        \State $\mathbf{h}^k(\mathbf{x}) \gets \mathbf{0}$
    \EndIf
\ElsIf{$\texttt{repr} = \texttt{LSP}$}
    \If{$\|\mathbf{z}^k(\mathbf{x})\|_2 > 0$}
        \State $\mathbf{h}^k(\mathbf{x}) \gets \mathbf{z}^k(\mathbf{x}) / \|\mathbf{z}^k(\mathbf{x})\|_2$
    \Else
        \State $\mathbf{h}^k(\mathbf{x}) \gets \mathbf{0}$
    \EndIf
\EndIf
\State Store $\mathbf{h}^k(\mathbf{x})$ in the gallery index
\end{algorithmic}
\end{algorithm}

Alg.~\ref{alg:index_psp_lsp} describes the initial gallery indexing procedure with PSP or LSP.
Alg.~\ref{alg:compatible_query_psp_lsp} describes the query-time compatibility correction used only when features extracted at step $t$ must be compared against a gallery indexed at step $k$.

\begin{algorithm}[t]
\caption{Query-Side Compatibility Across Model Updates}
\label{alg:compatible_query_psp_lsp}
\begin{algorithmic}[1]
\Require Gallery indexed at step $k$; updated model at step $t$ with
$C^t \geq C^k$; query $\mathbf{x}$; representation type
$\texttt{repr} \in \{\texttt{PSP},\texttt{LSP}\}$
\Ensure Query feature compatible with the gallery feature space at step $k$

\State Compute logits $\mathbf{z}^t(\mathbf{x})$

\If{$\texttt{repr}=\texttt{PSP}$}
    \State $\mathbf{p}^t(\mathbf{x}) \gets
    \sigma(\mathbf{z}^t(\mathbf{x}))$
    \If{$C^t>C^k$}
        \State $\tilde{\mathbf{h}}^{t\rightarrow k}(\mathbf{x})
        \gets \mathbf{P}_{t,k}\mathbf{p}^t(\mathbf{x})$
    \Else
        \State $\tilde{\mathbf{h}}^{t\rightarrow k}(\mathbf{x})
        \gets \mathbf{P}_{k,k}\mathbf{p}^t(\mathbf{x})$
    \EndIf

\ElsIf{$\texttt{repr}=\texttt{LSP}$}
    \If{$C^t>C^k$}
        \State $\tilde{\mathbf{h}}^{t\rightarrow k}(\mathbf{x})
        \gets \mathbf{P}_{t,k}\mathbf{z}^t(\mathbf{x})$
    \Else
        \State $\tilde{\mathbf{h}}^{t\rightarrow k}(\mathbf{x})
        \gets \mathbf{z}^t(\mathbf{x})$
    \EndIf
\EndIf

\If{$\|\tilde{\mathbf{h}}^{t\rightarrow k}(\mathbf{x})\|_2>0$}
    \State $\mathbf{h}^{t\rightarrow k}(\mathbf{x})
    \gets
    \tilde{\mathbf{h}}^{t\rightarrow k}(\mathbf{x})/
    \|\tilde{\mathbf{h}}^{t\rightarrow k}(\mathbf{x})\|_2$
\Else
    \State $\mathbf{h}^{t\rightarrow k}(\mathbf{x})\gets\mathbf{0}$
\EndIf

\State Compare $\mathbf{h}^{t\rightarrow k}(\mathbf{x})$ with
$\{\mathbf{h}^k(\mathbf{g})\}_{\mathbf{g}\in\mathcal{G}}$
\State Return the ranked gallery items
\end{algorithmic}
\end{algorithm}

Alg.~\ref{alg:index_psp_lsp} is used to define and store gallery features when the retrieval system is first deployed.
If, after a model update, the gallery is fully re-indexed, the same algorithm is simply applied again to all gallery items using the updated model and the same chosen representation type.
By contrast, Alg.~\ref{alg:compatible_query_psp_lsp} is used only in the backward-compatible retrieval setting, where the gallery remains fixed at step $k$ and only the incoming queries are processed by the updated model at step $t$.
In this case, PSP/LSP are not re-applied to the stored gallery features; only the query-side representation is corrected through $\mathbf{P}_{t,k}$ to enable comparison with the gallery without re-indexing.

We emphasize that the deterministic transformation $\mathbf{P}_{t,k}$ is not part of the generic feature-definition step of the retrieval system.
It is only required when features extracted at different update steps with an increased number of classes must be compared.
During training, there is no dependence on previous models, and no additional losses, alignment modules, or architectural modifications are introduced.

\newpage

\section{Theoretical Analysis of the Backward-Compatibility Definition}\label{sec:app_theo_compatibility}

\begin{lemma}
\label{theo:lemma}
Let \( \mathbf{X}_1, \mathbf{X}_2  \in \mathbb{R}^d  \) be independent random vectors in the $(d-1)$-hypersphere, each following a von Mises--Fisher distribution {\rm vMF($\boldsymbol{\mu}, \kappa$)}.
Then the expected cosine distance between \( \mathbf{X}_1 \) and \( \mathbf{X}_2 \) depends on the concentrations $\kappa_1$, $\kappa_2$ and the mean directions $\boldsymbol{\mu}_1,\boldsymbol{\mu}_2$ of the two distributions as follows:
\begin{equation*}
\mathbb{E}\Big[1 - \mathbf{X}_1^{\top} \mathbf{X}_2\Big] = 1 - m_d(\kappa_1) \, m_d(\kappa_2) \, \boldsymbol{\mu}_1^\top \boldsymbol{\mu}_2.
\end{equation*}
\end{lemma}
\begin{proof}
    
Let $\mathbf{X}_1, \mathbf{X}_2$  be two independent random vectors from two vMF distributions with parameters respectively $(\boldsymbol{\mu}_1, \kappa_1)$ and $(\boldsymbol{\mu}_2, \kappa_2)$, where $\boldsymbol{\mu_1}$, $\boldsymbol{\mu_2}$ are unit vectors indicating the mean directions around which the data are concentrated and $\kappa_1$, $\kappa_2$ are non-negative scalars that determine the concentrations of the distributions around their mean (large values of $\kappa$ correspond to high concentration).
By their independence, the expected cosine similarity between  $\mathbf{X}_1, \mathbf{X}_2$ is given by the product of their expectations:
\begin{equation}
\label{eq:indp_vector_mean}
\mathbb{E}\left[ \mathbf{X}_1^{\top} \mathbf{X}_2 \right] = \mathbb{E}\left[ \mathbf{X}_1 \right]^{\top} \mathbb{E}\left[ \mathbf{X}_2 \right]
\end{equation}
where the expected values of $\mathbf{X}_1$ and $\mathbf{X}_2$ are respectively \citep{Banerjee2005ClusteringOT}:
\begin{align}
\mathbb{E}\left[ \mathbf{X}_1 \right] &= m_d\left( \kappa_1 \right) \boldsymbol{\mu}_1, \label{eq:mean_X} \\
\mathbb{E}\left[ \mathbf{X}_2 \right] &= m_d\left( \kappa_2 \right) \boldsymbol{\mu}_2, \label{eq:mean_Y} 
\end{align}
where $m_d(\kappa)$ is the expected cosine of the angle between the random vector and $\boldsymbol{\mu}$, which is expressed as: 
\begin{equation*}
m_d(\kappa) = \frac{I_{\frac{d}{2}}(\kappa)}{I_{\frac{d}{2} - 1}(\kappa)}
\end{equation*}
with \( I_\nu \) the modified Bessel function of the first kind of order \( \nu \).

From Eq.~\ref{eq:indp_vector_mean}, Eq.~\ref{eq:mean_X}, and Eq.~\ref{eq:mean_Y}, we obtain:
\begin{equation*}
\begin{aligned}
\mathbb{E}\left[ \mathbf{X}_1^{\top} \mathbf{X}_2 \right] &= \left( m_d\left( \kappa_1 \right) \boldsymbol{\mu}_1 \right)^{\top} \left( m_d\left( \kappa_2 \right) \boldsymbol{\mu}_2 \right) =  m_d\left( \kappa_1 \right) m_d\left( \kappa_2 \right) \boldsymbol{\mu}_1^{\top} \boldsymbol{\mu}_2.
\end{aligned}
\end{equation*}
where $\boldsymbol{\mu}_1^{\top} \boldsymbol{\mu}_2$  is the cosine of the angle between mean directions $\boldsymbol{\mu}_1$ and  $\boldsymbol{\mu}_2$.

Since the cosine distance between \( \mathbf{X}_1 \) and \( \mathbf{X}_2 \) is $ (1 - \mathbf{X}_1^\top \mathbf{X}_2)$, it follows that:
\begin{equation}
\label{eq:expected_cos_distance}
\begin{aligned}
\mathbb{E}\Big[1 - \mathbf{X}_1^{\top} \mathbf{X}_2\Big]
& = 1 - \mathbb{E}\Big[\mathbf{X}_1^{\top} \mathbf{X}_2\Big]
= 1 - m_d(\kappa_1) \, m_d(\kappa_2) \, \boldsymbol{\mu}_1^{\top} \boldsymbol{\mu}_2,
\end{aligned}
\end{equation}
which provides an analytical form that relates the expected distance between two random vectors to the parameters of the two vMF distributions. 
\end{proof}

\begin{theowithref}[\ref{theo:compatibility} (\textit{\textbf{Compatibility Theorem}})]

Consider hyperspherical representations obtained at update steps $k$ and $t$, with $t > k$.
Assume that (i) the features of each class $c$ are drawn independently from a von Mises--Fisher distribution {\rm vMF($\boldsymbol{\mu}_c,\kappa_c$)} whose mean direction $\boldsymbol{\mu}_c$ is the corresponding class prototype, and (ii) the concentration $\kappa_c$ of each class distribution increases from step $k$ to step $t$.
Then the compatibility inequalities of Eq.~\ref{eq:first} and Eq.~\ref{eq:second} are satisfied in expectation if (a) corresponding class prototypes are aligned across updates and (b) the angle between prototypes of different classes is greater than $\pi/2$.
\end{theowithref}

\begin{proof}
Consider each feature of the model represented as an independent random vector. Then the compatibility inequalities of Eq.~\ref{eq:first} and Eq.~\ref{eq:second} in expectation can be expressed as: 
\begin{subequations}\label{eq:compatibility_def_expectation_pre}
\begin{align}
\mathbb{E} \left[ 1 - (\mathbf{X}_1^t)^\top \mathbf{X}_2^k \right] \leq \mathbb{E} \left[ 1 - (\mathbf{X}_1^k)^\top \mathbf{X}_2^k \right] \label{eq:compatibility_def_expectation_pre_same} \qquad & \text{for the case of same class} \\
\mathbb{E} \left[ 1 - (\mathbf{X}_1^t)^\top \mathbf{X}_2^k \right] \geq \mathbb{E} \left[ 1 - (\mathbf{X}_1^k)^\top \mathbf{X}_2^k \right] \label{eq:compatibility_def_expectation_pre_diff} \qquad & \text{for the case of different classes} 
\end{align}
\end{subequations}

where \( \mathbf{X}_1, \mathbf{X}_2  \in \mathbb{R}^d  \) are two independent random vectors with $\mathbb{E} \left[ 1 - (\mathbf{X}_1)^\top \mathbf{X}_2 \right]$ their expected cosine distance, and $t$ and $k$ are two distinct update steps.

Assuming that features follow a vMF distribution, according to Lemma~\ref{theo:lemma} the expected cosine distance between two independent random vectors \( \mathbf{X}_1, \mathbf{X}_2  \in \mathbb{R}^d  \) has the form:
\begin{equation}\label{eq:theo_semp}
\mathbb{E}\Big[1 - \mathbf{X}_1^{\top} \mathbf{X}_2\Big] = 1 - m_d(\kappa_1) \, m_d(\kappa_2) \, \boldsymbol{\mu}_1^\top \boldsymbol{\mu}_2
\end{equation}
where $\boldsymbol{\mu}_1,\boldsymbol{\mu}_2$ are the mean directions of the two distributions and $\kappa_1$, $\kappa_2$ their concentrations. Substituting Eq.~\ref{eq:theo_semp} into  Eq.~\ref{eq:compatibility_def_expectation_pre_same} and Eq.~\ref{eq:compatibility_def_expectation_pre_diff} gives:
\begin{small}
\begin{subequations}\label{eq:compatibility_def_expectation}
\begin{align}
 1 - m_d(\kappa_1^t) \, m_d(\kappa_2^k) \, (\boldsymbol{\mu}_1^t)^\top \boldsymbol{\mu}_2^k
 \leq 1 - m_d(\kappa_1^k)m_d(\kappa_2^k) \, (\boldsymbol{\mu}_1^k)^\top \boldsymbol{\mu}_2^k
& \quad \text{for the case of same class}  \label{eq:first_expectation}\\
 % \vspace{10pt}
 1 - m_d(\kappa_1^t) \, m_d(\kappa_2^k) \, (\boldsymbol{\mu}_1^t)^\top \boldsymbol{\mu}_2^k
 \geq  1 - m_d(\kappa_1^k)m_d(\kappa_2^k) \, (\boldsymbol{\mu}_1^k)^\top \boldsymbol{\mu}_2^k 
   & \quad \text{for the case of different classes} \label{eq:second_expectation}
\end{align}
\end{subequations}

\end{small}

Then, since $m_d(\kappa)$ is a strictly increasing function of the concentration parameter $\kappa$ \citep{Banerjee2005ClusteringOT}, and $\kappa^t > \kappa^k > 0$ across updates\footnote{We use ``increases'' in the strict sense, i.e., $\kappa^t > \kappa^k$, thus excluding the case $\kappa^t = \kappa^k$, for which the argument yields equality rather than a strict improvement. We also exclude the degenerate case $\kappa^k = 0$: the vMF distribution then reduces to the uniform distribution on the hypersphere, for which $m_d(0)=0$ and no preferred mean direction is identifiable, and therefore it does not describe a class distribution concentrated around its prototype.}, it follows that $
m_d(\kappa_1^t) \, m_d(\kappa_2^k)
>
m_d(\kappa_1^k) \, m_d(\kappa_2^k).
$

% Then, given that parameter $m_d(\kappa)$ is a strictly increasing function of the concentration parameter $\kappa$ \citep{Banerjee2005ClusteringOT}, and mean resultant length $m_d(\kappa^t) > m_d(\kappa^k)$ across updates\footnote{We discarded the triavial case when the new model at time t is equal to the old model at time k so that $m_d(\kappa^t) = m_d(\kappa^k)$.
% }, it follows:  
% $m_d(\kappa_1^t) \, m_d(\kappa_2^k) > m_d(\kappa_1^k) \, m_d(\kappa_2^k).$ 

It follows that:
\begin{itemize}
\item Eq.~\ref{eq:first_expectation} is satisfied when the mean directions of the same class remain the same across updates, i.e., $(\boldsymbol{\mu}_1^t)^\top \boldsymbol{\mu}_2^k = 1$.

\begin{figure*}
    \centering
    \begin{subfigure}{0.48\linewidth}
        \includegraphics[width=0.97\linewidth]{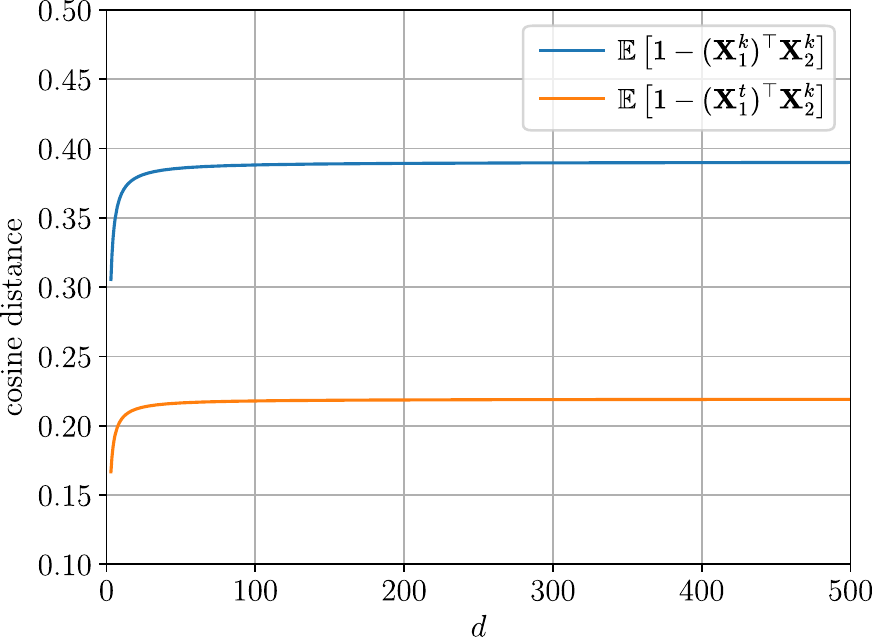}
        \caption{Same class}
        \label{fig:theo_inequalities_same}
    \end{subfigure}
    \begin{subfigure}{0.48\linewidth}
        \includegraphics[width=1\linewidth]{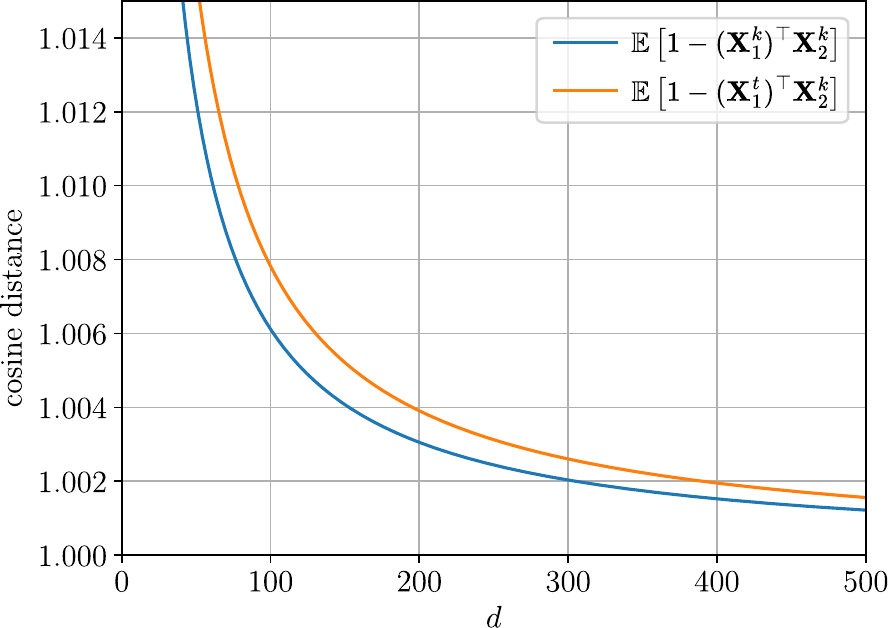}
        \caption{Different classes}
        \label{fig:theo_inequalities_diff}
    \end{subfigure}
    % \caption{Plot of the expected cosine distances of regular simplex representations for (a) same-class (Eq.~\ref{eq:compatibility_def_expectation_pre_same}) and (b) different-class (Eq.~\ref{eq:compatibility_def_expectation_pre_diff}) comparisons, calculated using the closed-form solutions in Eq.~\ref{eq:first_expectation} and Eq.~\ref{eq:second_expectation}. 
    % The results show that the compatibility inequalities are satisfied in expectation for every value of the feature space dimension $d$. 
    % }
    \caption{Expected cosine distances for aligned regular-simplex representations in the (a) same-class and (b) different-class cases, corresponding to Eq.~\ref{eq:compatibility_def_expectation_pre_same} and Eq.~\ref{eq:compatibility_def_expectation_pre_diff}, respectively. The curves are obtained from the closed-form expected-distance expressions in Eq.~\ref{eq:first_expectation} and Eq.~\ref{eq:second_expectation}, under the assumptions of Theorem~\ref{theo:compatibility}, with the updated class distributions having higher concentration than the old ones. Across all considered feature-space dimensions $d$, the expected distance between features from the updated and old models is lower than the corresponding old-model distance for same-class pairs and higher for different-class pairs, illustrating that the compatibility inequalities are satisfied in expectation.}

    \label{fig:theo_inequalities}
\end{figure*}

\item Eq.~\ref{eq:second_expectation} is satisfied when the angle between prototypes of different classes is greater than $\pi / 2$, i.e., $( \boldsymbol{\mu}_1^k)^\top \boldsymbol{\mu}_2^k = (\boldsymbol{\mu}_1^t)^\top \boldsymbol{\mu}_2^k < 0$---note that $( \boldsymbol{\mu}_1^k)^\top \boldsymbol{\mu}_2^k = (\boldsymbol{\mu}_1^t)^\top \boldsymbol{\mu}_2^k$ due to the alignment of the prototypes required to satisfy Eq.~\ref{eq:first_expectation}.
\end{itemize}

These two requirements are precisely conditions (a) and (b) of Theorem~\ref{theo:compatibility}, which concludes the proof.

\end{proof}

\begin{remark}[\textit{Dependence on concentration}]
The guarantee of Theorem~\ref{theo:compatibility} is conditional on the concentration of the learned representations.
From Eq.~\ref{eq:expected_cos_distance}, for fixed prototype geometry, the separation between same-class and different-class expected cosine distances is governed by the mean resultant length of the vMF distribution, which is strictly increasing in $\kappa$ \citep{DBLP:journals/cstat/HornikG14}.
As $\kappa$ approaches zero, the mean resultant length approaches zero and the expected cosine distances approach $1$, reducing the separation between the two cases.
Thus, the theorem does not impose a condition on classification accuracy; its distributional requirement is instead that the updated representation is more concentrated around its class prototypes than the old one.
% The NC1 curves in Fig.~\ref{fig:neural_collpase_proof} and Appendix~\ref{sec:app_neural_collapse_logits} provide empirical support for this trend toward convergence.
\end{remark}

\begin{remark}[\textit{Role of the vMF assumption}]
\label{rem:role_vmf}
Theorem~\ref{theo:compatibility} is stated under the vMF assumption, which is adopted as a tractable modeling assumption for class-conditional hyperspherical representations together with the explicit first-moment expression
$$
\mathbb{E}\left[\mathbf{X}\right]
=
m_d\left(\kappa\right)\boldsymbol{\mu}.
$$
At the expectation level, this first-moment characterization constitutes the distribution-specific quantity entering the proof. Indeed, for independent unit-norm random vectors $\mathbf{X}_1$ and $\mathbf{X}_2$,
$$
\mathbb{E}\!\left[d_{\mathrm{cos}}(\mathbf{X}_1,\mathbf{X}_2)\right]
=
1-
\mathbb{E}\!\left[\mathbf{X}_1\right]^\top
\mathbb{E}\!\left[\mathbf{X}_2\right],
$$
where $d_{\cos}(\mathbf{x},\mathbf{y}) = 1 - \mathbf{x}^\top\mathbf{y}$.

Hence, the vMF assumption enters the expected-distance analysis through its first moment. For a class-conditional distribution on the unit hypersphere with a nonzero first moment, this quantity can be expressed as
$$
\mathbb{E}[\mathbf{X}]
=
\rho\boldsymbol{\mu},
$$
where $\boldsymbol{\mu}$ denotes the mean direction and
$\rho=\|\mathbb{E}[\mathbf{X}]\|$ denotes the mean resultant length
\citep{mardia1975statistics,mardia2009directional}.
This formulation identifies the mean direction and mean resultant length as the first-moment quantities entering the expected-distance expression.
For the vMF model considered in Theorem~\ref{theo:compatibility}, $\rho=m_d(\kappa)$, and the strict monotonicity of $m_d(\kappa)$ with respect to $\kappa$
\citep{DBLP:journals/cstat/HornikG14}
provides an explicit relation between the concentration parameter $\kappa$ and the mean resultant length appearing in the proof.
\end{remark}

As a direct consequence of Theorem~\ref{theo:compatibility}, the following corollary holds.

\begin{corwithref}[\ref{cor:regulsimplex_compatibility} (\textit{Compatibility of Aligned Regular-Simplex Hyperspherical Representations})]
Under the assumptions of Theorem~\ref{theo:compatibility}, hyperspherical representations whose class prototypes form a regular simplex centered at the origin and remain aligned across updates satisfy the compatibility inequalities of Eq.~\ref{eq:first} and Eq.~\ref{eq:second} in expectation.
\end{corwithref}

\begin{proof}
Let the normalized class prototypes at update step $k$ form a regular simplex centered at the origin. For any two distinct classes $i\neq j$, their pairwise inner product is

$$
(\boldsymbol{\mu}_i^k)^\top\boldsymbol{\mu}_j^k=
-\frac{1}{C^k-1}
<0,
$$

where $C^k$ denotes the number of classes defining the simplex at step $k$. Hence, the angle between any two distinct class prototypes is greater than $\pi/2$.

For the same-class case, alignment of corresponding prototypes across updates gives $
(\boldsymbol{\mu}_i^t)^\top\boldsymbol{\mu}_i^k=1,$
which satisfies the geometric condition required by Eq.~\ref{eq:first_expectation}. For two different classes $i\neq j$, prototype alignment gives
$$
(\boldsymbol{\mu}_i^t)^\top\boldsymbol{\mu}_j^k=
(\boldsymbol{\mu}_i^k)^\top\boldsymbol{\mu}_j^k=
-\frac{1}{C^k-1}
<0,
$$
which satisfies the geometric condition required by Eq.~\ref{eq:second_expectation}.

Therefore, all geometric conditions of Theorem~\ref{theo:compatibility} are satisfied, and, under its vMF and increasing-concentration assumptions, the compatibility inequalities of Eq.~\ref{eq:first} and Eq.~\ref{eq:second} hold in expectation.
\end{proof}

Fig.~\ref{fig:theo_inequalities} illustrates the expected compatibility inequalities characterized above for hyperspherical regular-simplex representations across different representation dimensions. The two update steps are denoted by $k$, corresponding to the old model, and $t$, corresponding to the updated model, whose class-conditional vMF distributions have higher concentration. In the same-class case, Fig.~\ref{fig:theo_inequalities_same} shows that

$$
\mathbb{E}\left[1-(\mathbf{X}_1^t)^\top\mathbf{X}_2^k\right]
<
\mathbb{E}\left[1-(\mathbf{X}_1^k)^\top\mathbf{X}_2^k\right],
$$

consistent with Eq.~\ref{eq:first}. Conversely, for different classes, Fig.~\ref{fig:theo_inequalities_diff} shows that

$$
\mathbb{E}\left[1-(\mathbf{X}_1^t)^\top\mathbf{X}_2^k\right]
>
\mathbb{E}\left[1-(\mathbf{X}_1^k)^\top\mathbf{X}_2^k\right],
$$

consistent with Eq.~\ref{eq:second}.

\begin{remark}[\textit{Relation to Simplex ETF}]
\label{rem:simplex_etf}
The regular-simplex geometry considered in Corollary~\ref{cor:regulsimplex_compatibility} is closely related to the Simplex ETF structure arising in Neural Collapse~\citep{papyan2020prevalence}. Within a single representation space, Neural Collapse characterizes the centered class means as approaching a Simplex ETF configuration, in which the normalized class prototypes have constant pairwise inner products
$$
\boldsymbol{\mu}_i^\top\boldsymbol{\mu}_j=-\frac{1}{C-1},
\qquad i\neq j,
$$
and therefore pairwise angles greater than $\pi/2$.

This within-model geometric property, however, does not by itself imply alignment across independently obtained representation spaces. In particular, the Simplex ETF geometry is invariant under a common orthogonal transformation: applying the same orthogonal transformation to all prototypes preserves their norms and pairwise inner products while changing their absolute orientation. Consequently, two models may independently exhibit the same Simplex ETF geometry while their corresponding class prototypes are not aligned across updates.

Thus, the Simplex ETF property guarantees the obtuse inter-class geometry required by Theorem~\ref{theo:compatibility}, but not the cross-update alignment of corresponding prototypes. When this additional alignment condition is satisfied, and under the vMF and increasing-concentration assumptions of Theorem~\ref{theo:compatibility}, an aligned Simplex ETF configuration falls within the regular-simplex setting of Corollary~\ref{cor:regulsimplex_compatibility} and therefore satisfies the compatibility inequalities in expectation.
\end{remark}

\section{Empirical Evidence of Increased Concentration Across Model Updates}\label{sec:concentration}

In this appendix, we empirically assess the concentration-increase assumption discussed in Sec.~\ref{sec:theo_section}, showing that the distributions of both logits and softmax outputs become more compact around their respective class prototypes as model updates progress.
To quantify the compactness, we use the mean cosine distance between the feature representations and their respective class prototypes.
Specifically, let $\mathcal{Z}_c = \{\mathbf{z}_i^c\}_{i=1}^{N_c}$ denote the set of extracted logits for class $c$, where $N_c$ represents the number of samples belonging to class $c$ in the dataset.
The class prototype is defined as the mean of the $\ell_2$-normalized extracted logits, $\bar{\mathbf{z}}_c = \frac{1}{N_c} \sum_{i=1}^{N_c}\frac{\mathbf{z}_i^c}{\|\mathbf{z}_i^c\|}$, and the mean cosine distance between each feature and the prototype of class $c$ is given by
\begin{equation*}
\delta_c = \frac{1}{N_c} \sum_{i=1}^{N_c} \left( 1 - \frac{\mathbf{z}_i^c \cdot \bar{\mathbf{z}}_c}{\|\mathbf{z}_i^c\| \, \|\bar{\mathbf{z}}_c\|} \right) = 1-\|\bar{\mathbf{z}}_c\|.
\end{equation*}
The same quantities are computed for the softmax outputs by replacing $\mathbf{z}_i^c$ with $\sigma(\mathbf{z}_i^c)$.
Lower values of $\delta_c$ correspond to larger empirical mean resultant lengths (Remark~\ref{rem:role_vmf}) and therefore indicate a tighter concentration of the normalized representations around their class mean directions.

\begin{figure*}[t]
    \centering
    \includegraphics[height=0.20in]{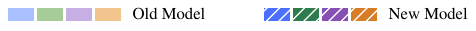}
    % Crop only blank left padding from the right PDF. The widths preserve
    % equal image scales and make the middle gap match the other panel gaps.
    \begin{subfigure}[t]{0.50705\textwidth}
        \centering
        \includegraphics[width=\linewidth]{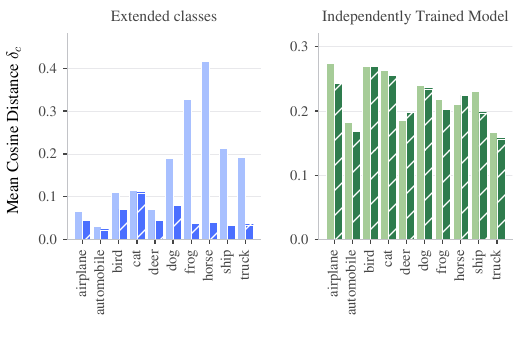}
        \vspace{-0.8cm}
        \caption{Logits.}
        \label{fig:concentration-logits}
    \end{subfigure}%
    \begin{subfigure}[t]{0.49295\textwidth}
        \centering
        \includegraphics[width=\linewidth,trim=0.096in 0 0 0,clip]{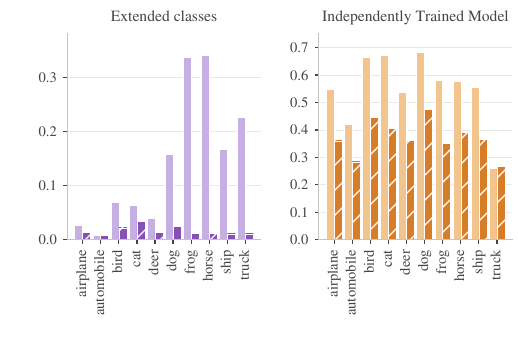}
        \vspace{-0.8cm}
        \caption{Softmax outputs.}
        \label{fig:concentration-softmax}
    \end{subfigure}
    \caption{Per-class dispersion $\delta_c$ of the $\ell_2$-normalized logits (a) and softmax outputs (b) on CIFAR-10 for the old model (solid bars) and the new model (hatched bars). In each subfigure, the left panel shows the extended-classes update from 5 to 10 classes, where the last 5 classes are unknown to the old model, and the right panel shows the update between independently trained models, with ResNet-18 as the old model and ResNet-50 as the new model, both pretrained on ImageNet-1K. Lower values of $\delta_c$ indicate a tighter concentration around the class mean direction.}
    \label{fig:concentration}
\end{figure*}

We conduct experiments on the CIFAR-10 dataset~\citep{Krizhevsky2009LearningML} under two scenarios analogous to those in Sec.~\ref{sec:when_experimental_results}.

\textbf{Extended Classes Scenario.} In this scenario, we investigate the concentration of the representations when the number of classes is expanded during a model update. The old model is trained on the first 5 classes of CIFAR-10, and the updated model is trained on the complete set of 10 classes. Both models use a ResNet-18 architecture~\citep{he2016deep} trained for 100 epochs using stochastic gradient descent (SGD) with an initial learning rate of 0.1, momentum of 0.9, and a cosine annealing learning rate schedule.
The results shown in Fig.~\ref{fig:concentration} demonstrate that both the logits and the softmax outputs of the updated model are more concentrated than those of the old model for all 10 classes. This is consistent with the semi-open-set scenario, where the new classes are unknown to the old model but are known to the updated model, thus explaining the significant gap in concentration in the last 5 classes. The increase in concentration also holds for the 5 shared classes, on which both models are already highly concentrated, supporting the assumption of increasing concentration with class expansion during model updates.

\textbf{Independently Trained Models in Public Repositories Scenario.} We compare two models pretrained on ImageNet-1K and obtained from torchvision~\citep{marcel2010torchvision}: ResNet-18 as the old model, and ResNet-50 as the updated model. We evaluate the concentration metric $\delta_c$ on the CIFAR-10 test set, with images resized to $224\times224$ by bicubic interpolation. As shown in Fig.~\ref{fig:concentration}, the updated model yields more concentrated representations than the old model, even in an open-set scenario where the evaluation distribution differs from the training one: the logits are more concentrated for 7 of the 10 classes, and the softmax outputs for all 10 classes.

The observed overall increase in logit concentration is consistent with the collapse behavior discussed in Sec.~\ref{sec:logits_theory}, and the softmax outputs exhibit the same trend.
In the ideal Neural Collapse limit, the two representations behave identically in this respect: the logits of each class converge to a common vector, and so do the corresponding softmax outputs~\citep{papyan2020prevalence}.

Since the hyperspherical simplex representations in Eq.~\ref{eq:logits_as_features} and Eq.~\ref{eq:softmax_outputs_as_features} are constructed from the logits and softmax outputs, respectively, these results directly support assumption~(ii) for both LSP and PSP.
Overall, the results provide empirical evidence of increasing first-moment concentration across model updates. Under the vMF modeling assumption adopted in Theorem~\ref{theo:compatibility}, this behavior is consistent with an increase in the corresponding concentration parameter $\kappa$.
At the same time, the few classes whose logits become less concentrated under domain shift show that the concentration-increase assumption holds predominantly but need not hold for every class.

\section{Empirical Analysis of the Centering Projection and the Angular Correction}
\label{app:psp_ablation}

In this appendix, we empirically examine the contribution of the two operations on which the design of PSP relies: the centering projection $\mathbf{P}_{k,k}$ and the deterministic transformation $\mathbf{P}_{t,k}$ for angular correction, motivated by condition~\ref{cond:obtuse} and condition~\ref{cond:align} of Theorem~\ref{theo:compatibility}, respectively (Sec.~\ref{sec:psp} and Sec.~\ref{sec:projection}).

\textbf{Experimental setting.}
We consider an extended-classes update on CIFAR-10~\citep{Krizhevsky2009LearningML} with a ResNet-18 backbone~\citep{he2016deep}: the old model is trained on the first 5 classes and the updated model on all 10 classes.
All models are trained from scratch for 60 epochs using SGD with an initial learning rate of 0.1, momentum of 0.9, weight decay of $5\cdot10^{-4}$, batch size of 128, and cosine annealing.
Each model is trained with 3 different seeds, and all $3\times3$ combinations of old and updated models are evaluated, yielding a total of 9 model pairs.
The old models reach a test accuracy of 95.10--95.70\% on their 5 classes, and the updated models 94.57--94.97\% on all 10 classes.

The inequalities of Def.~\ref{def:compatibility} are evaluated with the cosine distance $d$ on the 5{,}000 test images of the 5 shared classes, over all ordered pairs $(i,j)$ with $i\neq j$, which form the set $\mathcal{S}$ of 4{,}995{,}000 same-class pairs and the set $\mathcal{D}$ of 20{,}000{,}000 different-class pairs per model pair.
For each inequality, we compute the average margin
\begin{equation}
\gamma_{\mathrm{same}} = \widehat{\mathbb{E}}_{\mathcal{S}}\big[d(\mathbf{h}^k_i,\mathbf{h}^k_j)\big]-\widehat{\mathbb{E}}_{\mathcal{S}}\big[d(\mathbf{h}^t_i,\mathbf{h}^k_j)\big],
\qquad
\gamma_{\mathrm{diff}} = \widehat{\mathbb{E}}_{\mathcal{D}}\big[d(\mathbf{h}^t_i,\mathbf{h}^k_j)\big]-\widehat{\mathbb{E}}_{\mathcal{D}}\big[d(\mathbf{h}^k_i,\mathbf{h}^k_j)\big],
\label{eq:avg_margins}
\end{equation}
where $\widehat{\mathbb{E}}_{\mathcal{P}}$ denotes the empirical average over the pairs in $\mathcal{P}$.
A positive margin indicates that the corresponding inequality is satisfied on average over the evaluated pairs, providing an empirical counterpart to the expectation-level inequalities of Theorem~\ref{theo:compatibility}.

We compare four configurations, which differ only in the operations applied to the softmax outputs before $\ell_2$-normalization.
In all of them, the gallery features $\mathbf{h}^k$ are extracted with the old model and the query features $\mathbf{h}^t$ with the updated model; when their dimensions differ, $\mathbf{h}^k$ is zero-padded to $C^t$ dimensions.
The configurations are defined as follows:
\begin{itemize}
\item $\ell_2$-Softmax: directly $\ell_2$-normalized softmax outputs for both models, without either operation.
\item PSP w/o $\mathbf{P}_{t,k}$: both models use the representation of Eq.~\ref{eq:softmax_outputs_as_features}, each centered with $\mathbf{P}_{k,k}$ and $\mathbf{P}_{t,t}$, respectively, without angular correction.
\item PSP w/o $\mathbf{P}_{k,k}$: $\ell_2$-normalized softmax outputs for the gallery, and queries mapped by $[\mathbf{I}_{C^k}\,|\,\mathbf{0}]$, i.e., $\mathbf{P}_{t,k}$ without centering, which only discards the coordinates of the new classes.
\item PSP: the proposed hyperspherical representation of Eq.~\ref{eq:softmax_outputs_as_features}, with queries mapped by $\mathbf{P}_{t,k}$.
\end{itemize}

\begin{table}[!htbp]
\centering
\caption{Analysis of the centering projection $\mathbf{P}_{k,k}$ and the deterministic transformation $\mathbf{P}_{t,k}$ on CIFAR-10 with ResNet-18 under class expansion ($5\to10$ classes). Average margins of Eq.~\ref{eq:avg_margins} (mean $\pm$ std over 9 model pairs; positive values indicate that the inequality is satisfied on average) and number of compatible pairs, i.e., model pairs for which both margins are positive. $^{*}$Without the centering projection, i.e., $[\mathbf{I}_{C^k}\,|\,\mathbf{0}]$ in place of $\mathbf{P}_{t,k}$. Dark and light blue indicate respectively the highest and the second-highest value for each column.}
\label{tab:psp_ablation}
\begin{tabular}{lcc ccc}
\toprule
\textsc{Method} & $\mathbf{P}_{k,k}$ & $\mathbf{P}_{t,k}$ & $\gamma_{\mathrm{same}}$ & $\gamma_{\mathrm{diff}}$ & \textsc{Compatible pairs} \\
\midrule
$\ell_2$-Softmax & -- & -- & $-0.0111\pm0.0030$ & \textcolor[RGB]{65,105,225}{\boldmath$-0.0055\pm0.0009$} & 0/9 \\
PSP w/o $\mathbf{P}_{t,k}$ & \checkmark & -- & $-0.0568\pm0.0033$ & $-0.0142\pm0.0008$ & 0/9 \\
PSP w/o $\mathbf{P}_{k,k}$ & -- & \checkmark$^{*}$ & \textcolor[RGB]{65,105,225}{\boldmath$+0.0120\pm0.0019$} & $-0.0133\pm0.0010$ & 0/9 \\
\midrule
PSP & \checkmark & \checkmark & \textcolor[RGB]{0,0,139}{\boldmath$+0.0155\pm0.0025$} & \textcolor[RGB]{0,0,139}{\boldmath$+0.0039\pm0.0006$} & \textcolor[RGB]{0,0,139}{\textbf{9/9}} \\
\bottomrule
\end{tabular}
\end{table}

In PSP w/o $\mathbf{P}_{k,k}$, i.e., without centering, restricting the queries with $[\mathbf{I}_{C^k}\,|\,\mathbf{0}]$ makes $\gamma_{\mathrm{same}}$ positive in all model pairs, as it discards the probability mass that the updated model assigns to the new classes, which has no counterpart in the gallery features.
However, $\gamma_{\mathrm{diff}}$ remains negative in all model pairs for both $\ell_2$-Softmax and PSP w/o $\mathbf{P}_{k,k}$, as expected from the analysis of Sec.~\ref{sec:psp}: since the softmax outputs lie in the non-negative orthant, the angle between prototypes of different classes is at most $\pi/2$ and condition~\ref{cond:obtuse} of Theorem~\ref{theo:compatibility} is not satisfied.
In this case, the inner products between features of different classes are non-negative, so that an increase in concentration cannot yield a positive different-class margin.
The centering projection $\mathbf{P}_{k,k}$ removes this limitation by mapping the prototypes onto a regular simplex centered at the origin, which, together with the alignment restored by $\mathbf{P}_{t,k}$, accounts for the positive margins of PSP.

Overall, the two operations play complementary roles: $\mathbf{P}_{k,k}$ establishes angles greater than $\pi/2$ between prototypes of different classes (condition~\ref{cond:obtuse}), while $\mathbf{P}_{t,k}$ restores the alignment of corresponding prototypes that this centered geometry loses under class expansion (condition~\ref{cond:align}).
Among the evaluated configurations, only their combination, i.e., PSP, yields positive average margins for both inequalities of Def.~\ref{def:compatibility} in all model pairs, providing empirical support for the geometric conditions identified by Theorem~\ref{theo:compatibility}.

\section{Implementation Details For Standard Compatibility Benchmarks} \label{sec:implementation_details}
In the following, we report the implementation details used in the experiments of Sec.~\ref{sec:experiments_extended_classes}. 

\textbf{Experiments on CIFAR-100~{\normalfont \citep{Krizhevsky2009LearningML}.}} 
For this set of experiments, the image size is 32$\times$32. Both ResNet-18 and ResNet-50 architectures are used with the following hyperparameters for training: number of epochs 120; batch size 128; SGD optimizer with an initial learning rate of 0.1, reduced by a factor of 10 at epochs 80 and 100; SGD momentum 0.9; weight decay $5 \cdot 10^{-4}$.
A temperature factor of $12$ is used in the cross-entropy loss to scale logit vectors during training.  
Training images are subjected to random cropping, horizontal flipping, and tensor normalization.
For the Extended Classes benchmark, CIFAR-100 training set is partitioned into $T$ class-disjoint subsets, each containing $ 100 / T$  classes, where $T$ is the number of update steps.
% \textbf{Experiments on Tiny-ImageNet-200~{\normalfont \cite{le2015tiny}.}} For this set of experiments the image size is 64$\times$64.
%     A ResNet-18 architecture is used with the following hyperparameters for training: number of epochs 90; batch size 256; SGD optimizer with learning rate that starts from 0.1 and is divided by 10 after 50 and 70 epochs; SGD momentum 0.9; weight decay $5 \cdot 10^{-4}$. 
%     A temperature factor of $12$ is used in the cross-entropy loss to scale logit vectors during training.      
%     Training images are subjected to random cropping, horizontal flipping, and tensor normalization.
%     Tiny-ImageNet-200 training set is divided into chunks each having $ 200 / T$ classes, being $T$ the update steps.

\textbf{Experiments on ImageNet-1K~{\normalfont \citep{russakovsky2015imagenet}.}}
For this set of experiments, the image size is 224$\times$224.
Both ResNet-18 and ResNet-50 architectures are trained with the following hyperparameters: number of epochs 90; batch size 1536; SGD optimizer with an initial learning rate of 0.1, reduced by a factor of 10 at epochs 30 and 60; SGD momentum 0.9; weight decay $1 \cdot 10^{-4}$.
Training images are subjected to random cropping, horizontal flipping, color jitter, and tensor normalization, as is standard for ImageNet-1K training.
For the Extended Classes benchmark, the ImageNet-1K training set is partitioned into $T$ class-disjoint subsets, each containing $1000/T$ classes, where $T$ is the number of update steps.

\textbf{Experiments on Google Landmarks v2~{\normalfont \citep{weyand2020google}.}}
For this set of experiments, the image size is 224$\times$224.
A ResNet-18 architecture, initialized from ImageNet-1K pretrained weights at every update step, is trained with the following hyperparameters: number of epochs 30; batch size 512; SGD optimizer with an initial learning rate of 0.1, reduced by a factor of 10 at epochs 5, 10, and 20; SGD momentum 0.9; weight decay $5 \cdot 10^{-4}$.
Training images are subjected to random cropping and tensor normalization.
For the Extended Classes benchmark, the first update step uses a subset of 24{,}393 classes, and the remaining training set is partitioned into $T-1$ class-disjoint subsets, each containing $(81{,}313 - 24{,}393)/(T-1)$ classes, where $T$ is the number of update steps.
    
% \textbf{Experiments on ImageNet-1K~{\normalfont \citep{russakovsky2015imagenet}.}} For this set of experiments the image size is 224$\times$224.
% Both ResNet-18 and ResNet-50 architectures are trained with the following hyperparameters: number of epochs 90; batch size 1536; SGD optimizer with learning rate that starts from 0.1 and is divided by 10 after 30 and 60 epochs; SGD momentum 0.9;  weight decay $1 \cdot 10^{-4}$. 
% Training images are subjected to random cropping, horizontal flipping, color jitter, and tensor normalization, as standard for ImageNet-1K training.
% For the Extended Classes benchmark, ImageNet-1K training set is divided into chunks each having $ 1000 / T$ classes, being $T$ the number of update steps.

% \textbf{Experiments on Google Landmarks v2~{\normalfont \citep{weyand2020google}.}} For this set of experiments the image size is 224$\times$224.
% A ResNet-18 architecture, pretrained on ImageNet-1K, is used with the following hyperparameters for training: number of epochs 30; batch size 512; SGD optimizer with learning rate that starts from 0.1 and is divided by 10 after 5, 10 and 20 epochs; SGD momentum 0.9; weight decay $5 \cdot 10^{-4}$. 
% Training images are subjected to random cropping and tensor normalization.
% In the initial step, we use a subset with 24393 classes. In the following steps, Google Landmarks v2 training set is divided into chunks each having $(81313 - 24393) / (T-1)$ classes, being $T$ the update steps.

\section{Compatibility Matrices for 5-Step Update Scenario in the Extended Classes Benchmark}\label{sec:comp_CIFAR-100}

In Fig.~\ref{fig:comp_matrix_cifar_5step}, we present the compatibility matrices showing the CMC@$1$ retrieval metric values at each model update step for the 5-step scenario on the CIFAR-100 dataset of Tab.~\ref{tab:cifar}.

The matrices show that PSP and LSP achieve the highest compatibility while maintaining competitive retrieval performance across the update sequence.

\begin{figure*}[t]
    % % \vspace{-0.3cm}
    \centering
    \begin{subfigure}{0.33\linewidth}
        \centering
        \includegraphics[width=1\linewidth]{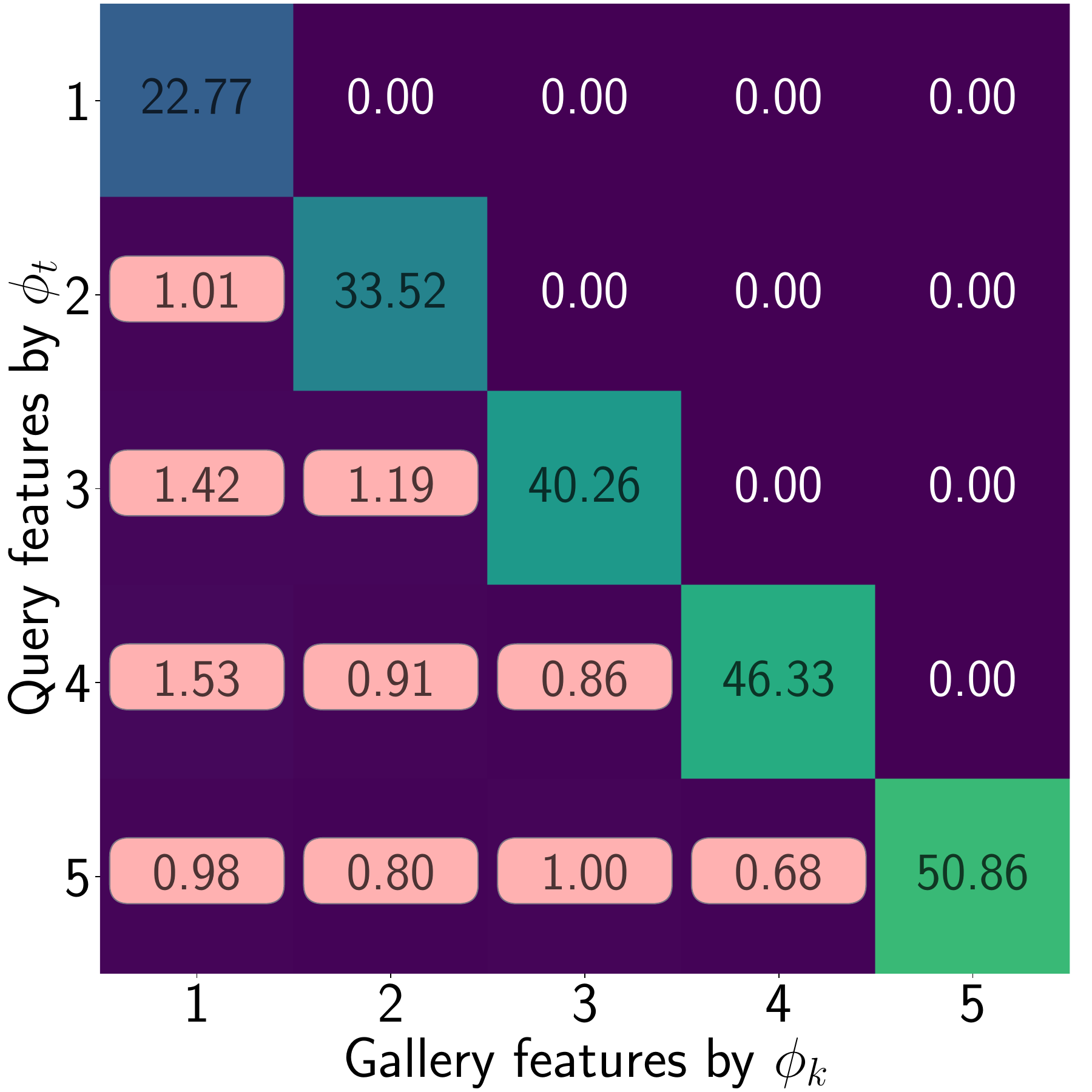}
        \caption{Baseline}
    \label{fig:cifar_prelogit_5step}
    \end{subfigure}
    \hspace{-10pt}
    \begin{subfigure}{0.33\linewidth}
        \centering
        \includegraphics[width=1\linewidth]{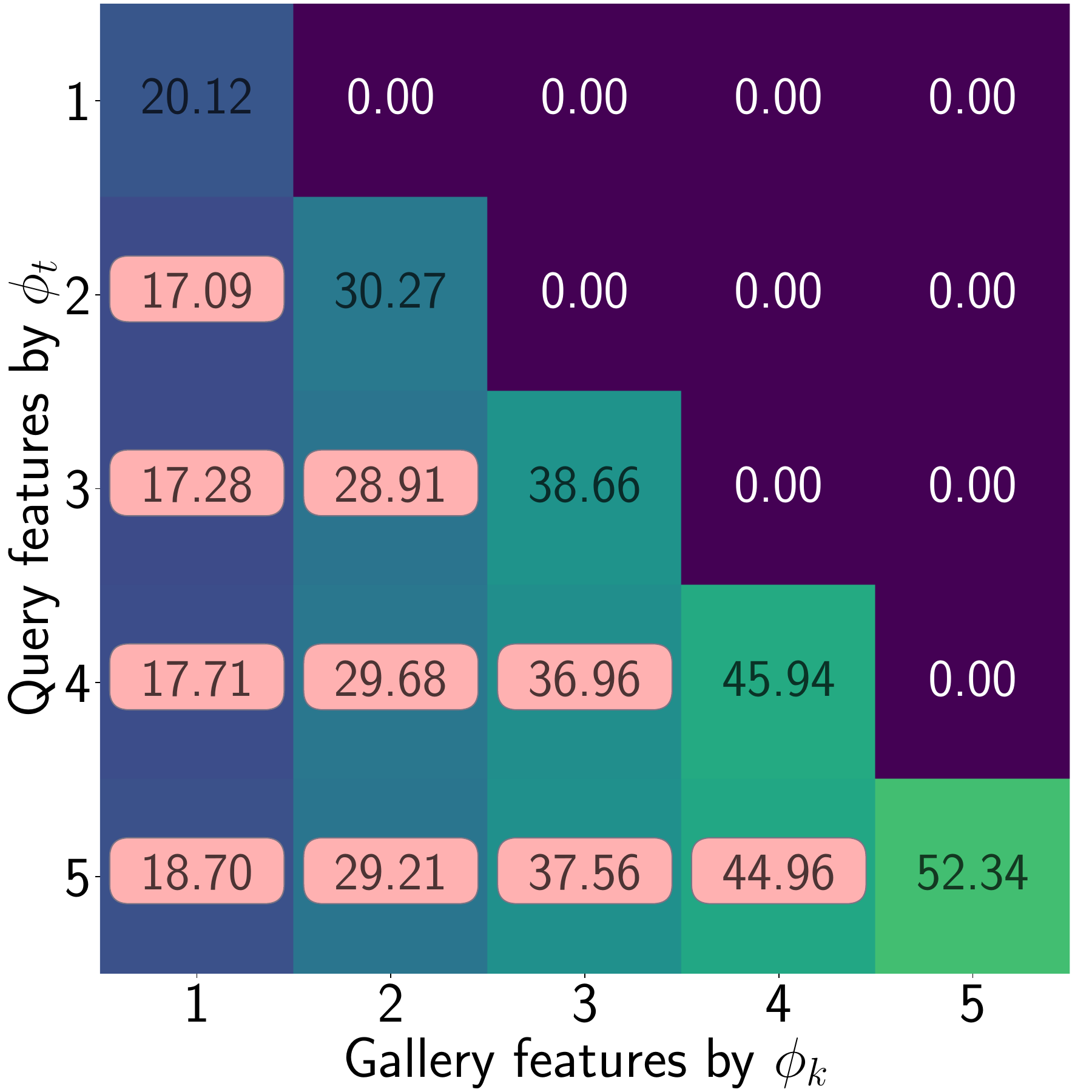}
        \caption{BCT}
    \label{fig:cifar_bct_5step}
    \end{subfigure}
    \hspace{-10pt}
    \begin{subfigure}{0.33\linewidth}
        \centering
        \includegraphics[width=1\linewidth]{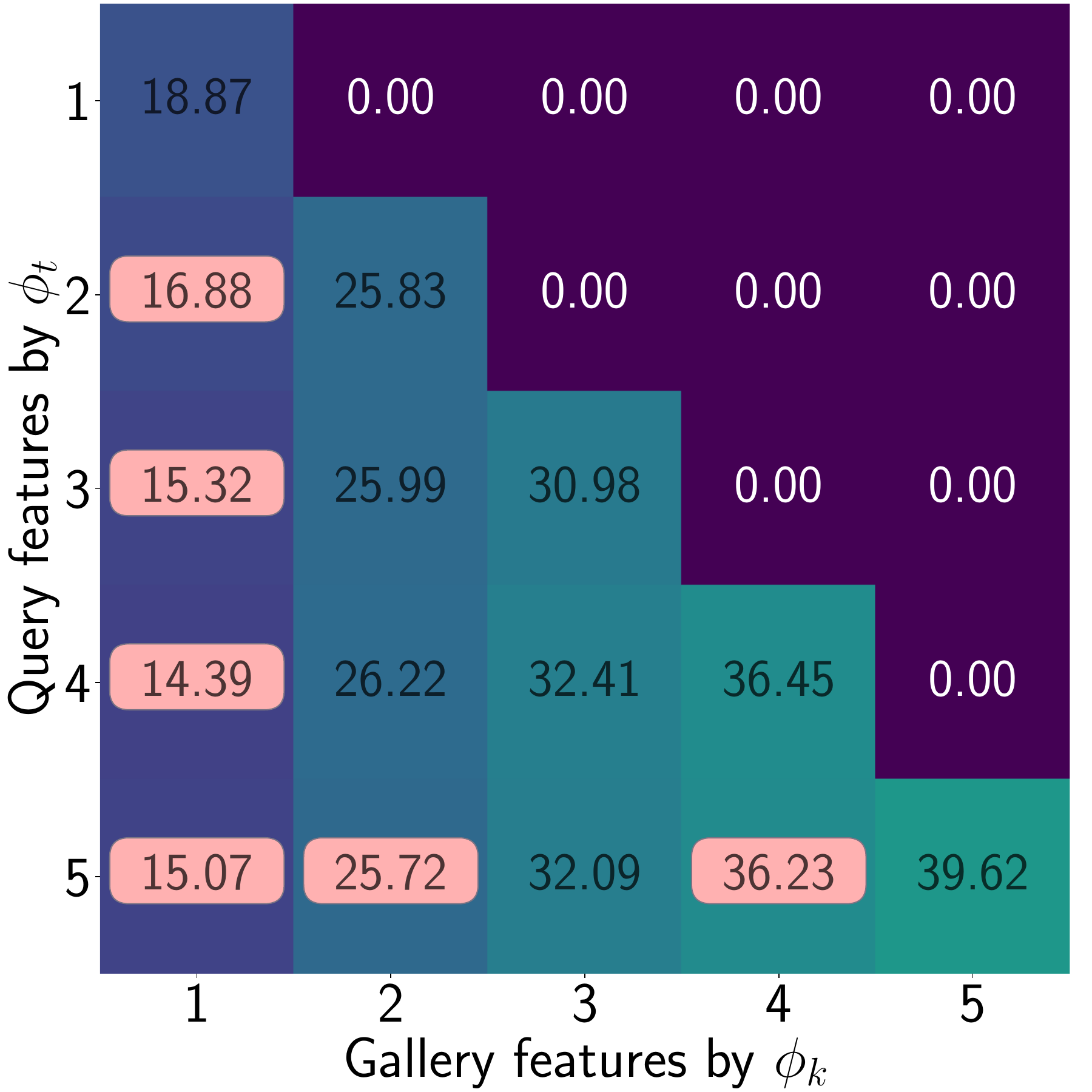}
        \caption{AdvBCT}
    \label{fig:cifar_advbct_5step}
    \end{subfigure}
    \hspace{-10pt}
    \begin{subfigure}{0.33\linewidth}
        \centering
        \includegraphics[width=1\linewidth]{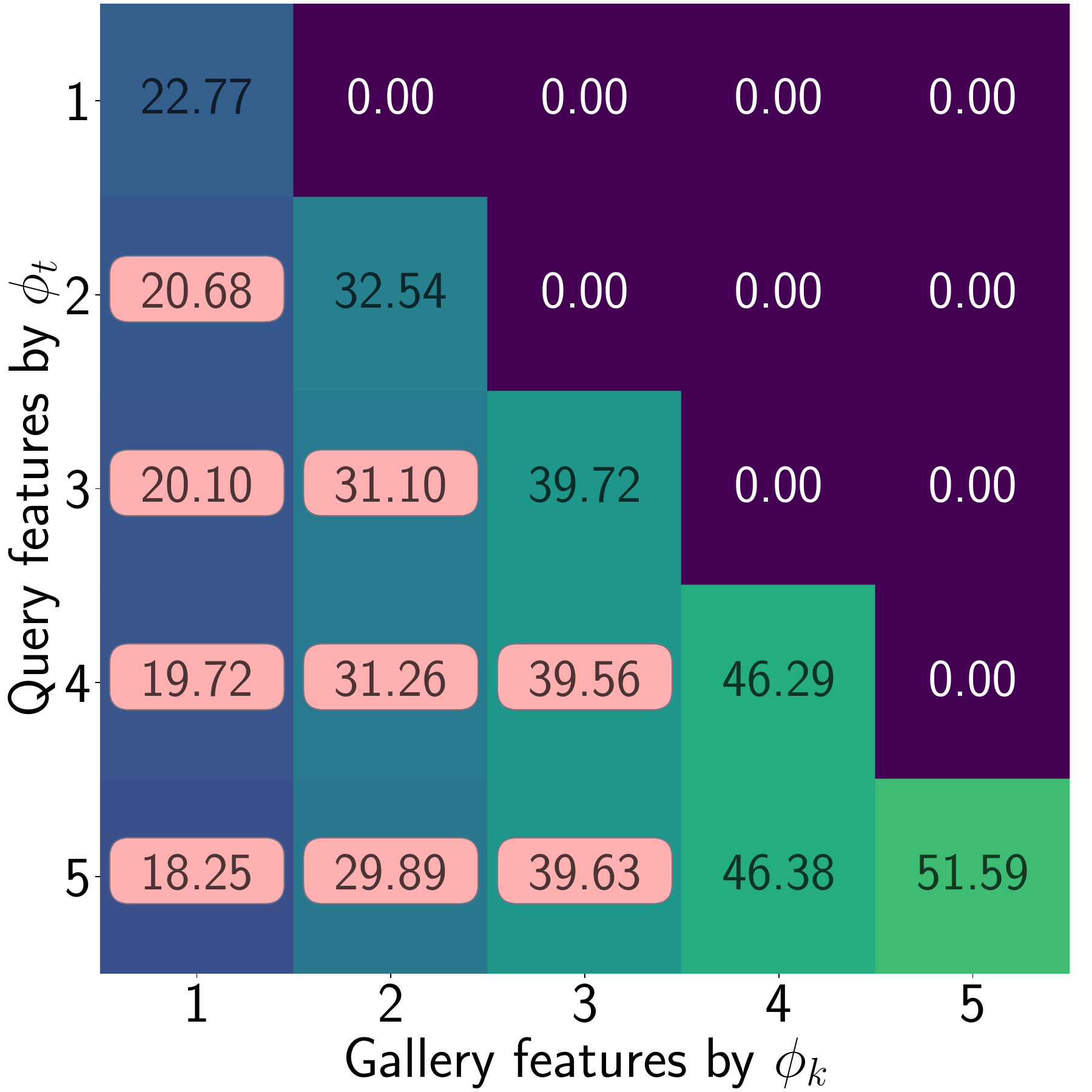}
        \caption{LCE}
    \label{fig:cifar_lce_5step}
    \end{subfigure}
    \hspace{-10pt}
    \begin{subfigure}{0.33\linewidth}
        \centering
        \includegraphics[width=1\linewidth]{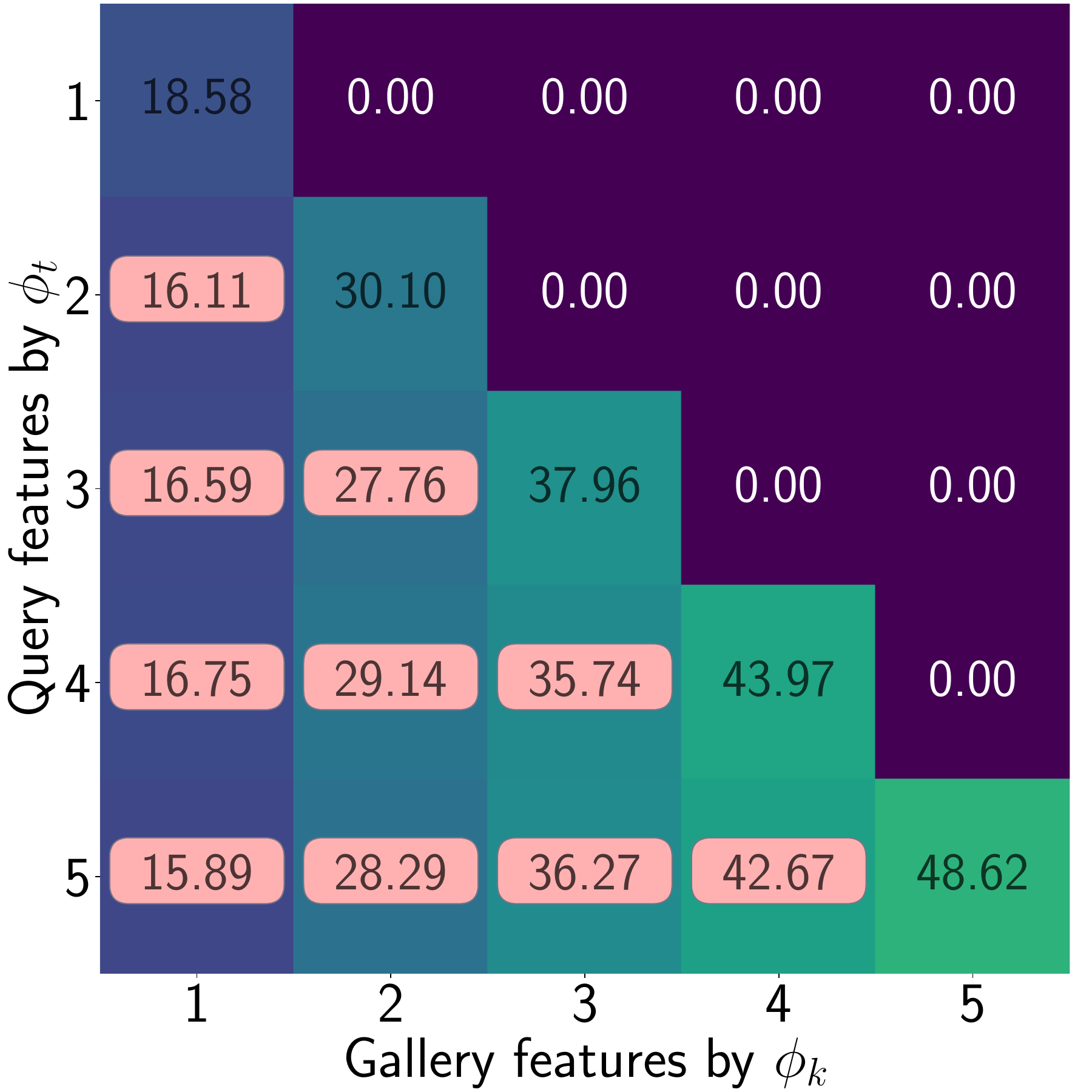}
        \caption{CoReS}
    \label{fig:cifar_cores_5step}
    \end{subfigure}
        \begin{subfigure}{0.33\linewidth}
        \centering
        \includegraphics[width=1\linewidth]{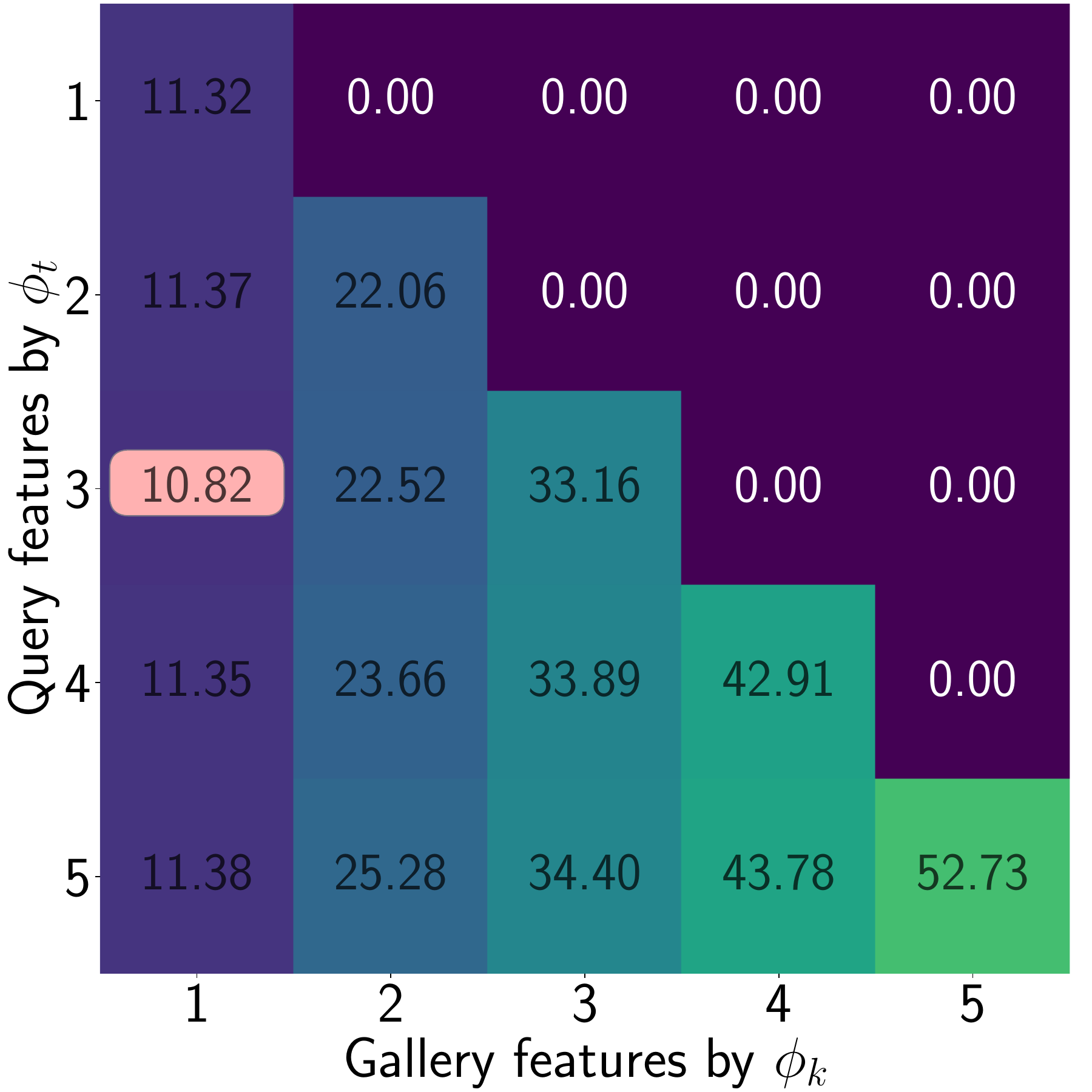}
        \caption{PSP}
    \label{fig:cifar_soft_5step}
    \end{subfigure}
    \hspace{-10pt}
    \begin{subfigure}{0.33\linewidth}
        \centering
        \includegraphics[width=1\linewidth]{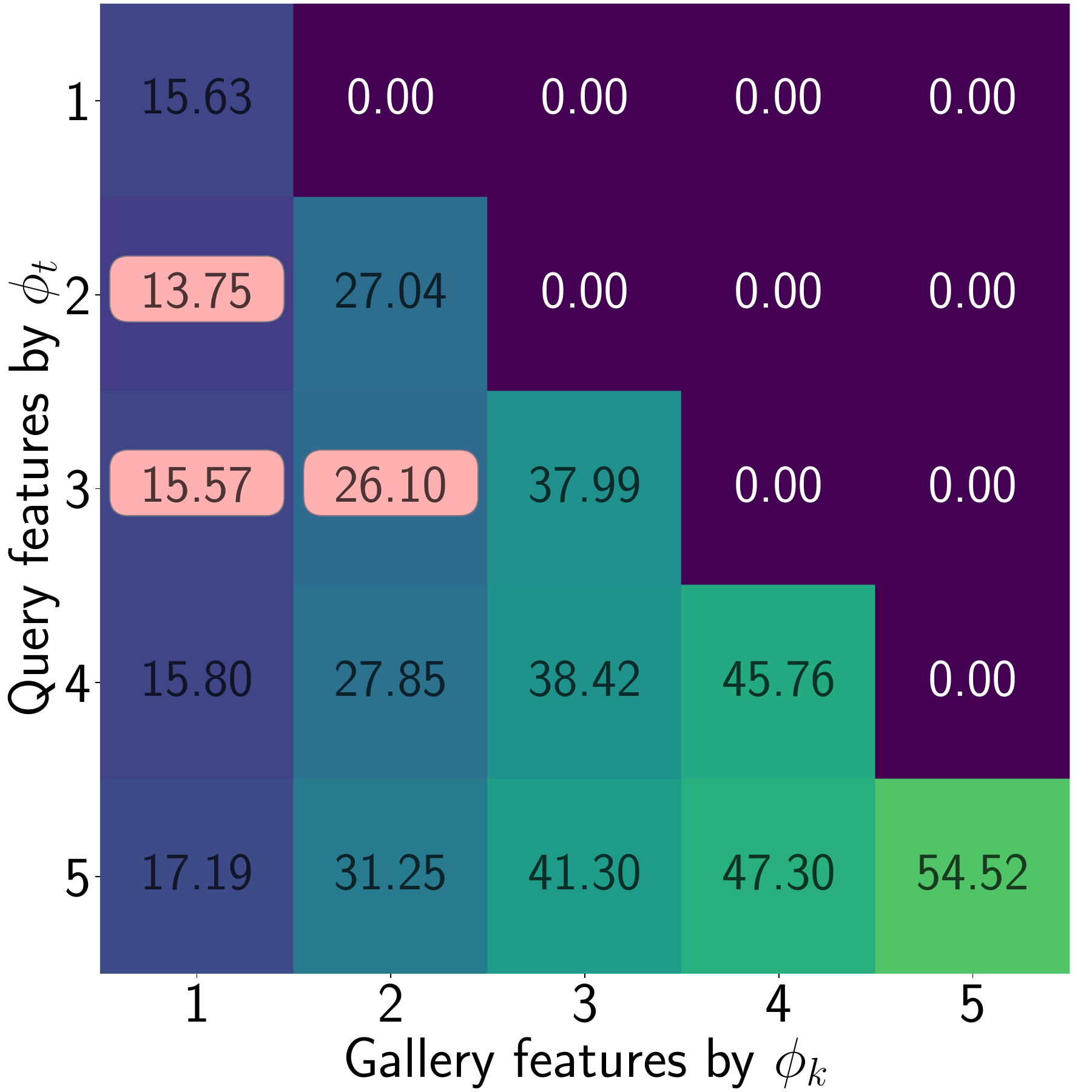}
        \caption{LSP}
    \label{fig:cifar_logits_5step}
    \end{subfigure}
    \caption{Compatibility matrices showing CMC@$1$ retrieval metric values at each model update step for the 5-step scenario on the CIFAR-100 dataset. Each matrix entry failing to satisfy the compatibility criterion defined in Eq.~\ref{eq:multistepecc} is highlighted with a light-red background.}
    \label{fig:comp_matrix_cifar_5step}
\end{figure*}

\section{Additional Results on Independently Trained Models from Public Repositories}\label{sec:compatiiblity_matrix_ind_trained}

\begin{figure*}[t]
    % % \vspace{-0.3cm}
    \centering
        \begin{subfigure}{0.33\linewidth}
        \centering
        \includegraphics[width=0.95\linewidth]{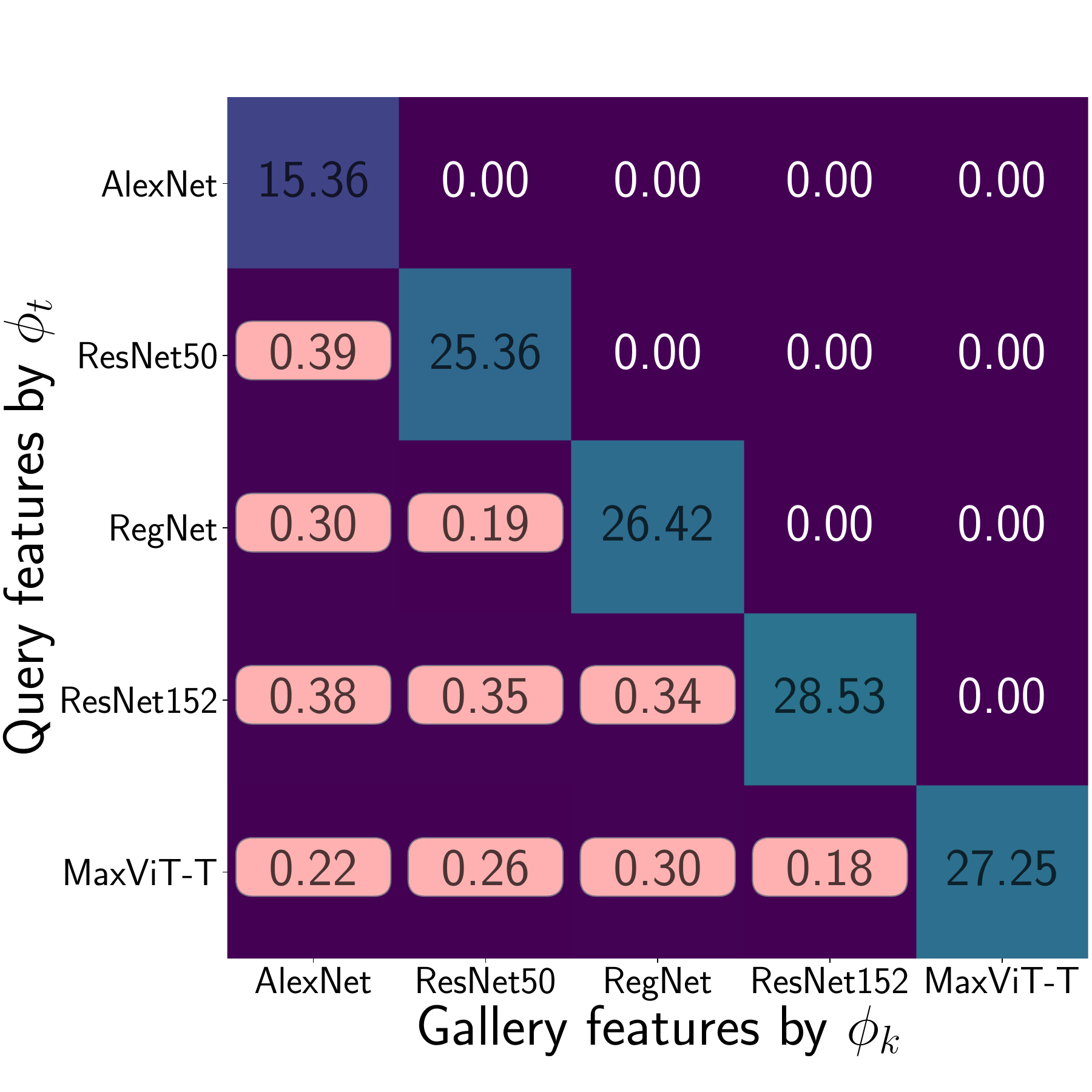}
        \caption{Baseline}
    \label{fig:place_penu_5step}
    \end{subfigure}
    \hspace{-10pt}
    \begin{subfigure}{0.33\linewidth}
        \centering
        \includegraphics[width=0.95\linewidth]{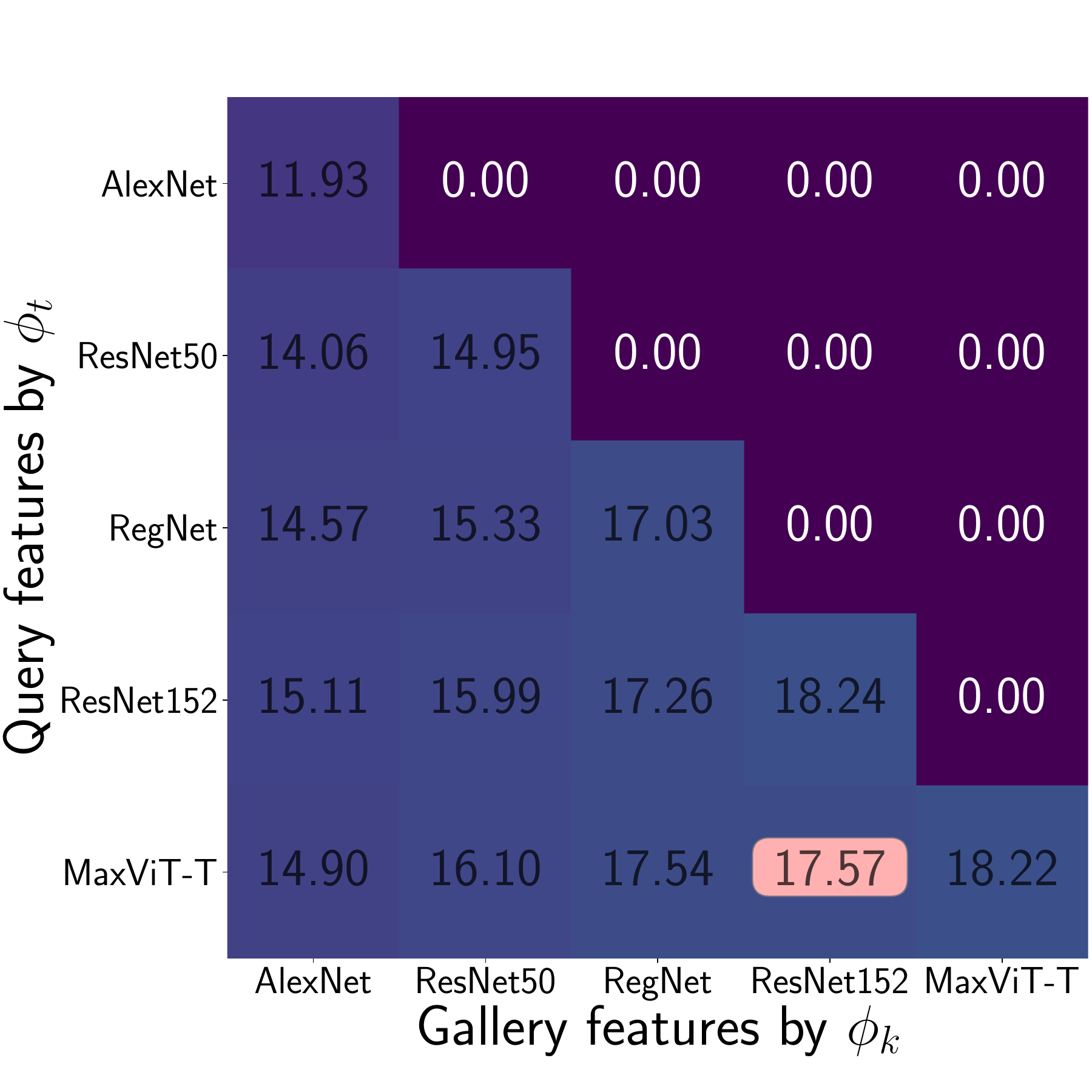}
        \caption{PSP}
    \label{fig:place_softmax_5step}
    \end{subfigure}
    \hspace{-10pt}
    \begin{subfigure}{0.33\linewidth}
        \centering
        \includegraphics[width=0.95\linewidth]{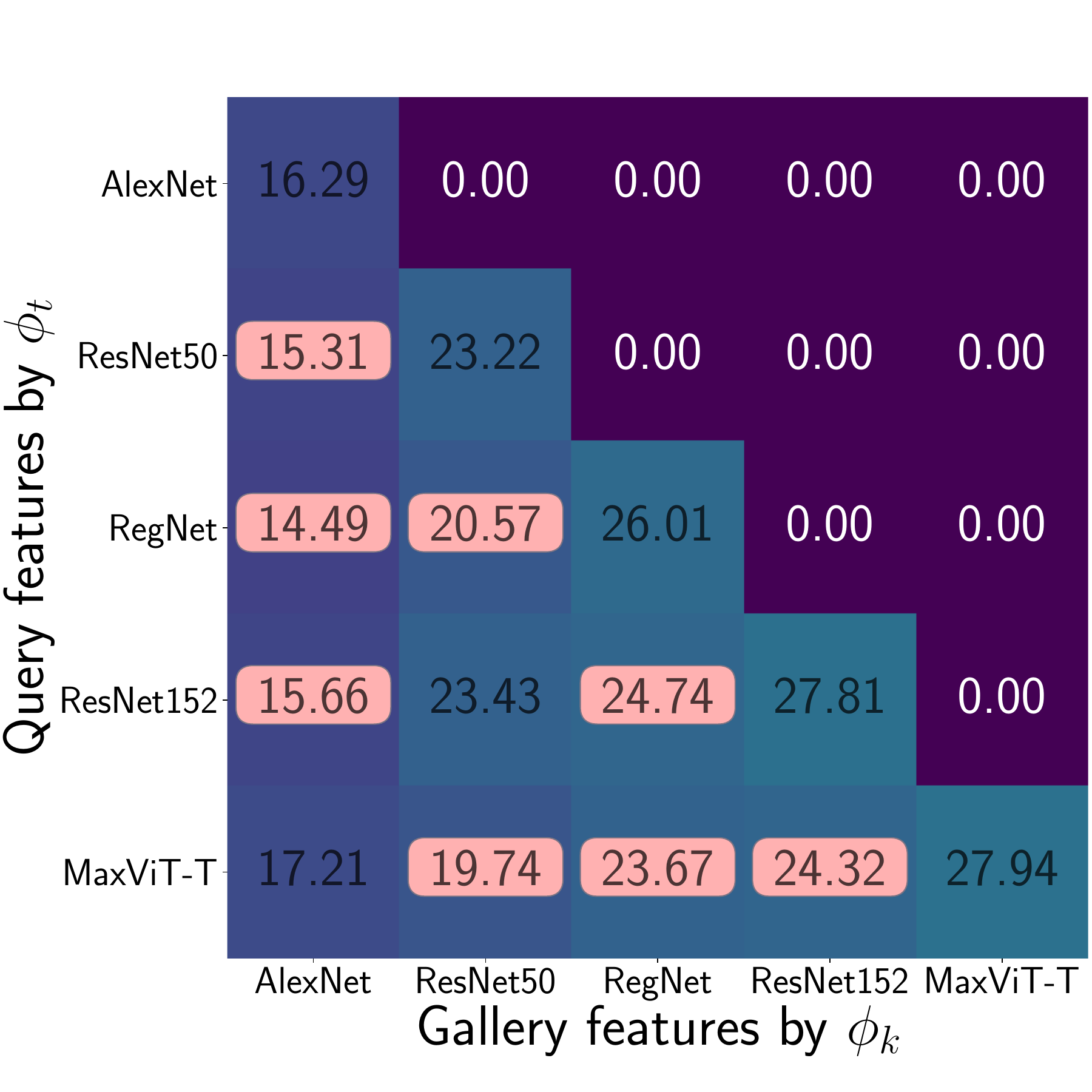}
        \caption{LSP}
    \label{fig:place_logit_5step}
    \end{subfigure}
    \caption{Open-set scenario. Compatibility matrices for five steps with increasingly expressive models (AlexNet, ResNet-50, RegNetX\_3.2GF, ResNet-152, and MaxViT\_T), all pretrained on ImageNet-1K. The matrices display retrieval performance on the Places365 dataset. 
     Entries that do not satisfy the compatibility criterion in Eq.~\ref{eq:multistepecc} are highlighted with a light-red background.
    Model representations obtained from: (a) encoder output (Baseline); (b) softmax outputs (PSP); (c) logits (LSP).}
    \label{fig:pretrained_comp_matrix_places}
\end{figure*}

In this appendix, we present the compatibility matrices corresponding to the experimental results reported in Tab.~\ref{tab:pretrained_models}  of Sec.~\ref{sec:compatibility_pretrained_models}.

Specifically, Fig.~\ref{fig:pretrained_comp_matrix_places} shows the CMC@$1$ values of the compatibility matrices in the open-set scenario using Places365 as the evaluation set. The diagonal entries represent \textit{same-model accuracy}, while the lower triangular entries show \textit{cross-model accuracy}. These experiments involve sequential model updates using AlexNet, ResNet-50, RegNetX\_3.2GF, ResNet-152, and MaxViT\_T, each independently pretrained on ImageNet-1K, with Places365 serving as the test set.

Furthermore, Fig.~\ref{fig:pretrained_comp_matrix_vit_imagenet} and~\ref{fig:pretrained_comp_matrix_vit_places} show the compatibility matrices for closed-set and open-set scenarios, respectively, using three Vision Transformer (ViT) architectures: ViT-B/32 and ViT-B/16, both trained from scratch on ImageNet-1K, and ViT-L/16, which is pretrained using weakly supervised learning on the IG-3.6B dataset and subsequently fine-tuned end-to-end on ImageNet-1K~\citep{singh2022revisiting}.

The results show that in the closed-set scenario, PSP achieves the highest compatibility and retrieval performance. Conversely, in open-set scenarios, while PSP maintains superior compatibility, it exhibits reduced retrieval performance. LSP achieves higher compatibility than the Baseline but lower than PSP; however, it substantially outperforms both the Baseline and PSP in cross-model retrieval performance, as evidenced by the off-diagonal values in the compatibility matrices.

\begin{figure*}[t]
    % \vspace{-0.5cm}
    \centering
    \begin{subfigure}{0.33\linewidth}
        \centering
        \includegraphics[width=0.7\linewidth]{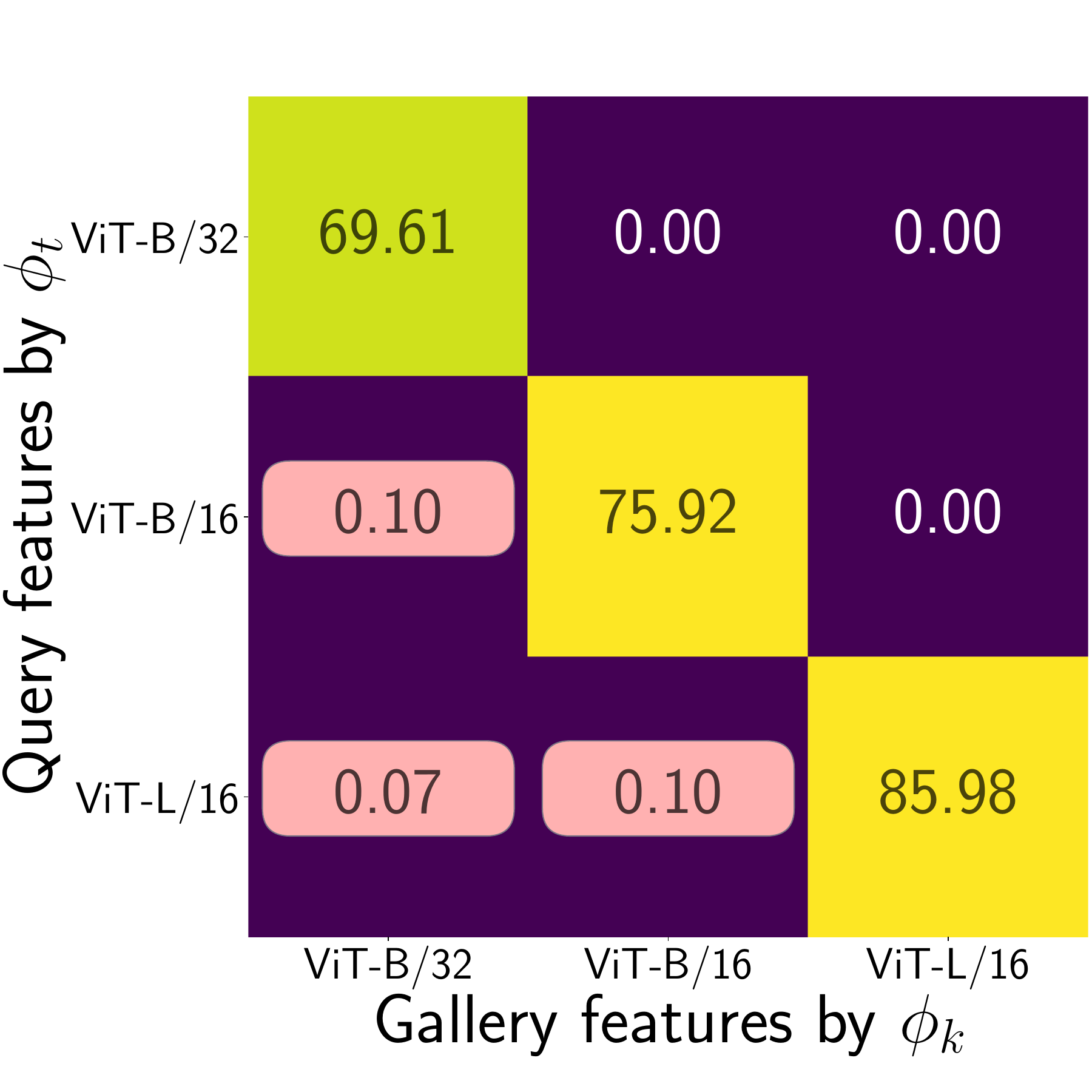}
        \caption{Baseline}
    \label{fig:prelogit_3step}
    \end{subfigure}
    \hspace{-10pt}
    \begin{subfigure}{0.33\linewidth}
        \centering
        \includegraphics[width=0.7\linewidth]{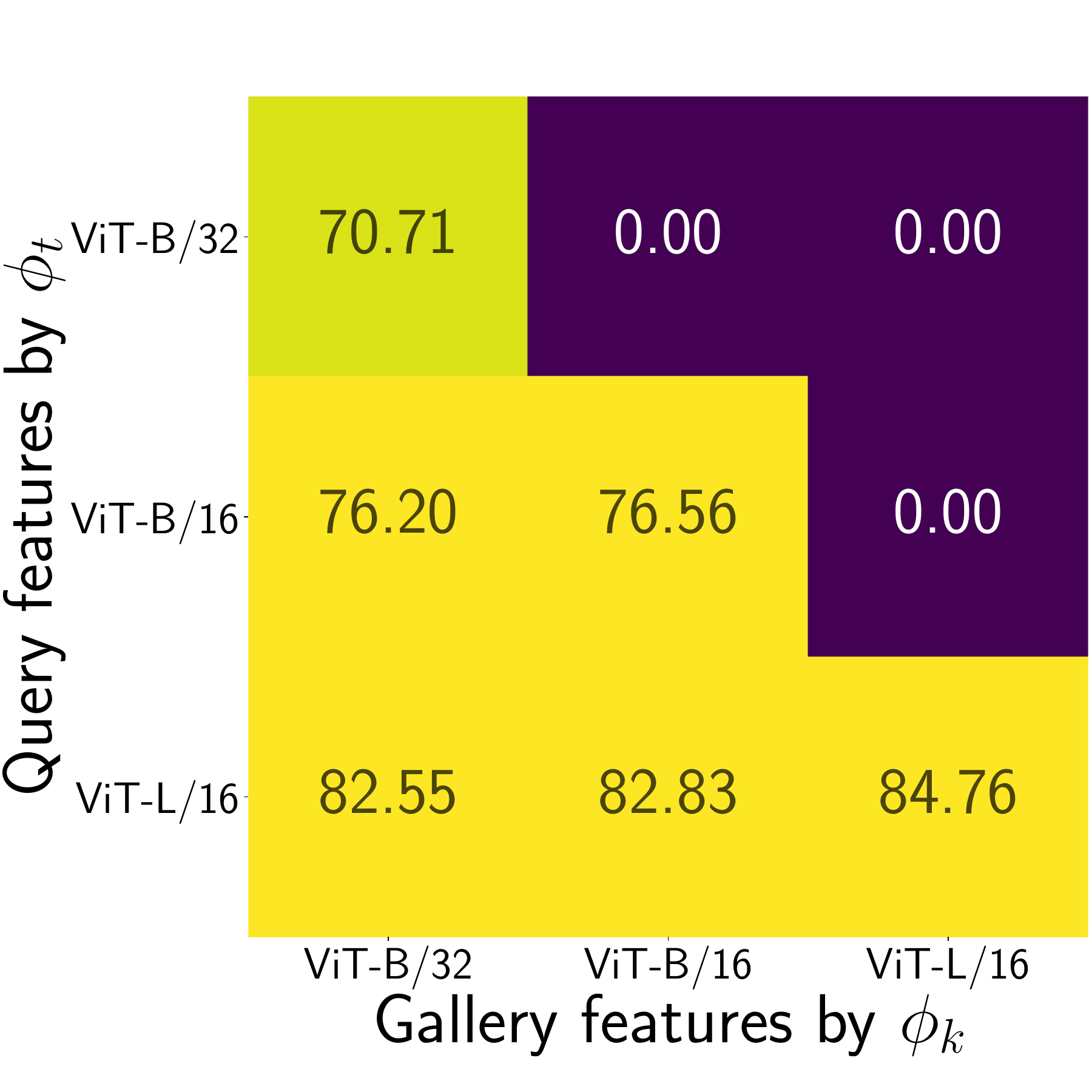}
        \caption{PSP}
    \label{fig:softmaxp_3step}
    \end{subfigure}
    \hspace{-10pt}
    \begin{subfigure}{0.33\linewidth}
        \centering
        \includegraphics[width=0.7\linewidth]{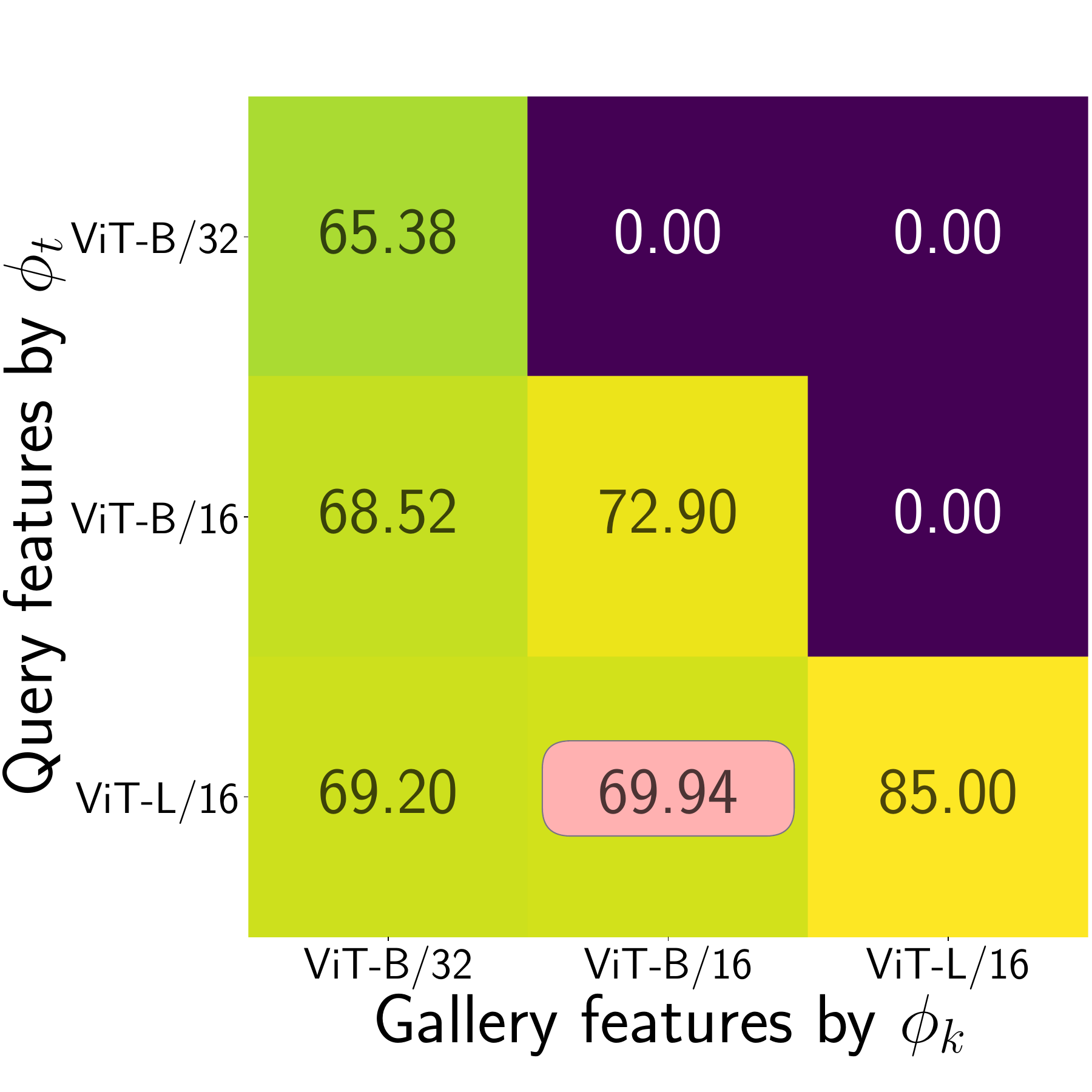}
        \caption{LSP}
    \label{fig:logit_3step}
    \end{subfigure}
    \caption{Closed-set scenario. Compatibility matrices for three steps with advanced Vision Transformer (ViT) model architectures (ViT-B/32, ViT-B/16, and ViT-L/16), all pretrained on ImageNet-1K. The matrices show closed-set retrieval performance on the ImageNet-1K test set.  Entries that do not satisfy the compatibility criterion in Eq.~\ref{eq:multistepecc} are highlighted with a light-red background.
    Model representations obtained from: (a) encoder output (Baseline); (b) softmax outputs (PSP); (c) logits (LSP).}
    \label{fig:pretrained_comp_matrix_vit_imagenet}
\end{figure*}

\begin{figure*}[t]
    % \vspace{-0.5cm}
    \centering
    \begin{subfigure}{0.33\linewidth}
        \centering
        \includegraphics[width=0.7\linewidth]{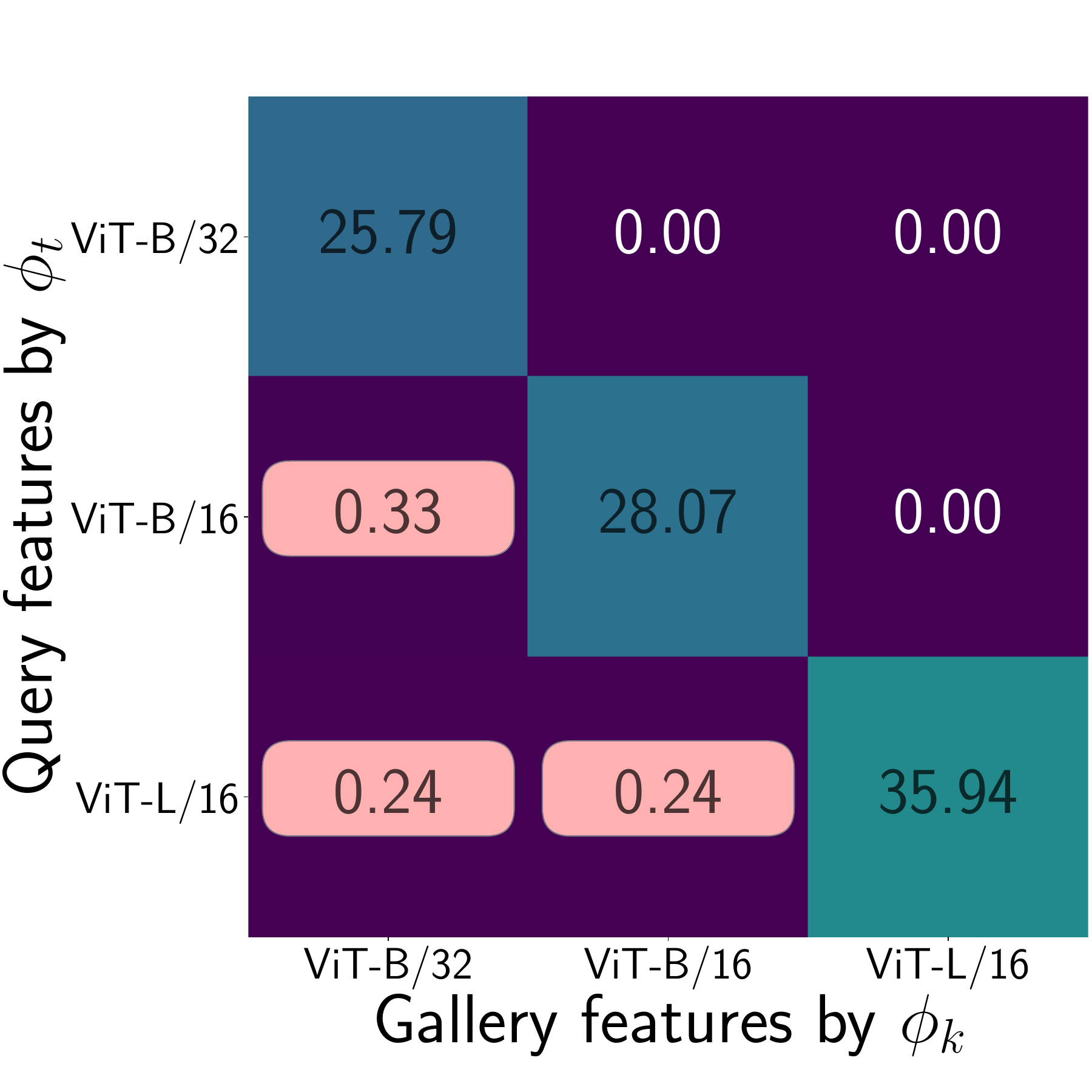}
        \caption{Baseline}
    \label{fig:place_penu_3step}
    \end{subfigure}
    \hspace{-10pt}
    \begin{subfigure}{0.33\linewidth}
        \centering
        \includegraphics[width=0.7\linewidth]{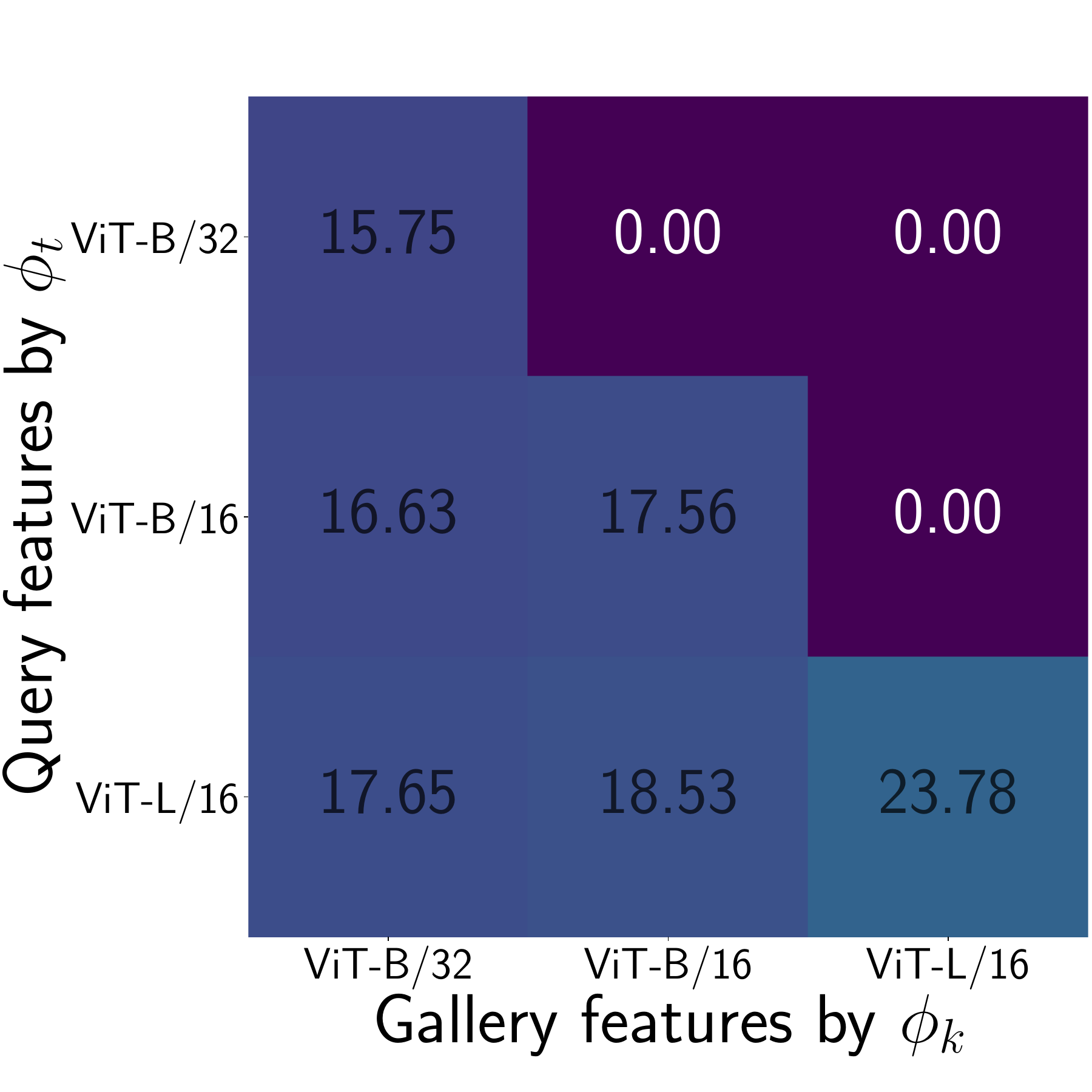}
        \caption{PSP}
    \label{fig:place_soft_3step}
    \end{subfigure}
    \hspace{-10pt}
    \begin{subfigure}{0.33\linewidth}
        \centering
        \includegraphics[width=0.7\linewidth]{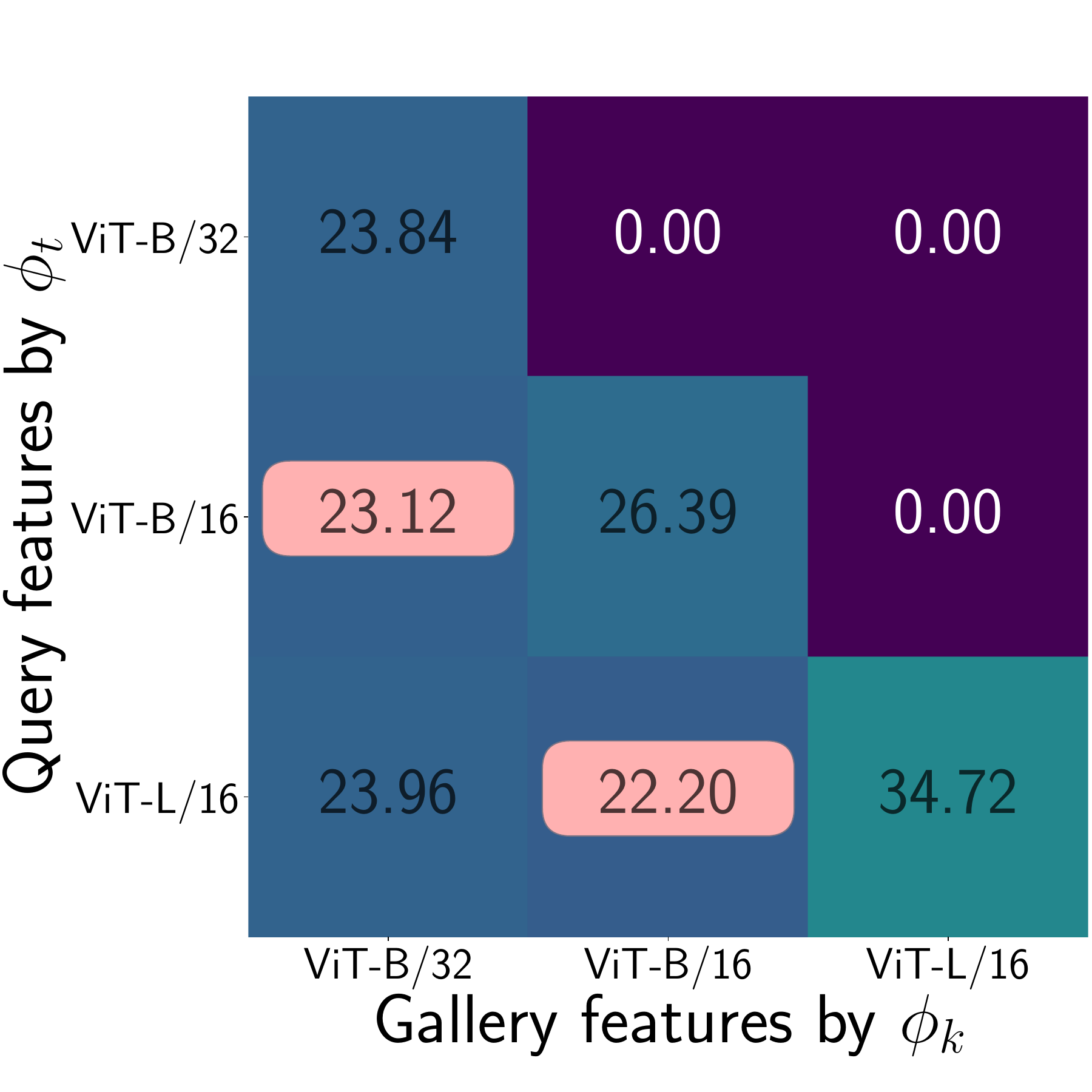}
        \caption{LSP}
    \label{fig:place_logit_3step}
    \end{subfigure}
    \caption{Open-set scenario. Compatibility matrices for three steps with advanced Vision Transformer (ViT) model architectures (ViT-B/32, ViT-B/16, and ViT-L/16), all pretrained on ImageNet-1K. The matrices show open-set retrieval performance on the Places365 test set. 
     Entries that do not satisfy the compatibility criterion in Eq.~\ref{eq:multistepecc} are highlighted with a light-red background.
     Model representations obtained from: (a) encoder output (Baseline); (b) softmax outputs (PSP); (c) logits (LSP).}
    \label{fig:pretrained_comp_matrix_vit_places}
\end{figure*}

\subsection{Extension to Self-Supervised and Vision-Language Encoders via Linear Probing}\label{sec:linear_probe}

In addition to supervised ImageNet-pretrained classifiers from torchvision \citep{paszke2019pytorch,marcel2010torchvision}, we extend our analysis to self-supervised pretrained encoders and pretrained vision-language models in a linear-probe setting. Specifically, we consider two families of publicly available pretrained models: (i) three checkpoints from the DINO family, namely DINO \citep{caron2021emerging}, DINOv2 \citep{oquab2023dinov2}, and DINOv3 \citep{simeoni2025dinov3}, obtained from the Hugging Face Hub \citep{wolf2019huggingface}; and (ii) three CLIP variants pretrained by OpenAI, namely ViT-B/32, ViT-B/16, and ViT-L/14, obtained from the OpenCLIP repository \citep{ilharco_gabriel_2021_5143773}. In all cases, the image encoder is used as a frozen feature extractor.

For each checkpoint, we train a bias-free linear classifier on top of the extracted features using ImageNet-1K as the training set for 30 epochs, while keeping the encoder weights frozen. We then construct PSP and LSP representations from the resulting classifier outputs, using softmax outputs and logits, respectively. As a baseline, we directly use the $\ell_2$-normalized encoder output without any additional training. Retrieval performance is evaluated on the ImageNet-1K test set, following the same protocol adopted in the main benchmark.

\begin{table*}[t]
    \centering
    \caption{Compatibility results for self-supervised and vision-language pretrained encoders in a linear-probe setting on ImageNet-1K. For each checkpoint, a linear probe is trained on top of frozen backbone features. Dark blue numbers indicate the highest values, while light blue numbers indicate the second-highest values.}
    \label{tab:pretrained_ssl_models}
    \begin{subtable}[t]{0.49\textwidth}
        \centering
        \small
        \caption{DINO family: DINO-ViT-S $\xrightarrow{}$ DINOv2-ViT-S/14 $\xrightarrow{}$ DINOv3-ViT-L/16}
        \label{tab:pretrained_ssl_models_dino}
        \begin{tabular}{l ccc c ccc}
            \toprule
            \textsc{Method}
            & {$AC$} & {$AA$} & {$ACA$} \\
            \midrule
            Baseline
                & 0 & 35.50 & 0  \\
            PSP
                & \textcolor[RGB]{0,0,139}{\textbf{1.00}}
                & \textcolor[RGB]{0,0,139}{\textbf{71.00}}
                & \textcolor[RGB]{0,0,139}{\textbf{72.00}}
                \\
            LSP
                & \textcolor[RGB]{65,105,225}{\textbf{0.67}}
                & \textcolor[RGB]{65,105,225}{\textbf{68.70}}
                & \textcolor[RGB]{65,105,225}{\textbf{47.02}}
                \\
            \bottomrule
        \end{tabular}%
    \end{subtable}
    \hfill
    \begin{subtable}[t]{0.49\textwidth}
        \centering
        \small
        \caption{CLIP family: CLIP ViT-B/32 $\xrightarrow{}$ CLIP ViT-B/16 $\xrightarrow{}$ CLIP ViT-L/14}
        \label{tab:pretrained_ssl_models_clip}
        \begin{tabular}{l ccc c ccc}
            \toprule
            \textsc{Method}
            & {$AC$} & {$AA$} & {$ACA$}  \\
            \midrule
            Baseline
                & 0 & 30.70 & 0  \\
            PSP
                & \textcolor[RGB]{0,0,139}{\textbf{1.00}}
                & \textcolor[RGB]{0,0,139}{\textbf{60.10}}
                & \textcolor[RGB]{0,0,139}{\textbf{59.20}}
                 \\
            LSP
                & \textcolor[RGB]{65,105,225}{\textbf{0.33}}
                & \textcolor[RGB]{65,105,225}{\textbf{52.10}}
                & \textcolor[RGB]{65,105,225}{\textbf{16.42}}
                 \\
            \bottomrule
        \end{tabular}%
    \end{subtable}
\end{table*}

\begin{figure*}[t]
    \vspace{-0.2cm}
    \centering
        \begin{subfigure}{0.32\linewidth}
        \centering
        \includegraphics[width=0.78\linewidth]{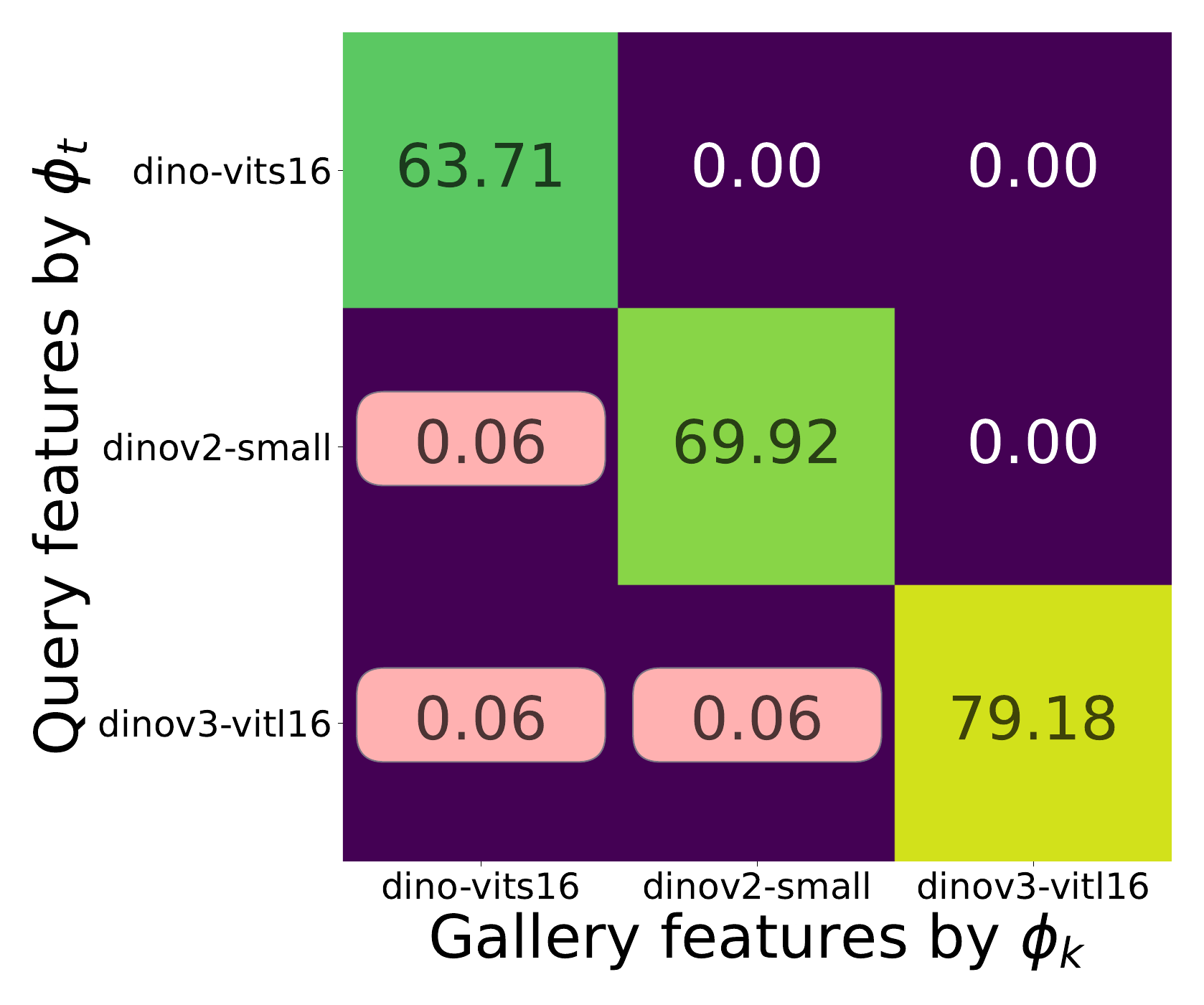}        \caption{Baseline}
    \label{fig:pretrained_comp_matrix_dino_imagenet_logit_prelogit}
    \end{subfigure}
    % \hspace{-5pt}
    \hfill
    \begin{subfigure}{0.32\linewidth}
        \centering
        \includegraphics[width=0.78\linewidth]{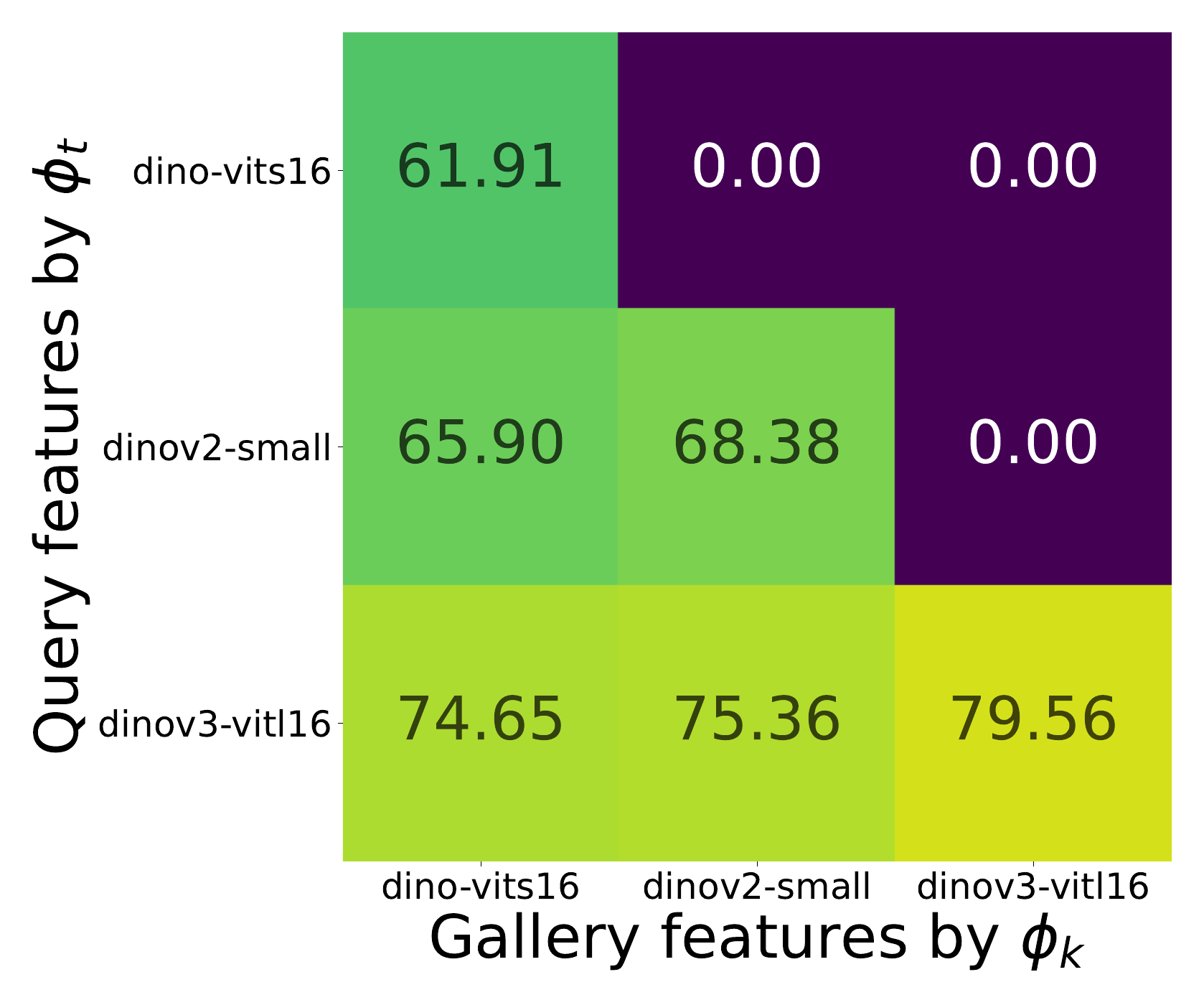}
        \caption{PSP}
    \label{fig:pretrained_comp_matrix_dino_imagenet_logit_softmax}
    \end{subfigure}
    \hfill
    % \hspace{-5pt}
    \begin{subfigure}{0.32\linewidth}
        \centering
        \includegraphics[width=0.78\linewidth]{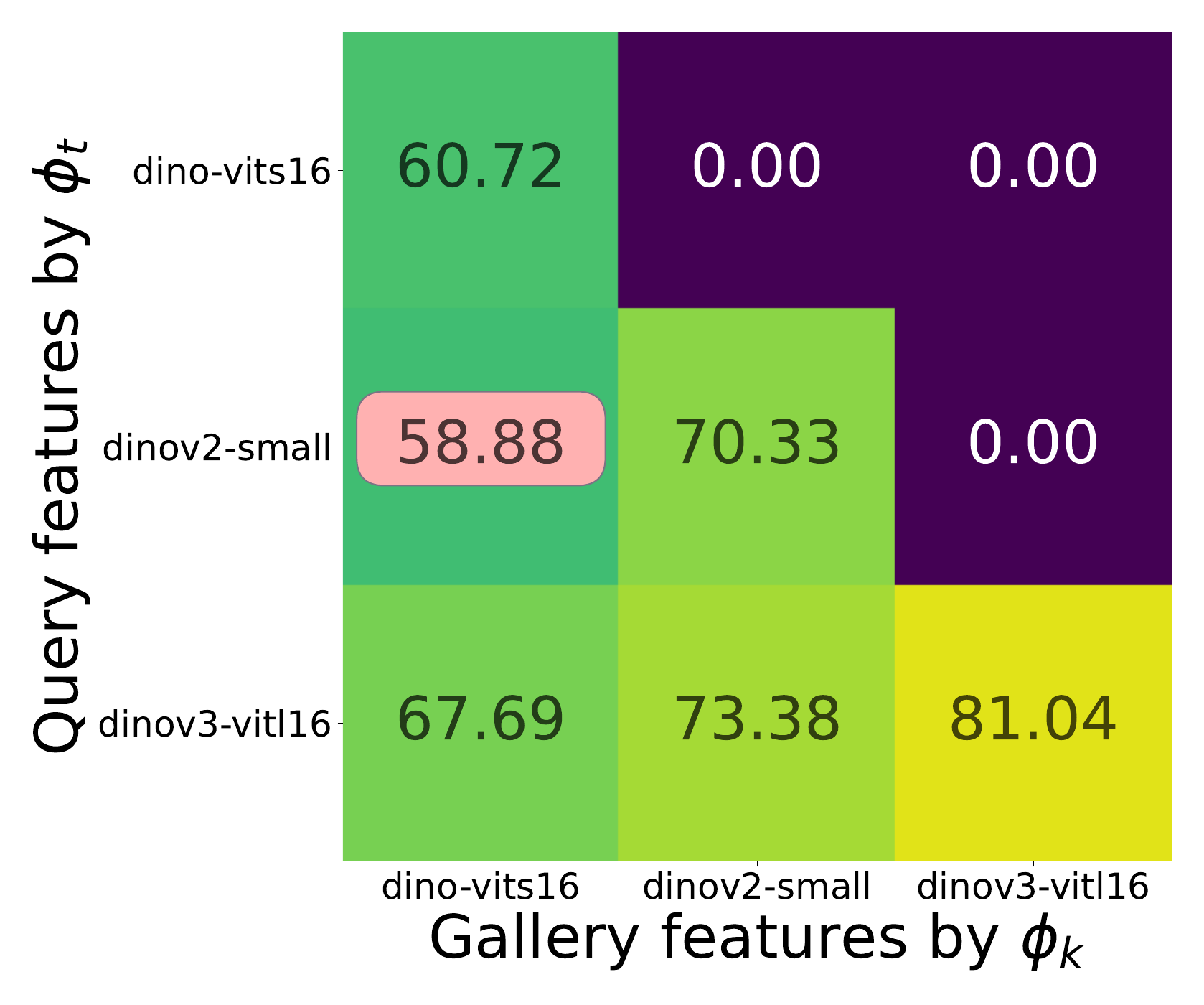}
        \caption{LSP}
    \label{fig:pretrained_comp_matrix_dino_imagenet_logit}
    \end{subfigure}
    \caption{Compatibility matrices for three steps with self-supervised Vision Transformer checkpoints from the DINO family (DINO-ViT-S, DINOv2-ViT-S/14, and DINOv3-ViT-L/16) obtained from Hugging Face Hub~\citep{wolf2019huggingface}. For each checkpoint, we train a linear probe on ImageNet-1K with the backbone frozen, and construct PSP/LSP features from the resulting classifier outputs. The matrices show closed-set retrieval performance on the ImageNet-1K test set. Entries that do not satisfy the compatibility criterion in Eq.~\ref{eq:multistepecc} are highlighted with a light-red background. Model representations obtained from: (a) encoder output (Baseline); (b) softmax outputs (PSP); (c) logits (LSP).}
\label{fig:pretrained_comp_matrix_dino_imagenet}
\end{figure*}

\begin{figure*}[t]
    \vspace{-0.2cm}
    \centering
        \begin{subfigure}{0.32\linewidth}
        \centering
        \includegraphics[width=0.78\linewidth]{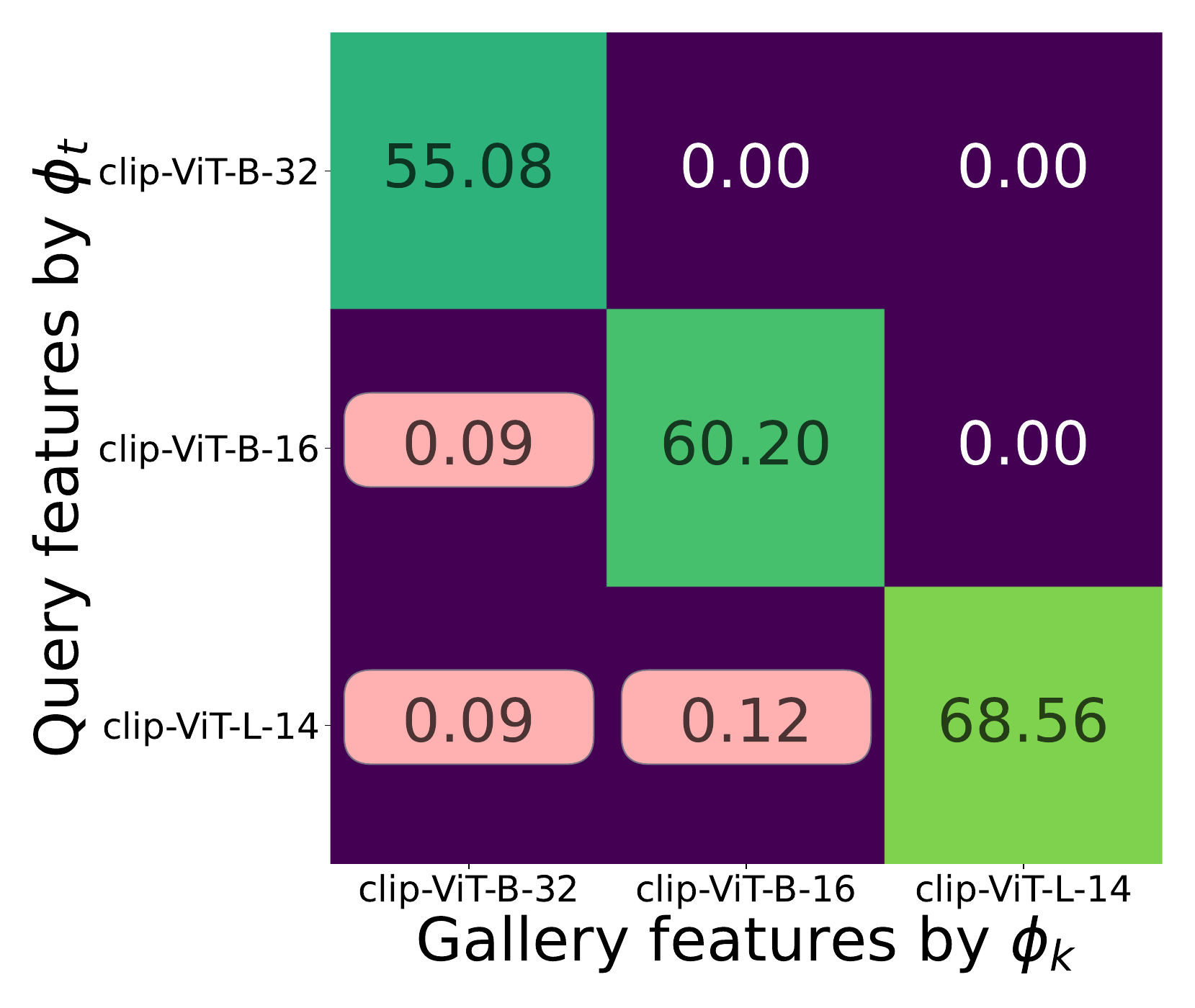}        \caption{Baseline}
    \label{fig:pretrained_comp_matrix_dino_imagenet_logit_prelogit_clip}
    \end{subfigure}
    % \hspace{-5pt}
    \hfill
    \begin{subfigure}{0.32\linewidth}
        \centering
        \includegraphics[width=0.78\linewidth]{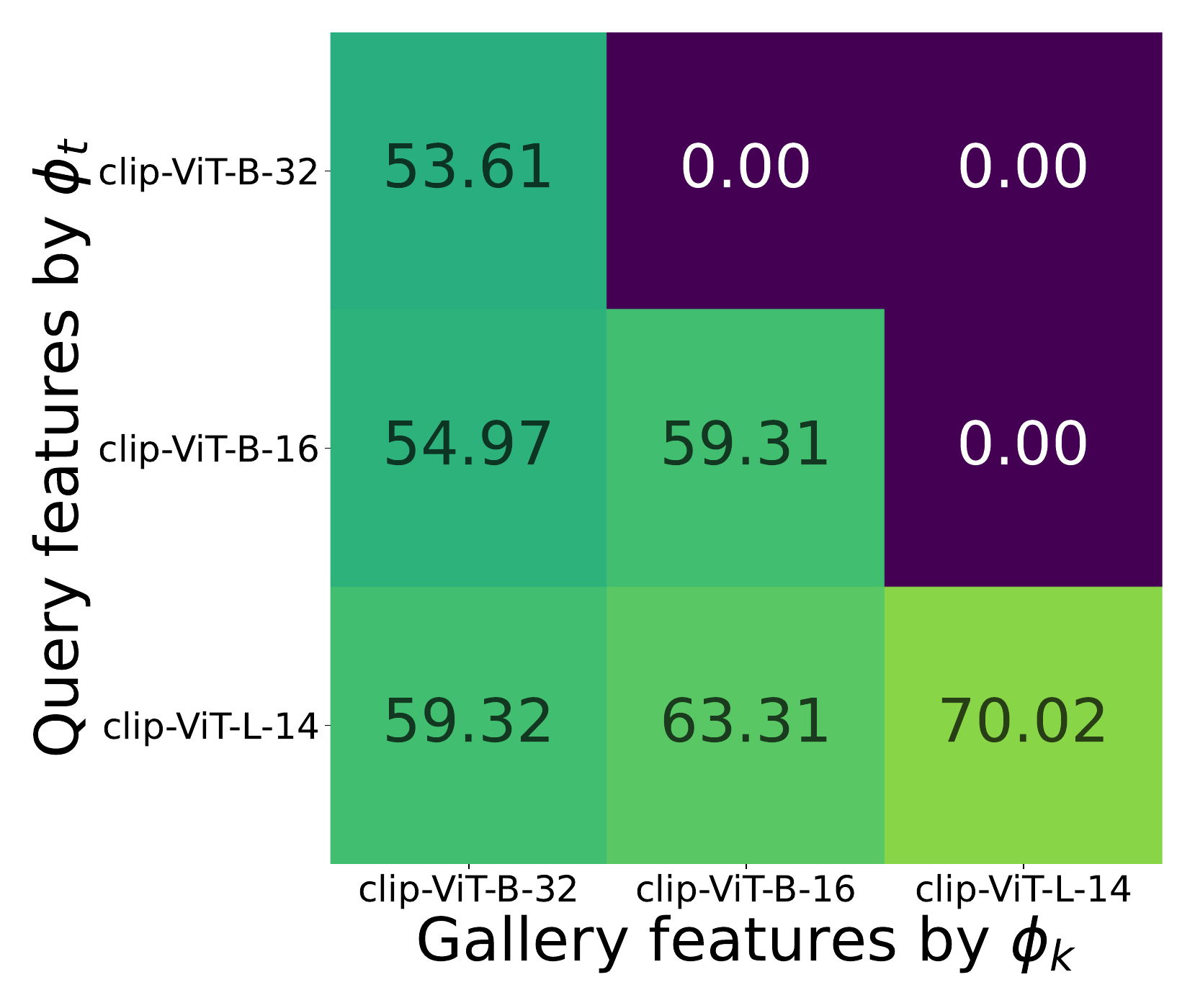}
        \caption{PSP}
    \label{fig:pretrained_comp_matrix_dino_imagenet_logit_softmax_clip}
    \end{subfigure}
    \hfill
    % \hspace{-5pt}
    \begin{subfigure}{0.32\linewidth}
        \centering
        \includegraphics[width=0.78\linewidth]{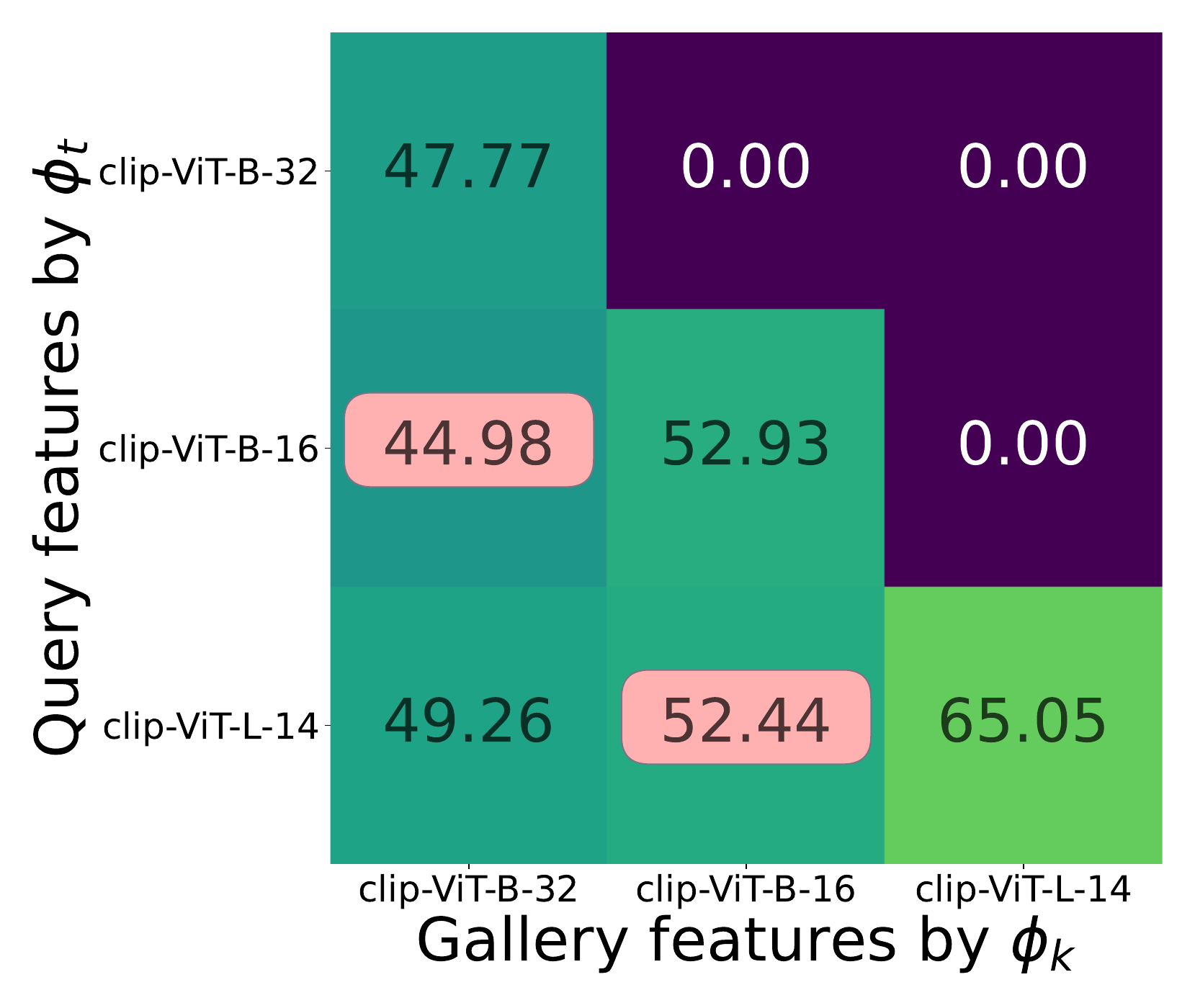}
        \caption{LSP}
    \label{fig:pretrained_comp_matrix_dino_imagenet_logit_clip}
    \end{subfigure}
    \caption{Compatibility matrices for three steps with vision-language CLIP family (ViT-B/32, ViT-B/16, and ViT-L/14) pretrained by OpenAI. For each checkpoint, we train a linear probe on ImageNet-1K with the backbone frozen, and construct PSP/LSP features from the resulting classifier outputs. The matrices show closed-set retrieval performance on the ImageNet-1K test set. Entries that do not satisfy the compatibility criterion in Eq.~\ref{eq:multistepecc} are highlighted with a light-red background. Model representations obtained from: (a) encoder output (Baseline); (b) softmax outputs (PSP); (c) logits (LSP).}
\label{fig:pretrained_comp_matrix_clip_imagenet}
\end{figure*}

Tab.~\ref{tab:pretrained_ssl_models} summarizes the compatibility results obtained in this linear-probe setting on ImageNet-1K. For both DINO and CLIP model families, PSP and LSP achieve substantially higher compatibility metrics than the baseline, confirming that the proposed framework can be extended beyond standard supervised classifiers.

For the DINO family, the observed trend is consistent with that reported for supervised ImageNet-pretrained models in Sec.~\ref{sec:compatibility_pretrained_models}. In particular, PSP achieves the strongest compatibility results, with AC~=~1.0 and ACA~=~72.00, while LSP achieves AC~=~0.67 and ACA~=~47.02. As shown in Fig.~\ref{fig:pretrained_comp_matrix_dino_imagenet}, the baseline achieves reasonable same-model retrieval performance, but cross-model retrieval remains negligible, indicating that independently pretrained self-supervised encoders do not naturally produce aligned representation spaces across updates.

For the CLIP family, we observe the same overall pattern, as shown in Fig.~\ref{fig:pretrained_comp_matrix_clip_imagenet}. PSP again achieves the strongest compatibility results (AC~=~1.0, ACA~=~59.20), while LSP improves substantially over the baseline but remains below PSP. The consistency of this behavior across both model families provides additional empirical support for the theoretical analysis in Sec.~\ref{sec:theo_section}.

Overall, these results show that a simple linear-probe protocol provides a practical way to apply PSP and LSP to pretrained encoders that do not natively expose a classification head, including both self-supervised and vision-language models.

\end{document}